\documentclass{article}

\usepackage[preprint,nonatbib]{neurips_2020}
\usepackage[numbers]{natbib}
\usepackage{xcolor}

\usepackage[utf8]{inputenc} 
\usepackage[T1]{fontenc}    
\usepackage{hyperref}       
\usepackage{url}            
\usepackage{booktabs}       
\usepackage{amsfonts}       
\usepackage{amssymb}        
\usepackage{amsmath}        
\usepackage{nicefrac}       
\usepackage{microtype}      
\usepackage{graphicx}       
\usepackage{amsmath,amsfonts,amsthm,graphicx,float, color,url,tikz,tabularx,enumitem}
\usepackage[leftcaption]{sidecap}
\newcolumntype{Y}{>{\centering\arraybackslash}X}

\usepackage{fancyvrb}
\usepackage{subcaption}
\sidecaptionvpos{figure}{c}

\newtheorem{definition}{Definition}

\newtheorem{proposition}{Proposition}
\newtheorem{claim}{Claim}

\graphicspath{{images/}}

\definecolor{olivier}{rgb}{.6,0,0}
\definecolor{sylvain}{rgb}{0,.6,0}
\definecolor{daniel}{rgb}{0,0,.6}
\definecolor{ibrahim}{rgb}{.6,.6,0}
\definecolor{hartmut}{rgb}{0,.6,.6}
\definecolor{robert}{rgb}{.6,0,.6}
\definecolor{ilya}{rgb}{0.8705882352941177, 0.050980392156862744, 0.8156862745098039}

\newcommand{\BE}{{\mathbb E}}%
\newcommand{\BR}{{\mathbb R}}%
\newcommand{\calD}{{\mathcal D}}%
\newcommand{\calG}{{\mathcal G}}%
\newcommand{\calN}{{\mathcal N}}%
\newcommand{\calV}{{\mathcal V}}%
\newcommand{\calY}{{\mathcal Y}}%
\newcommand{\eps}{\epsilon}%
\newcommand{\arrow}{\rightarrow}%
\DeclareMathOperator*{\argmin}{arg\,min}%
\author{%
  Hartmut Maennel\thanks{Equal contribution.}\\
  \texttt{hartmutm@google.com}
  \And
  Ibrahim Alabdulmohsin\footnotemark[1]\\
  \texttt{ibomohsin@google.com}
  \And
  Ilya Tolstikhin\\
  \texttt{tolstikhin@google.com}
  \And
  Robert J. N. Baldock\thanks{Work completed during the Google AI Residency Program.}\\
  \texttt{rbaldock@google.com}
  \And
  Olivier Bousquet\\
  \texttt{obousquet@google.com}
  \And
  Sylvain Gelly\\
  \texttt{sylvaingelly@google.com}
  \And
  Daniel Keysers\\
  \texttt{keysers@google.com}\\[3ex]
  Google Research, Brain Team\\
  Z\"urich, Switzerland
}

\title{What Do Neural Networks Learn\\ When Trained With Random Labels?}

\begin{document}
\maketitle

\begin{abstract}
We study deep neural networks (DNNs) trained on natural image data with entirely random labels. 
Despite its popularity in the literature, where it is often used to study memorization, generalization, and other phenomena, little is known about what DNNs learn in this setting.
In this paper, we show analytically for convolutional and fully connected networks that an alignment between the principal components of network parameters and data takes place when training with random labels.
We study this alignment effect by investigating neural networks pre-trained on randomly labelled image data  and subsequently fine-tuned on disjoint datasets with random or real labels.
We show how this alignment produces a \emph{positive} transfer: networks pre-trained with random labels train faster downstream compared to training from scratch even after accounting for simple effects, such as weight scaling.
We analyze how competing effects, such as specialization at later layers, may hide the positive transfer.
These effects are studied in several network architectures, including VGG16 and ResNet18, on CIFAR10 and ImageNet.
\end{abstract}

\section{Introduction}
\label{sec:Introduction}
Over-parameterization {helps} deep neural networks (DNNs) to generalize better in real-life applications \cite{brown2020language, huang2018gpipe, alex2019big,zagor2016}, despite providing them with the capacity to fit almost any set of random labels~\cite{Zhang17}. 
This~phenomenon has spawned a growing body of work that aims at identifying fundamental differences between real and random labels, such as in training time \cite{arpit2017closer, gu2019neural,han2018co,zhang2018generalized}, sharpness of the minima \cite{keskar2016large,neyshabur2017exploring}, dimensionality of layer embeddings \cite{ansuini2019intrinsic,collins2018detecting,pmlr-v80-ma18d},  and sensitivity \cite{arpit2017closer,novak2018sensitivity}, among other complexity measures  \cite{bartlett1998sample,Bartlett2017,neyshabur2015norm,neyshabur2017exploring}. 
While it is obvious that over-parameterization helps DNNs to interpolate any set of random labels, it is not immediately clear what DNNs \emph{learn} when trained in this setting.
The objective of this study is to provide a partial answer to this question. 

There are at least two reasons why answering this question is of value. 
First, 
in order to understand how DNNs work, it is imperative to observe how they behave under ``extreme'' conditions, such as when trained with labels that are entirely random. 
Since the pioneering work of~\cite{Zhang17}, several works have looked into the case of random labels. 
What distinguishes our work from others is that previous works aimed to demonstrate {differences} between real and random labels, highlighting the \emph{negative} side of training on random labels. By contrast, this work provides insights into what properties of the 
data distribution DNNs learn when trained on random labels.

Second, observing DNNs trained on random labels can explain phenomena that have been previously noted, but were poorly understood. In particular, by studying what is learned on random labels, we offer new insights into:  (1)~why DNNs exhibit critical stages 
\cite{achille2017critical, frankle2020early}, (2)~how earlier layers in DNNs generalize while later layers specialize \cite{ansuini2019intrinsic,arpit2017closer,cohen2018dnn,yosinski2014transferable}, (3)~why the filters learned by DNNs in the first layer seem to encode some useful structure when trained on random labels \cite{arpit2017closer}, 
and (4)~why pre-training on random labels can accelerate training in downstream tasks \cite{pondenkandath2018leveraging}. 
We show that even when controlling for simple explanations like weight scaling (which was not always accounted for previously), such curious observations continue to hold. 

The main contributions of this work are:
\begin{itemize}[leftmargin=*]
    \item We investigate DNNs trained with random labels and fine-tuned on disjoint image data with real or random labels, demonstrating unexpected positive and negative effects.
    \item We provide explanations of the observed effects. 
    We show analytically for convolutional and fully connected networks that an alignment between the principal components of the network parameters and the data takes place. 
    We demonstrate experimentally how this effect explains why pre-training on random labels helps. 
    We also show why, under certain conditions, pre-training on random labels can hurt the downstream task due to specialization at the later layers. 
    \item We conduct experiments verifying that these effects are present in several network architectures, including VGG16 \cite{simonyan2014very} and ResNet18-v2 \cite{He2016}, on CIFAR10 \cite{Krizhevsky09learningmultiple} and
    ImageNet ILSVRC-2012~\cite{deng2009imagenet}, across a range of hyper-parameters,
    such as the learning rate, initialization,
    number of training iterations, 
    width and depth.
\end{itemize}

In this work,
we do not use data augmentation as it provides a (weak) supervisory signal. 
Moreover, we use the terms ``positive'' and ``negative'' to describe the impact of what is learned with random labels on the downstream training, such as faster/slower training.
The networks reported throughout the paper are taken from a big set of experiments that we conducted using popular network architectures, datasets, and wide hyperparameter ranges. 
Experimental details are provided in Appendix~A and~B. 
We use boldface for random variables, 
small letters for their values,
and capital letters for matrices.

\subsection{Motivating example}

\begin{figure}[tb]
  \scriptsize \sffamily
  \centering
  \begin{minipage}{0.23\textwidth}
    \centering 
    1. Pre-training helps\\
    (real labels)
  \end{minipage}
  \begin{minipage}{0.23\textwidth}
    \centering
    2. Pre-training helps\\
    (random labels)
  \end{minipage}
  \begin{minipage}{0.23\textwidth}
    \centering
    3. Pre-training hurts\\
    (real labels)
  \end{minipage}
  \begin{minipage}{0.23\textwidth}
    \centering
    4. Pre-training hurts\\
    (random labels)
  \end{minipage}
  \includegraphics[width=0.23\textwidth,height=0.21\textwidth]{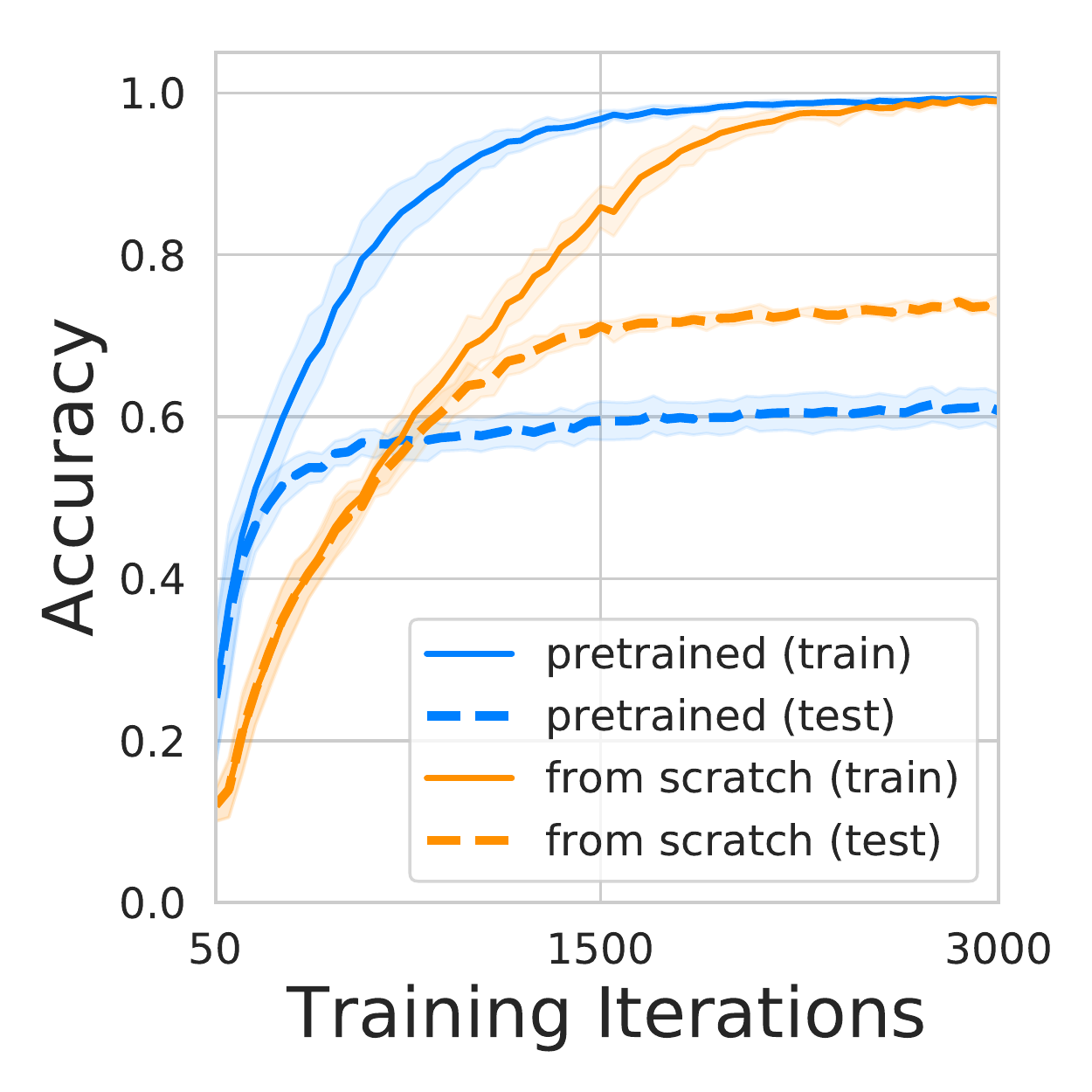}
  \includegraphics[width=0.23\textwidth,height=0.21\textwidth]{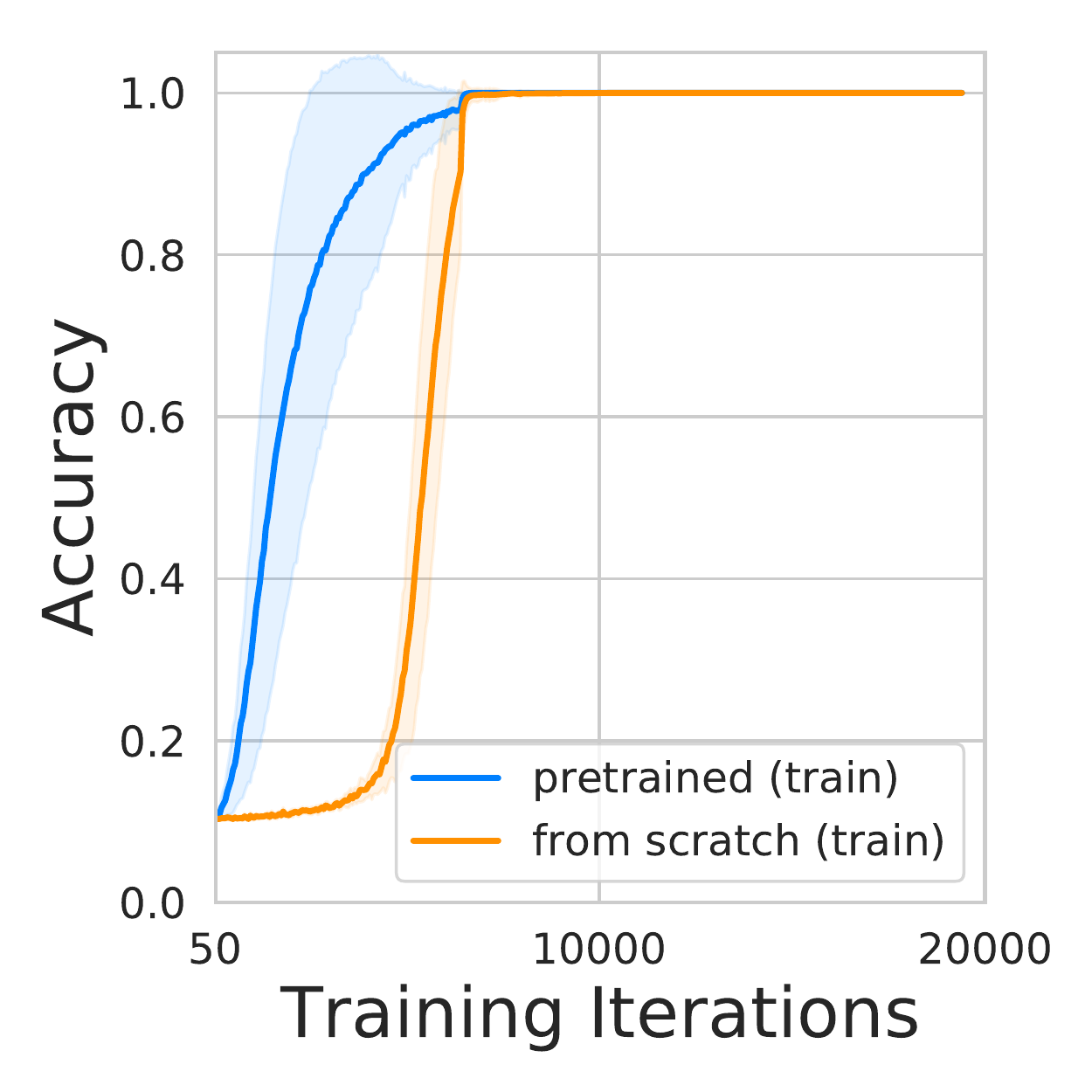}
  \includegraphics[width=0.23\textwidth,height=0.21\textwidth]{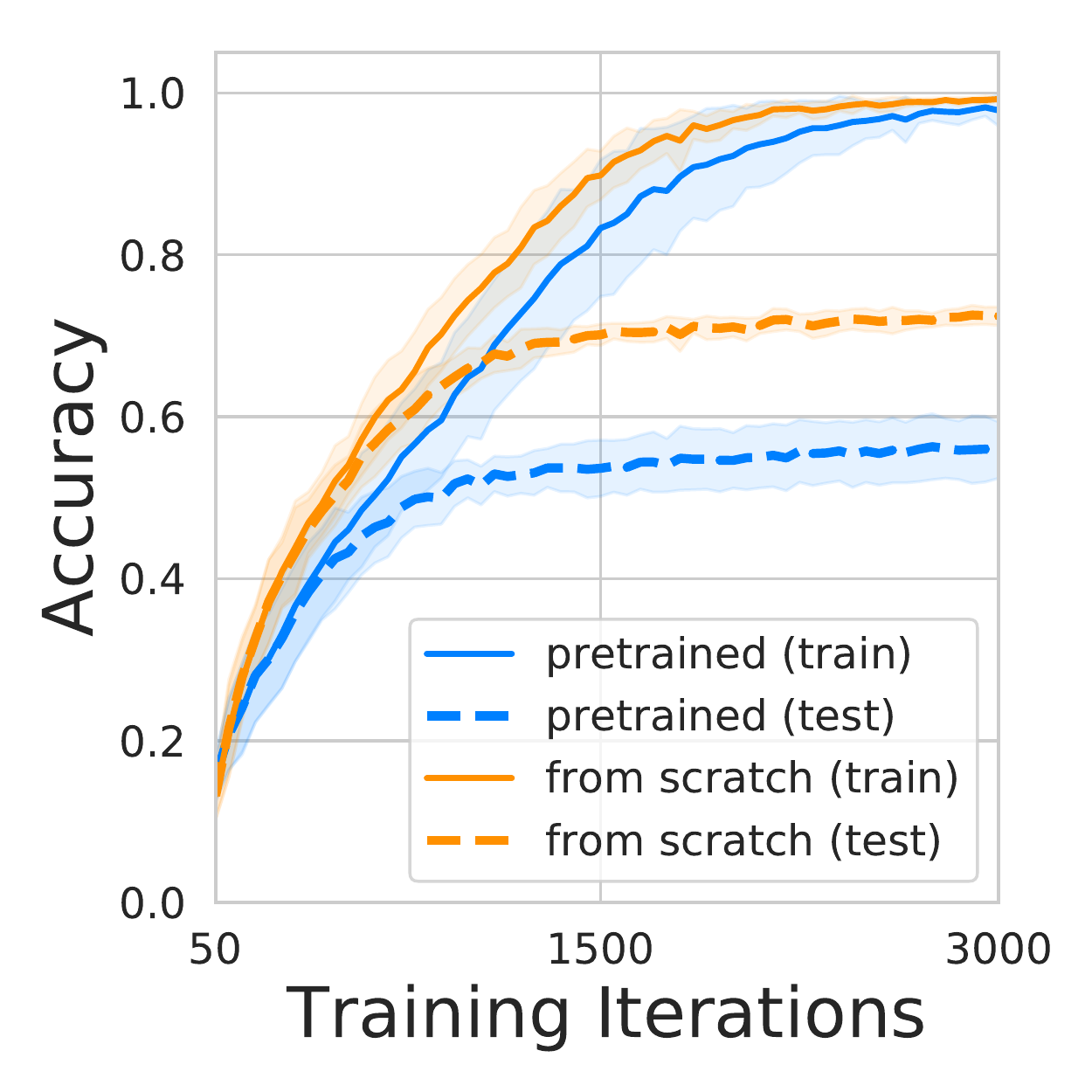}
  \includegraphics[width=0.23\textwidth,height=0.21\textwidth]{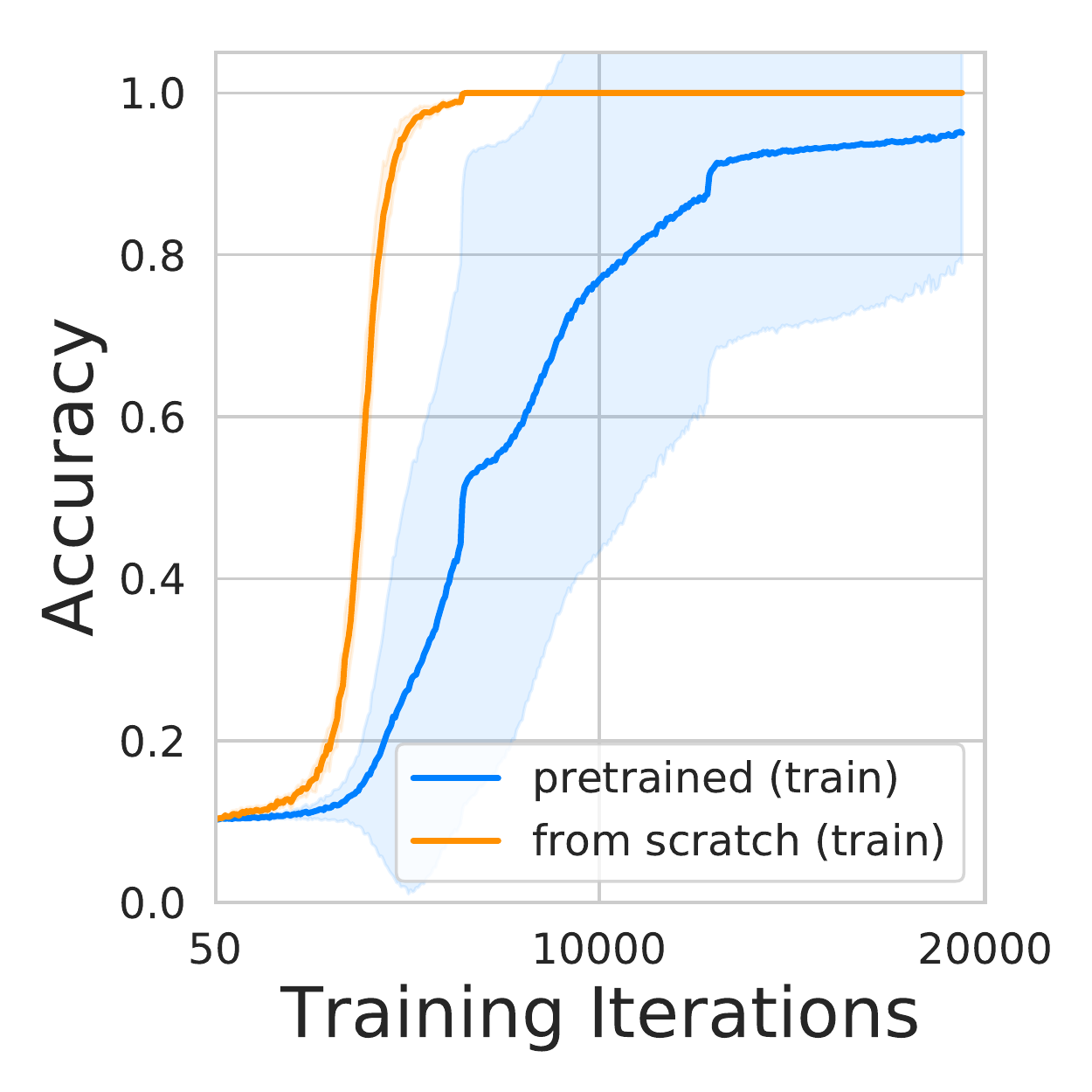}
  \caption{Pre-training on random labels may exhibit both positive ({\sc 1 \& 2}) and negative ({\sc 3 \& 4}) effects on the downstream fine-tuning depending on the setup.
  VGG16 models are pre-trained on 
  CIFAR10 examples with random labels and subsequently fine-tuned on the fresh 
  CIFAR10 examples with either real labels ({\sc 1 \& 3}) or 10 random labels ({\sc 2 \& 4}) using different hyperparameters.
  }
  \label{fig:image1}
\end{figure}

Figure~\ref{fig:image1} shows learning curves of the VGG16 architecture~\cite{simonyan2014very} pre-trained 
on 20k CIFAR10 examples~\cite{Krizhevsky09learningmultiple} with random labels (upstream) and fine-tuned on a disjoint subset of 25k CIFAR10 examples with either random or real labels (downstream). 
We observe that in this setup, pre-training a neural network on images with random labels accelerates training on a second set of images, both for real and random labels (positive effect). However, in the \emph{same setting} but with a different initialization scale and number of random classes upstream, a negative effect can be observed downstream: training becomes slower. 
We also observe a lower final test accuracy for real labels in both cases, which we are 
not explicitly investigating in this paper (and which has been observed before, e.g.\ in~\cite{frankle2020early}).

The fact that pre-training on random labels can accelerate training downstream has been observed previously, e.g.\ in \cite{pondenkandath2018leveraging}. However,  there is a ``simple'' property that can explain improvements in the downstream task: Because the cross-entropy loss is scale-sensitive, training the network tends to increase the scale of the weights \cite{neyshabur2017exploring}, which can increase the effective learning rate of the downstream task (see the gray curve in Figure~\ref{fig:PCA_init}).
To eliminate this effect, in all experiments we re-scale the weights of the network after pre-training to match their $\ell_2$ norms at initialization.
We show that even after this correction, pre-training on random labels positively affects the downstream task.
This holds for both VGG16 and ResNet18 trained on CIFAR10 and ImageNet (see Appendix B).

We show experimentally that some of the positive transfer is due to the second-order statistics of the network parameters. 
We prove that
when trained on random labels, the principal components of weights at the first layer are aligned with the principal components of data. 
Interestingly, this \emph{alignment effect} implies that the model parameters learned at the first layer can be summarized by a {one-dimensional mapping} between the eigenvalues of the data and the eigenvalues of the network parameters.
We study these mappings empirically and raise some new open questions.  
We also analyze how, under certain conditions, a competing effect of specialization at the later layers may hide the positive transfer of pre-training on random labels, which we show to be responsible for the negative effect demonstrated in Figure \ref{fig:image1}. 

To the best of our knowledge, the alignment effect has not been established in the literature before. 
This paper proves the existence of this effect and studies its implications.
Note
that while these effects are established for training on random labels,
we also observe them empirically for real labels. 

\subsection{Related work}
\label{sec:PreviousWorks}

A large body of work in the literature has developed techniques for mitigating the impact of \emph{partial} label noise, such as \cite{zhang2018generalized,Huber2011,natarajan2013learning,jiang2018mentornet,li2017learning,liu2015classification,sukhbaatar2014learning}. Our work is distinct from this line of literature because we focus on the case of purely random labels.

The fact that positive and negative learning takes place is related to the common observation that earlier layers in DNNs learn general-purpose representations whereas later layers specialize \cite{ansuini2019intrinsic,arpit2017closer,cohen2018dnn,yosinski2014transferable}. For random labels, it has been noted that memorization happens at the later layers, as observed by measuring the classification accuracy using activations as features \cite{cohen2018dnn} or by estimating the intrinsic dimensionality of the activations \cite{ansuini2019intrinsic}. We show that specialization at the later layers has a negative effect because it exacerbates the inactive ReLU phenomenon. Inactive ReLUs have been studied in previous works, which suggest that this effect could be mitigated by either increasing the width or decreasing the depth \cite{lu2019dying}, using skip connections \cite{douglas2018relu}, or using other activation functions, such as the leaky ReLU \cite{maas2013rectifier,he2015} or the exponential learning unit (ELU)~\cite{clevert2015fast}. 

For transfer learning, it has been observed that pre-training on random labels can accelerate training on real labels in the downstream task \cite{pondenkandath2018leveraging}. However, prior works  have  not accounted for simple effects, such as the change in first-order weight statistics (scaling), which increases when using the scale-sensitive cross-entropy loss \cite{neyshabur2017exploring}. Changing the norm of the weights alters the effective learning rate. \cite{Raghu19} investigated transfer from ImageNet to medical data and observed  that the transfer of first-order weight statistics provided faster  convergence. 
We show that even when taking the scaling effect into account, additional gains from second-order statistics are identified. 

Other works have considered PCA-based convolutional filters either as a model by itself without training~\cite{Gan15,Dehmamy2019}, as an initialization ~\cite{Ren16,Wagner13}, or to estimate the dimensionality of intermediate activations~\cite{collins2018detecting,montavon2011kernel}.
Note that our results suggest an initialization by \textit{sampling} from the data covariance instead of initializing the filters directly using the principal axes of the data.
\mbox{\citet{Ye2020deconv}} show that a ``deconvolution'' of data at input and intermediate layers can be beneficial. This deconvolution corresponds to a whitening of the data distribution, therefore aligning data with an isotropic weight initialization, which is related to a positive effect of alignment we observe in this paper.

In addition, there is a large body of work on unsupervised learning.
Among these, the \emph{Exemplar-CNN} method~\cite{Dosovitskiy14} can be seen as the limiting case of using random labels with infinitely many classes (one label per image) and large-scale data augmentation. 
In our study
we do \emph{not} use data augmentation since it provides a supervisory signal to the neural network that can cause additional effects.

\section{Covariance matrix alignment between network parameters and data}
\label{sec:PCA}

Returning to the motivating example in Figure \ref{fig:image1}, we observe that pre-training on random labels can improve training in the downstream tasks for both random and real labels. This improvement is in the form of \emph{faster training}. In this section, we explain this effect. We start by considering the first layer in the neural network, and extend the argument to later layers in Section~\ref{subsec:DeepNetworks}.

\subsection{Preliminaries}\label{sect::pca_preliminaries}
Let $\calD$ be the probability distribution over the instance space $\mathcal{X}\subseteq\BR^d$ and $\mathcal{Y}$ be a  finite target set.
We fix a network architecture, a loss function, a  learning rate/schedule, and a 
distribution of weights for random initialization. Then, ``training on 
random labels'' corresponds to the following procedure: We randomly sample i.i.d.\ instances $ \textbf{x}_1,..., \textbf{x}_N \sim \calD$,
and i.i.d.\ labels $\textbf{y}_1,...,\textbf{y}_N\in \mathcal{Y}$ independently of each other. 
We also sample the initial weights of the neural network, and
 train the network on the data $\{(\textbf{x}_i, \textbf{y}_i)\}_{i=1,\ldots,N}$ for $T$ training iterations using stochastic gradient descent (SGD). During training, the
weights are \emph{random variables} due to the randomness of the initialization and the training sample. Hence, we can speak of their statistical properties, such as their first and second moments. 

In the following, we are interested in layers that are convolutional or fully connected. We assume that the output of the $k$-th neuron in the first layer can be written as: $f_k(x)=g(\langle w_k,\,x\rangle +b_k)$ for some activation function $g$. 
We write $\mu_x=\mathbb{E}[\textbf{x}]$ and observe that since the covariance matrix $\Sigma_x = \mathbb{E}[(\textbf{x}-\mu_x)\cdot (\textbf{x}-\mu_x)^T]$ is symmetric positive semi-definite,
there exists an orthogonal de\-com\-po\-si\-tion $\BR^d = V_1 \oplus ... \oplus V_r$
such that $V_i$ are eigenspaces to $\Sigma_x$ with distinct eigenvalues $\sigma_i^2$.

\begin{definition}[Alignment]\label{def:alignment} A symmetric matrix $A$ is said to be \emph{\textbf{aligned}} with a symmetric matrix $B$ if each eigenspace of $B$ is a subset of an eigenspace of $A$. If $A$ is aligned with $B$, we  define for each eigenvalue
of $B$ with eigenspace $V\subseteq \BR^d$ the \emph{\textbf{corresponding}} eigenvalue of 
$A$ as the one belonging to the eigenspace that contains $V$.
\end{definition}

If $A$ and $B$'s eigenspaces are all of dimension 1 (which
is true except for a Lebesgue null set in the space of symmetric matrices),
 ``$A$ is aligned with $B$'' becomes equivalent to the assertion that they share the same eigenvectors. 
However, the
relation is not symmetric in general (e.g.\ only scalar multiples of the identity matrix $I_d$ are aligned with $I_d$, but $I_d$ is aligned with any symmetric matrix).

\subsection{Alignment for centered Gaussian inputs}
\label{subsec:Gaussian}

\begin{proposition}
\label{prop:1}
Assume the instances $\textbf{x}$ are drawn i.i.d. from $\mathcal{N}(0,\Sigma_x)$ and the initial weights in
the first layer are drawn from an isotropic distribution (e.g. the standard Gaussian). Let $\textbf{w}\in\mathbb{R}^d$ be a random variable whose value is drawn uniformly at random from the set of weights in the first layer after training using SGD with random labels (see Section \ref{sect::pca_preliminaries}). Then: (1) $\mathbb{E}[\textbf{w}]=0$ and (2) $\Sigma_w = \mathbb{E}[\textbf{w}\cdot\textbf{w}^T]$ is aligned with the covariance matrix of data $\Sigma_x$.
\end{proposition}

\begin{proof}
The proof exploits symmetries: The input, initialization, and gradient descent are invariant under
the product of the orthogonal groups of the eigenspaces of $\Sigma_x$, so the distribution of
weights  must have the same invariance properties. The full proof is given in Appendix~C. 
\end{proof}

Proposition~\ref{prop:1} says that independently of many settings (e.g.\ number of random labels, network architecture, learning rate or schedule), the eigenspaces of $\Sigma_w\in\mathbb{R}^{d\times d}$
are given by the eigenspaces of $\Sigma_x\in\mathbb{R}^{d\times d}$. Hence, the only information needed to fully determine $\Sigma_w$ is a function $f$ that maps the eigenvalues $\sigma_i^2$ of $\Sigma_x$ 
to the corresponding eigenvalues $\tau_i^2$ of $\Sigma_w$. 
Note that the argument of the proof of Proposition \ref{prop:1} also apply to the case of a single training run of an infinitely wide network in which the layers are given by weight vector distributions, see e.g.~\cite{sirignano2018}. 
For finite networks, in practice, one would only be able to approximate $\Sigma_w$ based on several independent training runs.

Next, we present experimental evidence that first-layer alignment actually takes place, not just for Gaussian input with random labels, but also in real image datasets with random labels, and even when training on real labels using convolutional networks. The intuition behind this result for real labels is that small patches in the image (e.g. $3\times 3$) are nearly independent of the labels. Before we do that, we introduce a suitable measure of alignment that we use in the experiments.

\begin{definition}
For two positive definite matrices $A, B$, the ``misalignment''  
$M(A, B)$ is defined as: 
\begin{equation}
   M(A, B) := \inf_{\Sigma\succ 0\ \hbox{ aligned with}\ A} 
   \big\{\tfrac{1}{2}\textbf{tr}(\Sigma^{-1}B+B^{-1}\Sigma)-d\big\}
   \label{eq:DefinitionA}
\end{equation}
\end{definition}
The rationale behind this definition of misalignment is presented in Appendix~D. 
In particular, it can be shown that for any $A, B\succ 0$, we have $M(A, B)\ge 0$ with equality if and only if $B$ is aligned with $A$. In addition, $M(A,B)$ is continuous in $B$ and satisfies desirable  equivariance and invariance properties and can be computed in closed form by
$M(A, B) + d = \sum_{i=1}^r \sqrt{ \textbf{tr}(B|_{V_i}) \cdot \textbf{tr}(B^{-1}|_{V_i})}$
where $V_1 \oplus ... \oplus V_r$ is the orthogonal decomposition of $\BR^d$ into eigenspaces of $A$,
and $B|_{V_i}$ is the linear map $V_i\arrow V_i, \textbf{v}\mapsto pr_i(B(\textbf{v}))$ with $pr_i$ the 
orthogonal projection $\BR^d\arrow V_i$.

\begin{SCfigure}[2]
  \centering
  \includegraphics[width=0.5\textwidth]%
    {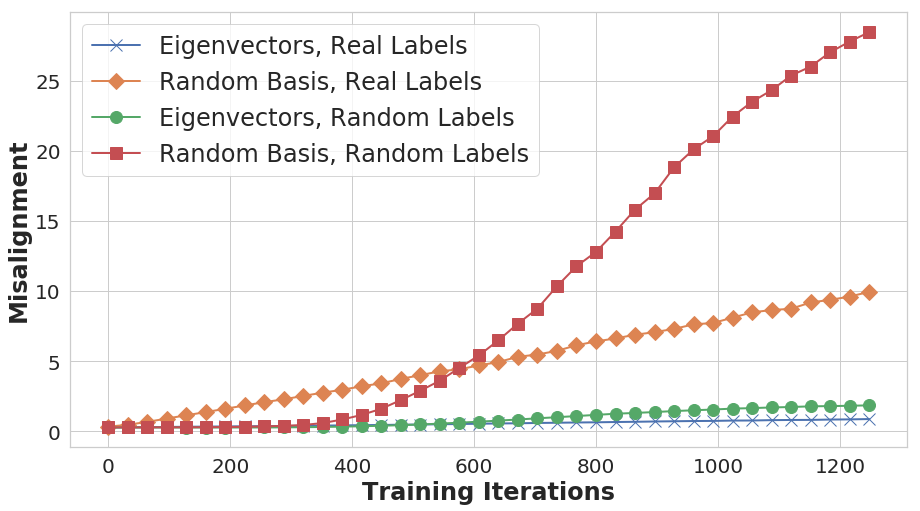}
  \caption{Plots of the misalignment scores of the filters of the first layer in a two-layer neural network (256 convolutional filters, 64 fully-connected nodes) when trained on CIFAR10 with either real or random labels. Throughout training, misalignment scores between the filters of the first layer and the data remain very small compared to those between filters and a random orthonormal basis.}
  \label{fig:Experiment_alignment}
\end{SCfigure}

Figure \ref{fig:Experiment_alignment} displays the misalignment scores between the covariance of filters at the first layer with the covariance of the data (patches of images). For comparison, the misalignment scores with respect to some random orthonormal basis are plotted as well. 
As predicted by Proposition \ref{prop:1}, the weight eigenvectors stay aligned to the data eigenvectors but not to an arbitrary random basis. 

\newcommand{\vw}{$v_w$}
\newcommand{\vxa}{$\overline{v_x}$}
\newcommand{\vxi}{$v_x$}
\begin{figure}[tb]
\centerline{
\small
\raisebox{-0.55\height}{\includegraphics[height=0.22\textwidth]{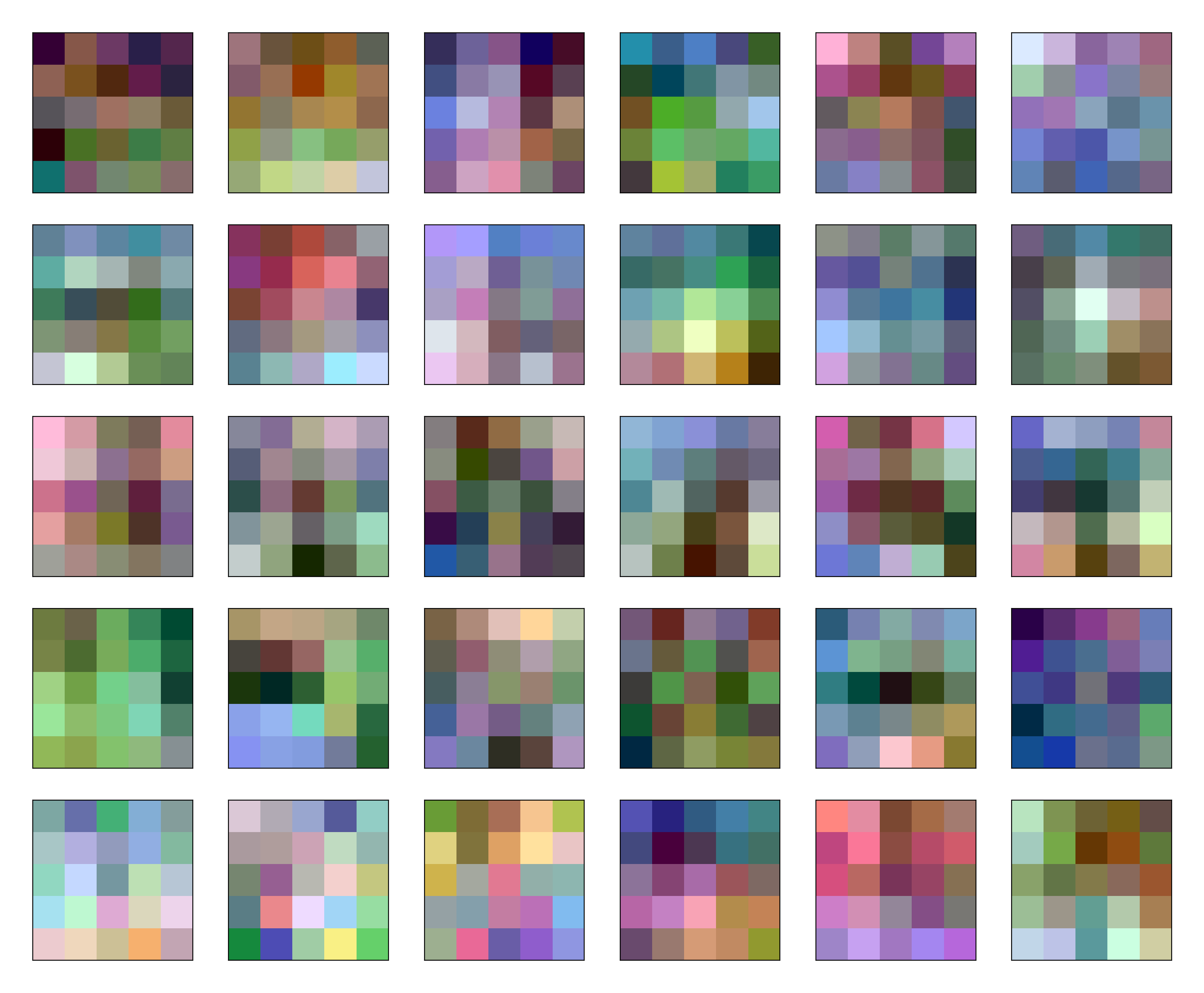}}\hfill
{\setlength{\extrarowheight}{5pt}
\def\myheight{0.038\textwidth}
\newcommand{\cpi}[2]{$\vcenter{\hbox{\includegraphics[height=\myheight]{images/pca/pca-#1-#2.png}}}$}
\newcommand{\ncpi}[2]{$\vcenter{\hbox{\includegraphics[height=\myheight]{images/pca_new/pca-#1-#2.png}}}$}
\begin{tabular}{r@{\hspace{1em}}c@{ $\sim$ }c}
 \# & \vw{}  &  \vxi{}\\
4 & \ncpi{3}{0} & \ncpi{3}{1}  \\
6 & \ncpi{5}{0} & \ncpi{5}{1}  \\
7 & \ncpi{6}{0} & \ncpi{6}{1}  \\
8 & \ncpi{7}{0} & \ncpi{7}{1}  \\  
10& \ncpi{9}{0} & \ncpi{9}{1} 
\end{tabular}\hfill
\begin{tabular}{r@{\hspace{1em}}c@{ $\sim$ }c@{ $=$ }l}
\#  & \vw{}  & \vxa{}  & \vxi{} $+\ldots$\\
1 & \ncpi{0}{0} & \ncpi{0}{1} & \ncpi{0}{2} $+$ \ncpi{0}{3}  \\
2 & \ncpi{1}{0} & \ncpi{1}{1} & \ncpi{1}{2} $+$ \ncpi{1}{3}  \\
3 & \ncpi{2}{0} & \ncpi{2}{1} & \ncpi{2}{3} $+$ \ncpi{2}{2}  \\
5 & \ncpi{4}{0} & \ncpi{4}{1} & \ncpi{4}{3} $+$ \ncpi{4}{2}  \\
9 & \ncpi{8}{0} & \ncpi{8}{1} & \ncpi{8}{3} $+$ \ncpi{8}{2} 
\end{tabular}%
}
}
  \caption{Visualization of covariance alignment. {\sc left}: Random selection of WRN-28-4 first-layer convolutional filters (CIFAR10, random labels). {\sc center/right}: Eigenvectors \vw{} of $\Sigma_w$ 
  with largest eigenvalues (rank in column `\#') and eigenvectors \vxi{} of $\Sigma_x$ with $\langle v_x, v_w\rangle > 0.4$. {\sc center}: Cases where one \vxi{} matches. {\sc right}: Cases where two \vxi{}  and their weighted combination \vxa{} match.}
  \label{fig:filters}
\end{figure}

For image data, we can also visualize the alignment of $\Sigma_w$ to $\Sigma_x$. 
Figure~\ref{fig:filters} shows results based on 70 wide ResNet models \cite{zagor2016} trained on CIFAR10 with random labels.
For better visualization, we use a $5\times 5$ initial convolution here. The left part of the figure shows a random selection of some of the 70$\cdot$64
convolution filters. From the filters we estimate $\Sigma_w$ and 
compute the eigenvectors \vw{}, then  visualize the ten \vw{} with the largest eigenvalues. From the image data we compute the 
patch data covariance $\Sigma_x$ and its eigenvectors \vxi{}. 
We show the data eigenvectors for which the inner product with the filter eigenvectors exceeds a
threshold and the weighted sum \vxa{} of these if there are multiple such \vxi{}. 
(See Appendix D.1. for why this is expected to occur as well.)

The visual similarity between the \vw{}  and the \vxa{} illustrates the predicted covariance alignment. Note that this alignment is visually non-obvious when looking at individual filters as shown on the left.

\subsection{Mapping of eigenvalues}\label{subsect::mapping_eigenvalues}
As stated earlier, Proposition \ref{prop:1} shows that, on average, the first layer effectively learns a 
function which maps each eigenvalue of $\Sigma_x$ to the \emph{corresponding} eigenvalue of $\Sigma_w$ (see
Definition \ref{def:alignment}). In this section, we examine the shape of this function.

Since, in practice, we will only have \textit{approximate} alignment due to the finiteness of the number of inputs and weights, non--Gaussian inputs, and correlations between overlapping patches, we extend the definition of $f(\sigma)$.
Such an extension is based on the following identity~\eqref{eq:general_tau}:
For $\Sigma_x\in\mathbb{R}^{d\times d}$ let $\textbf{v}_i$ be an eigenvector of length 1 
with eigenvalue $\sigma_i^2$. If $\Sigma_w$ is aligned with 
$\Sigma_x$, $\textbf{v}_i$ is also an eigenvector of $\Sigma_w$ and the corresponding eigenvalue $\tau_i^2$ is:
\begin{equation}\label{eq:general_tau}
  \tau_i^2 = \textbf{v}_i^T \Sigma_w \textbf{v}_i 
    =   \textbf{v}_i^T \BE[(\textbf{w}-\mu_w)(\textbf{w}- \mu_w)^T] \textbf{v}_i  = \BE[ \langle 
    \textbf{w} - \mu_w, \textbf{v}_i\rangle ^2], 
\end{equation}
which is the variance of the projection of the weight vectors onto the principal axis $\textbf{v}_i$.
We can take this as the definition of $\tau_i$ in the general case, since this formulation can be applied even when we have an imperfect alignment between the eigenspaces.
\begin{definition}
Given two positive definite symmetric $d\times d$ matrices $\Sigma_x, \Sigma_w$, such that 
$\Sigma_w$ is aligned with $\Sigma_x$ or $\Sigma_x$ has $d$ distinct eigenvalues. 
Let $\sigma_1^2,\sigma_2^2,...$ be the eigenvalues of $\Sigma_x$ with corresponding 
eigenvectors $\textbf{v}_1,\textbf{v}_2,...$ of length 1, we define 
the transfer function from $\Sigma_x$ to $\Sigma_w$ as
\begin{equation}
   f: \{\sigma_1,\sigma_2,...\} \arrow \BR, \ \ \sigma_i \mapsto \sqrt{\textbf{v}_i^T \Sigma_w \textbf{v}_i}
\end{equation}
\end{definition}

In practice, the eigenvalues are distinct almost surely so every eigenvalue of the data  has a unique corresponding eigenvector of length 1 (up to $\pm$) and the function $f(\sigma)$ is well-defined. 

\begin{figure}[tb]
  \begin{center}
  \centerline{
  \includegraphics[width=0.33\textwidth,height=3.4cm]{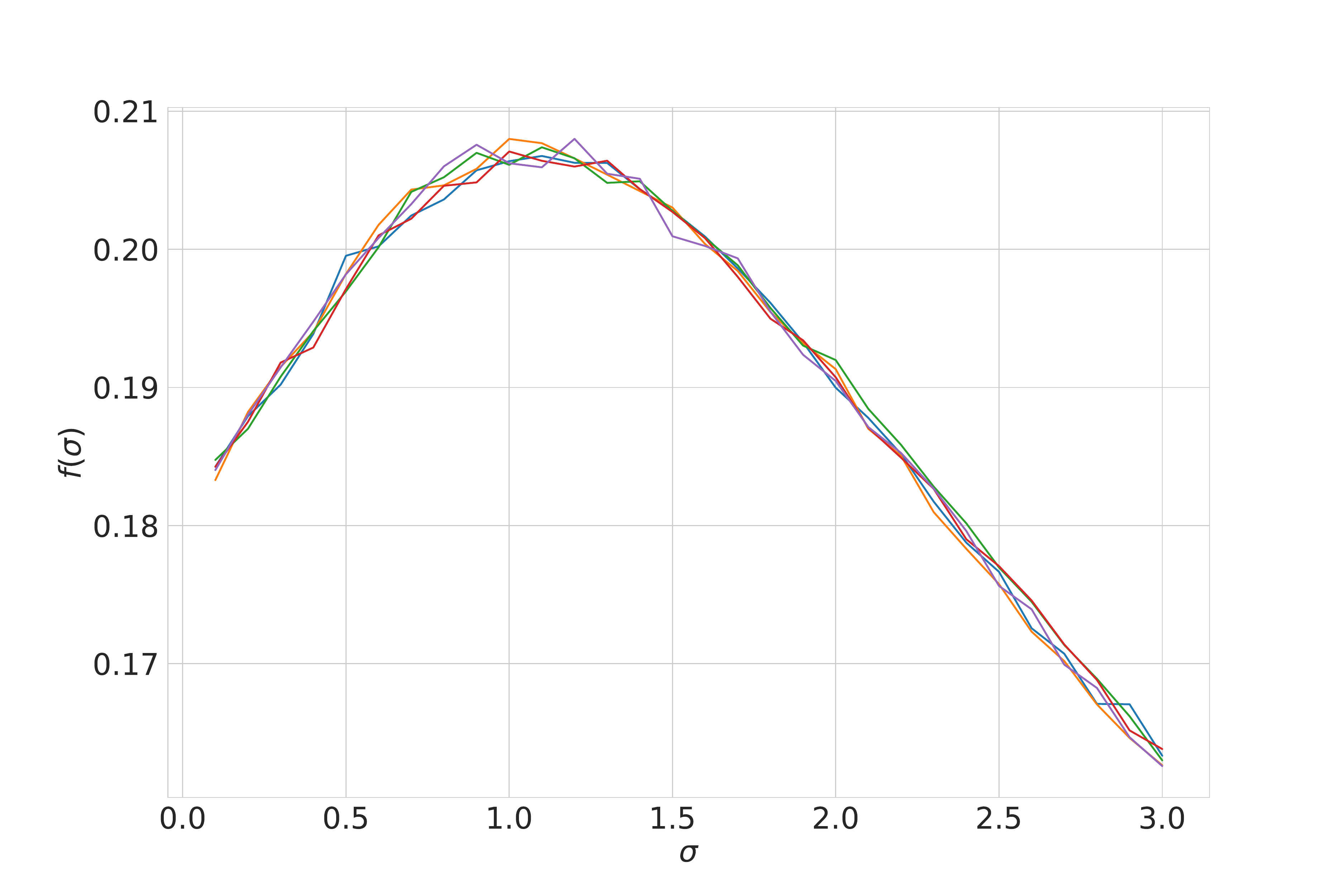} 
  \includegraphics[width=0.33\textwidth]{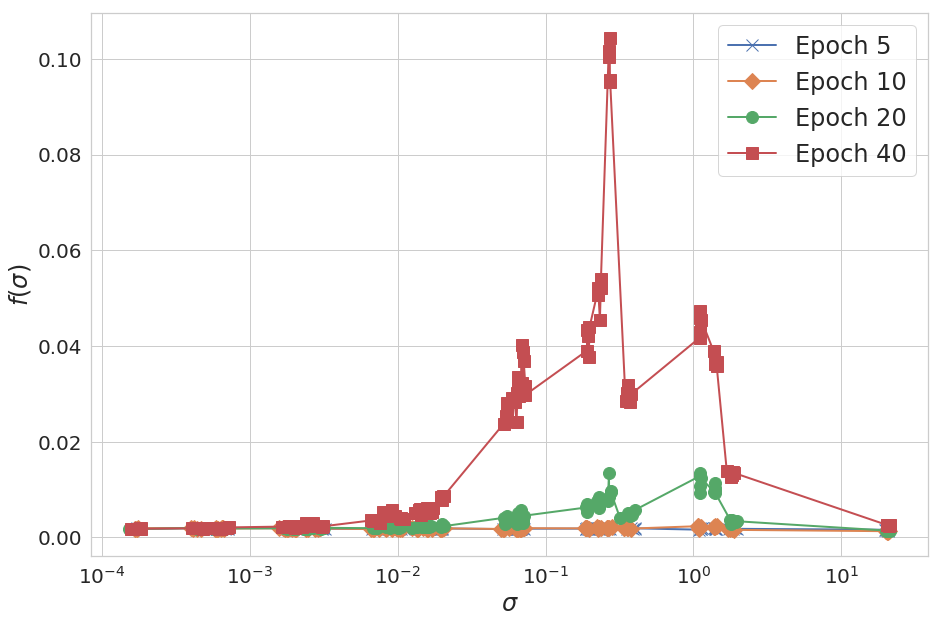}
  \includegraphics[width=0.33\textwidth]{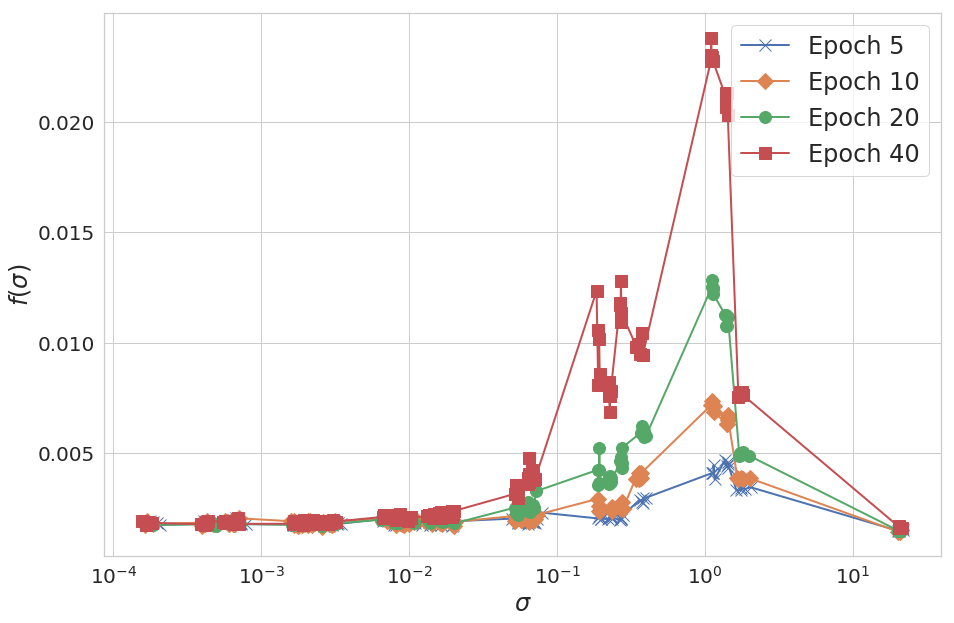}
  }
  \caption{
  {\sc left:} $f(\sigma)$ for synthetic data $\calN(0,\mathrm{diag}(0.1, 0.2, \dots, 3.0))$ in fully-connected neural networks with two layers of size 256. The graph is approximately continuous and of a regular structure: increasing, then decreasing. {\sc center,right:}
  $f(\sigma)$ that results from training a 2-layer convolutional network (256 filters followed by 64 fully-connected nodes) on CIFAR10 for random ({\sc center}) and real labels ({\sc right}) after 5, 10, 20, and 40 epochs ($\sim$195 training iterations per epoch).
  } 
  \label{fig:Experiment2_f_sigma}
  \end{center}
\end{figure}

Using this definition of $f(\sigma)$, we can now look at the shape of the function for concrete examples.
Here, we train a simple fully connected network
50 times, collect statistics, and plot the corresponding mapping between each eigenvalue $\sigma_i$ of $\Sigma_x$ with the corresponding $\tau_i$ in (\ref{eq:general_tau}) (see Appendix E.1 for more details). 
The result is shown in Figure~\ref{fig:Experiment2_f_sigma} ({\sc left}).
In general, $f(\sigma)$ on synthetic data looks smooth and exhibits a surprising structure: the function first increases before it decreases. Both the decreasing part or the increasing part may be missing depending on the setting (e.g.\ dimensionality of data and network architecture) but we observe the same shape of curves in all experiments.
For real data (CIFAR10), Figure~\ref{fig:Experiment2_f_sigma} ({\sc center/right}) shows that the function $f(\sigma)$ appears to have a similar shape (increasing, then decreasing, but less smooth) for training with both real and random labels. 

We interpret (without a formal argument) this surprisingly regular shape of $f(\sigma)$ to be the result of two effects:
(1)~Larger eigenvalues $\sigma_i$ lead to larger effective learning rate in gradient descent, which leads in turn to larger corresponding $\tau_i$, hence the increasing part of $f$.
(2)~Very large eigenvalues $\tau_i$ would dominate the output of the layer, masking the contribution of other components. Backpropagation compensates for this effect to capture more of the input signal. This leads to the decreasing part of $f$ for higher $\sigma_i$.
(See also Appendix E.1)

\subsection{Covariance alignment and eigenvalue mapping explains positive transfer experimentally}
\label{subsect:explain_pos_transfer}
\begin{figure}[tb]
  \centerline{
\includegraphics[width=0.33\textwidth]{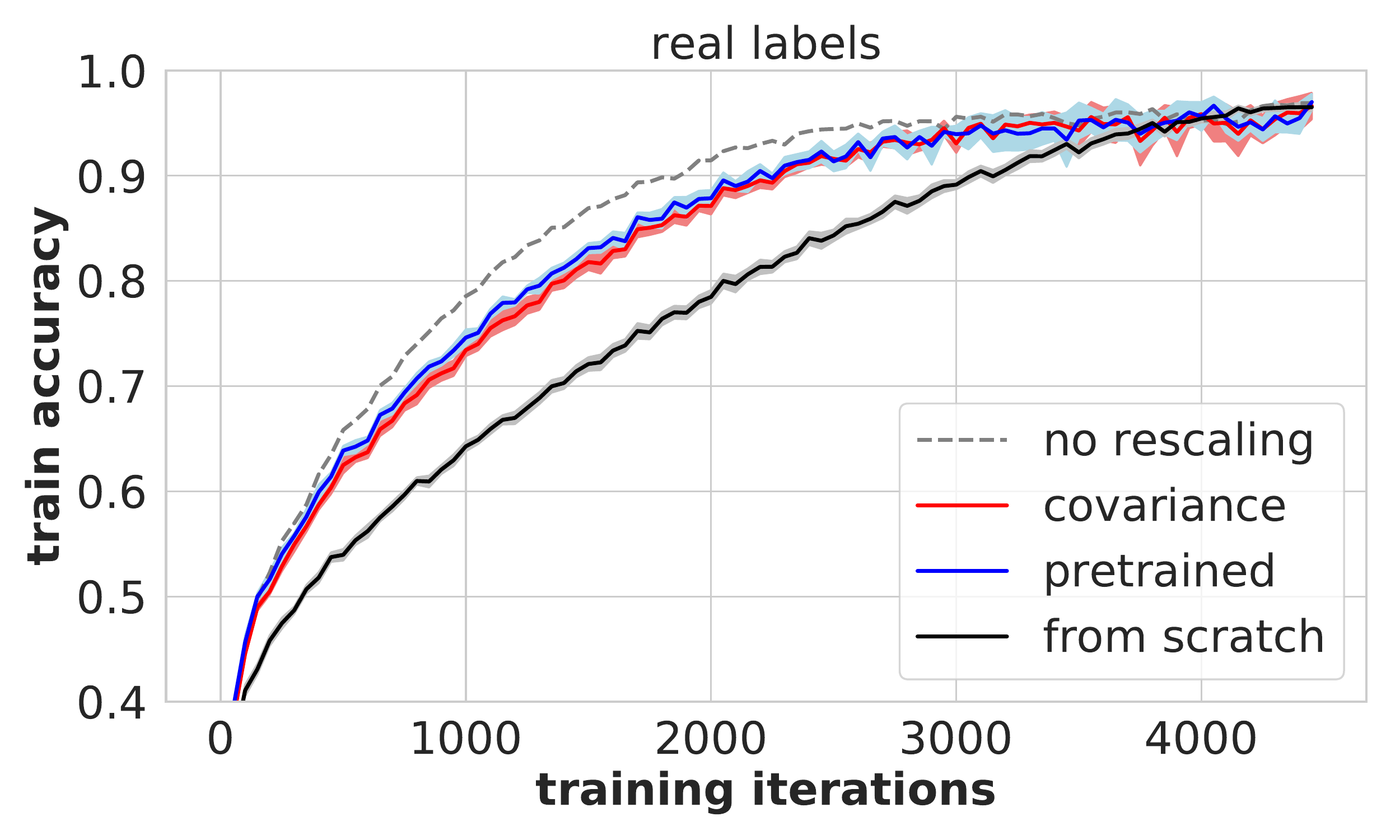}%
\includegraphics[width=0.33\textwidth]{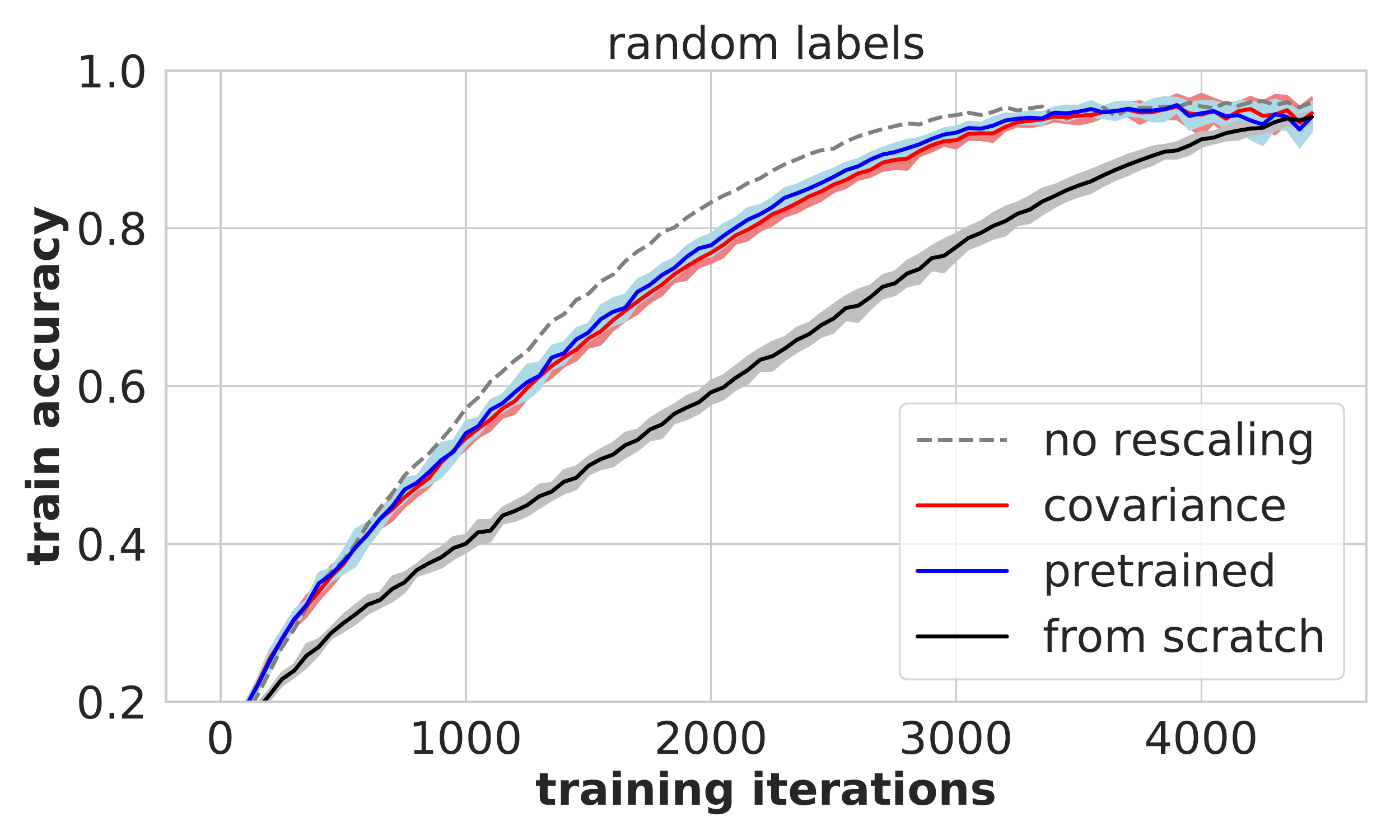}%
\includegraphics[width=0.33\textwidth]{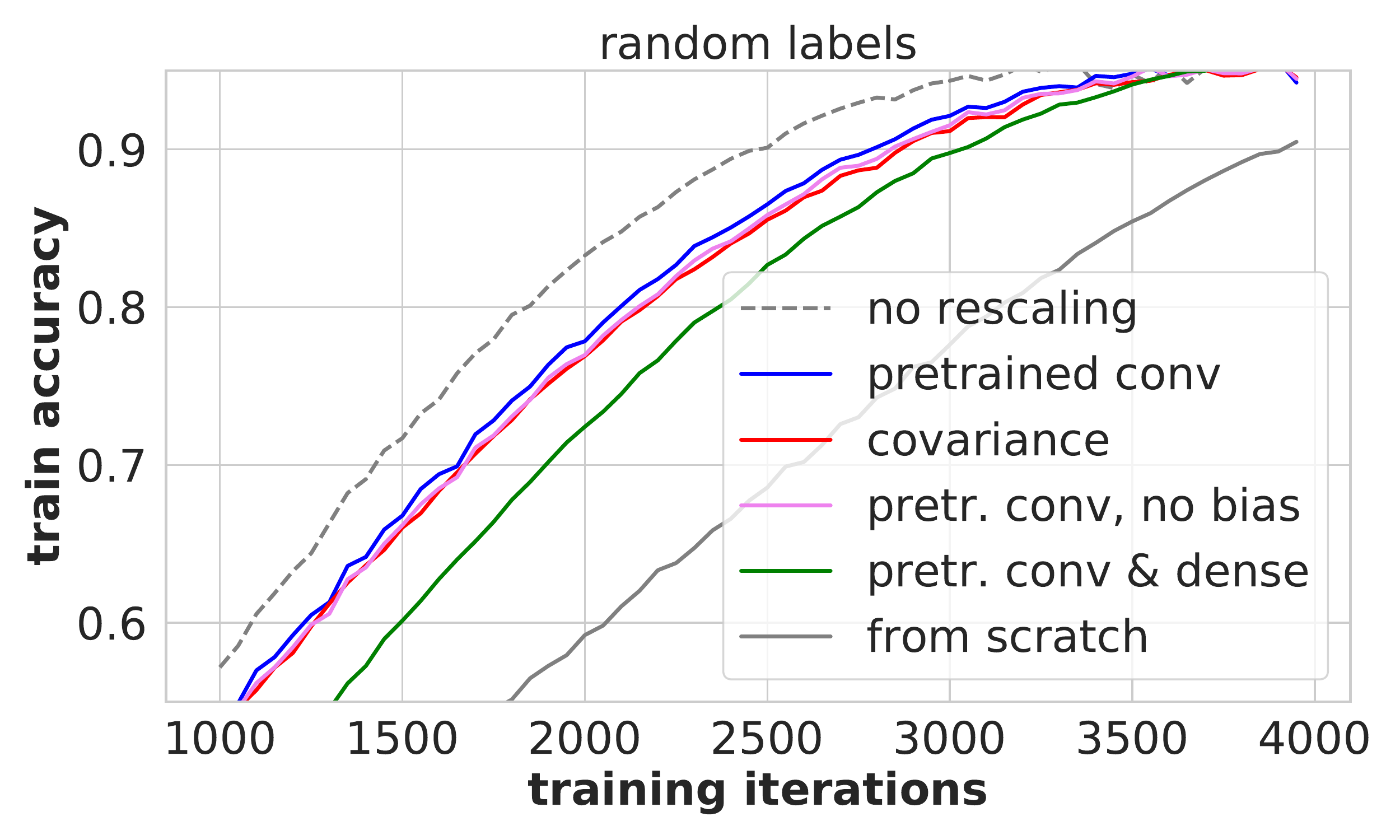}
  }
  \caption{
  Training accuracy for learning real labels ({\sc left}) and random labels ({\sc middle}, zoomed in: {\sc right}) on CIFAR10 with a simple CNN (one $3\times 3$ convolution, one hidden dense layer). Randomly initializing the convolutional filters from a learned covariance reproduces the effect of pre-training within measurement error (red and pink lines are almost indistinguishable in the right plot).      }
     \label{fig:PCA_init}
\end{figure}

To connect the alignment effect to the faster learning downstream, we conduct the following experiment.
Suppose that instead of pre-training on random labels, we sample from the Gaussian approximation of the filters in the first layer that were trained on random labels. 
In a simple CNN on CIFAR10, consisting of one convolutional and one fully connected layers, the gain in downstream task is almost fully recovered, as shown in Figure~\ref{fig:PCA_init}. 
The gray curves show the raw effect of pre-training, but this contains the 
scaling effect.
To eliminate the latter effect, we always re-scale the weights after pre-training to match their $\ell_2$ norm at initialization (using $\ell_1$ norm gives similar results).
Recovering the transfer effect in this way implies that the positive transfer is mainly due to the second-order statistics of the weights in the first layer, which, by Proposition~\ref{prop:1}, are fully described by the alignment of principal components in combination with the shape of the function $f(\sigma)$. 

Note that the combined results presented up to here indicate that both analytically and experimentally the following seems to be true: Given the training data of a neural network that is trained with random labels, we can predict the second-order statistics of its first layer weights  up to a one-dimensional scaling function $f(\sigma)$ of the eigenvalues and this function has a surprisingly regular shape. If further investigation of $f$ may lead us to understand its shape (which we regard as an interesting area of research), we could predict these after-training statistics perfectly by only gathering the data statistics. 

\subsection{Deeper layers}
\label{subsec:DeepNetworks}
So far, we have only discussed effects in the first layer of a neural network.
However, in Figure 6 we show that transferring more of the layers from the model pre-trained on
random labels improves the effect considerably. We now turn to generalizing Section
\ref{subsect:explain_pos_transfer} to the multi-layer case, i.e. we reproduce this effect with
weight initializations computed from the input distribution.

For the first layer, we have seen that we can reproduce the effect of training on random labels
by randomly sampling weights according to the corresponding covariance matrix $\Sigma_w$, which
in turn is given by the same eigenvectors $e_1,...,e_d$ as the data covariance $\Sigma_x$, 
and a set of new eigenvalues $\tau_1^2,...,\tau_d^2$. 
So, if we can approximate the right (or good) eigenvalues, 
we can directly compute an initialization that results in faster training in a subsequent
task of learning real labels. See the first two accuracy columns in Table 1 for
results in an example case, and Appendix E for different choices of $\tau_1,...,\tau_d$ (it turns
out different reasonable choices of the $\tau_i$ give all results very similar to Table 1).

We can then iterate this procedure also for the next (fully connected or convolutional) layers.
Given the filters for the earlier layers $L_1,...,L_{k-1}$, for each training image
we can compute the output after layer $L_{k-1}$, which becomes the input to the layer $L_k$.
Treating this as our input data, we determine the
eigenvectors $e_1, e_2,...$ of the corresponding data covariance matrix. 
Then we compute the $d$ most important directions and use $\tau_1 e_1, \tau_2 e_2,...,
\tau_d e_d$ (with the same assumed $\tau_1, \tau_2,...$ as before) as our constructed filters.
(Alternatively, we can sample according to the covariance
matrix given by the eigenvectors $e_i$ and eigenvalues $\tau_1^2,...,\tau_d^2,0,...,0$, which
gives again essentially the same results, compare Table 2 and 4 in Appendix E.3.)

Applying this recipe to a CNN with three convolutional layers and 
one fully connected layer, we see that this indeed gives initializations that 
become better when applied to 1,2, and 3 layers, see Table \ref{table:experiment_3conv}.
The performance after applying this approach to all three convolutional 
layers matches the performance of transferring the first three layers of a network trained on random labels.
See Appendix E.3 for details.

\begin{SCtable}
\caption{Training and test accuracy on subsets of CIFAR10 
of the initialization procedure described in Section~\ref{subsec:DeepNetworks} on the layers of a simple convolutional network. Both training and test accuracies improve with the number of layers that are initialized in this way.}
\label{table:experiment_3conv}
\begin{tabular}{@{}rlcccc@{}}
\hline 
  & & \multicolumn{4}{c}{Convolutional layers sampled} \\
Iterations & Data &$\{\}$ & $\{1\}$ & $\{1,2\}$ & $\{1,2,3\}$ \\ \hline 
100 & Train  & 0.31  & 0.34  & 0.38 & 0.41 \\
    & Test   & 0.31  & 0.33  & 0.37 & 0.40 \\
1000 & Train & 0.58  & 0.61  & 0.67 & 0.68 \\
     & Test  & 0.53  & 0.55  & 0.56 & 0.56 \\
\hline
\end{tabular}
\end{SCtable}

\begin{figure}[tb]
  \centering
  \scriptsize \sffamily
  \begin{minipage}{0.33\textwidth}
    \centering 
    Simple CNN    
  \end{minipage}
  \begin{minipage}{0.33\textwidth}
    \centering
    VGG16
  \end{minipage}
  \begin{minipage}{0.33\textwidth}
    \centering
    ResNet18
  \end{minipage}\\
  \includegraphics[width=0.33\textwidth]{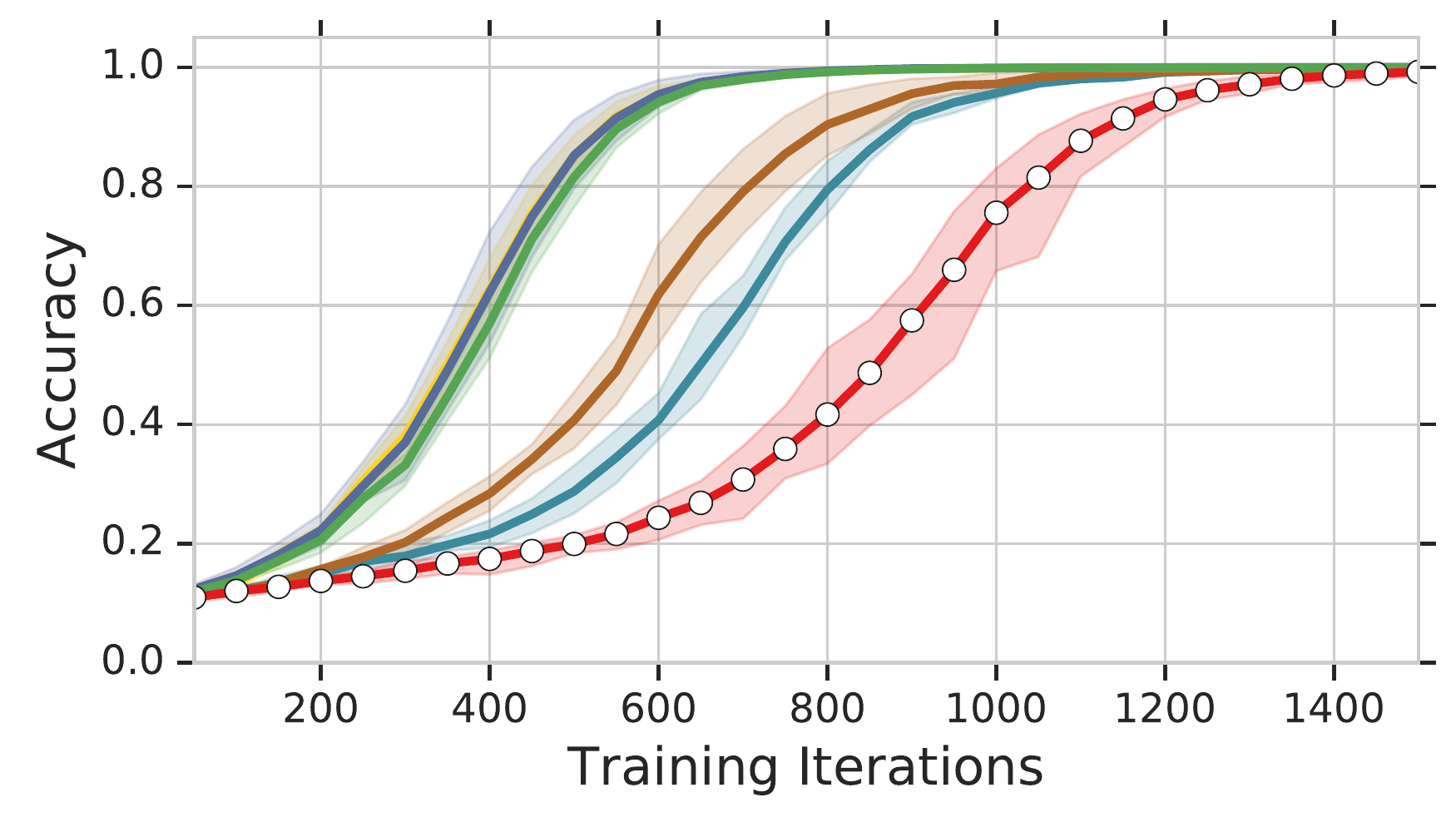}
  \includegraphics[width=0.33\textwidth]{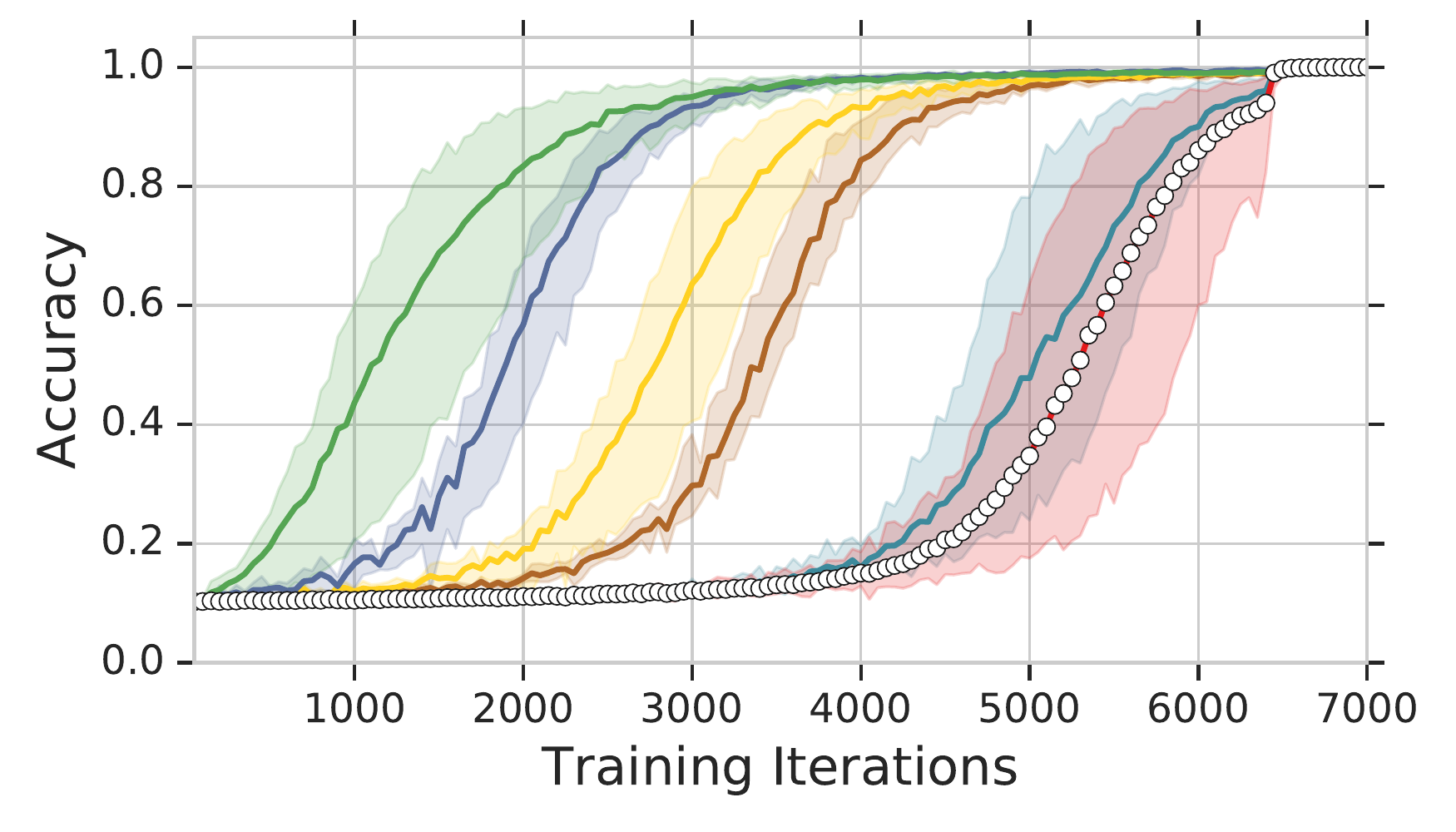}
  \includegraphics[width=0.33\textwidth]{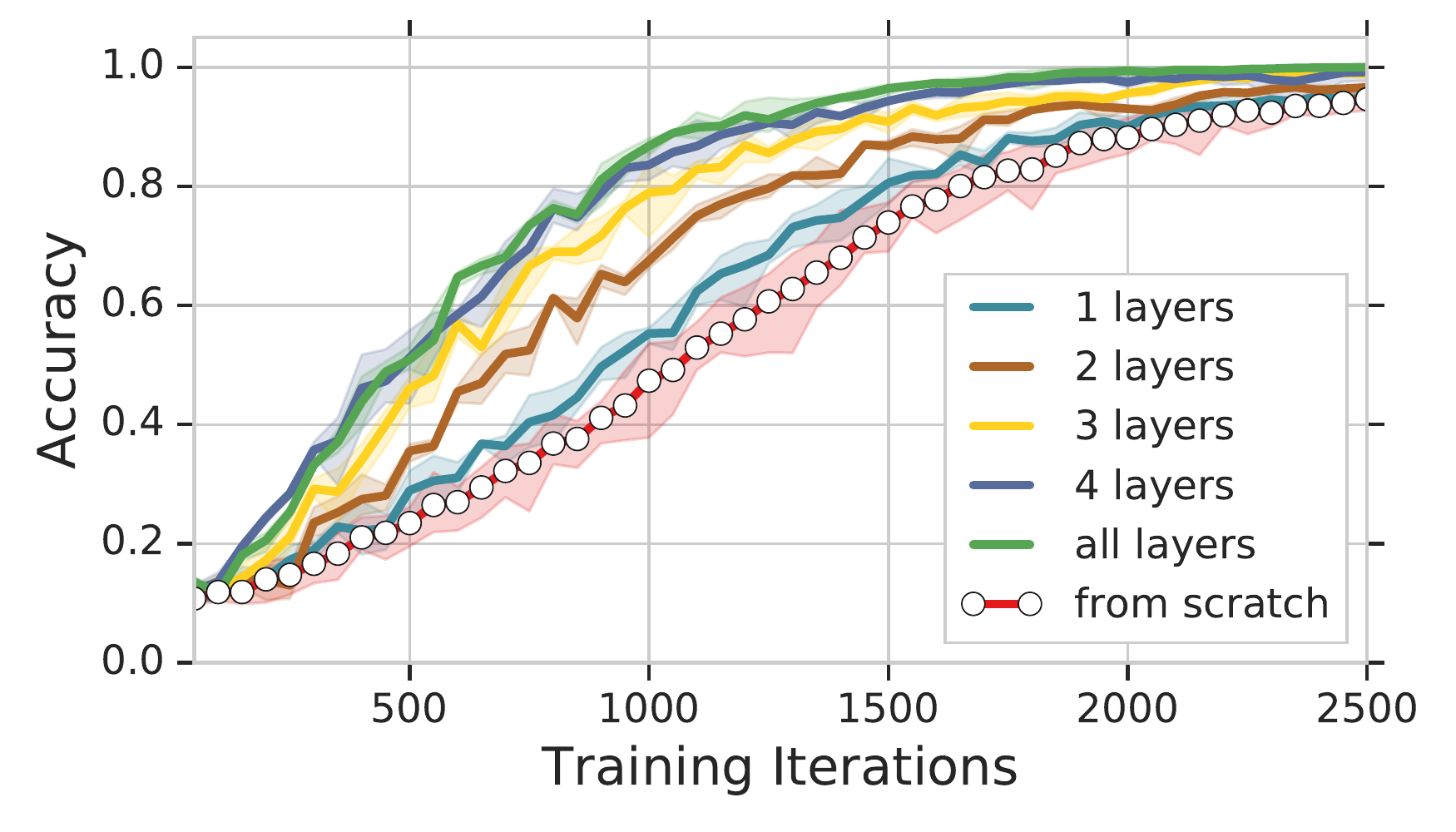}
  \caption{Transferring more layers improves downstream performance.
  Simple CNN architecture with 3 convolutional layers ({\sc left}),
  VGG16 ({\sc center}), and
  ResNet18 ({\sc right})
  pre-trained on CIFAR10 examples with random labels and subsequently
  fined-tuned on 25k fresh CIFAR10 examples with random labels.
  Lines with circular markers correspond to training from scratch.
  Error bars correspond to min/max over 3 runs.
  Plots for fine-tuning with real labels available in the appendix.
  }
  \label{fig:TransferKLayers}
\end{figure}

\section{Specializing neurons}
\label{sec:Specializing}
Despite the alignment effect taking place at the earlier layers of the neural network when trained with random labels, negative transfer is sometimes observed when fine-tuning on a downstream task as shown in Figure~\ref{fig:image1}. In this section, we show that this is likely due to a specialization at the later layer.

Figure~\ref{fig:VGGNeuronActivation} displays the distribution of neurons with respect to the number of held out images they are activated by for the settings of Figure~\ref{fig:image1} that exhibited positive (top row) and negative (bottom row) transfers. Comparing neural activations during initialization, end of pre-training, and end of fine-tuning, we note that neural activations are markedly diminished in the negative transfer case compared to the positive transfer case despite the fact that their neural activation distributions were identical during initialization. In Appendix F, we show that the significant drop in neural activation in the negative transfer case happens immediately after switching to the downstream task. As a result, the effective  capacity available downstream is diminished. By contrast, neural activations are not severely impacted in the positive transfer setting. In Appendix F, we provide detailed figures describing this phenomenon across all layers of VGG16, which reveal that such a specialization effect becomes more prominent in the later layers compared to the earlier layers. In particular, Appendix F shows that neural activations at the later layers can drop abruptly and permanently once the
 switch to the downstream task takes place due to specialization, which can prevent the network for recovering its fully capacity.

\begin{figure}[tb]
  \centering \scriptsize \sffamily
  \begin{minipage}{0.32\textwidth}
    \centering
    \small
    \hspace{.9cm}Initialization
  \end{minipage}
  \begin{minipage}{0.32\textwidth}
    \centering
    \small
    End of pre-training
  \end{minipage}
  \begin{minipage}{0.32\textwidth}
    \centering
    \small
    End of fine-tuning
  \end{minipage}\\
    \includegraphics[width=.32\textwidth]{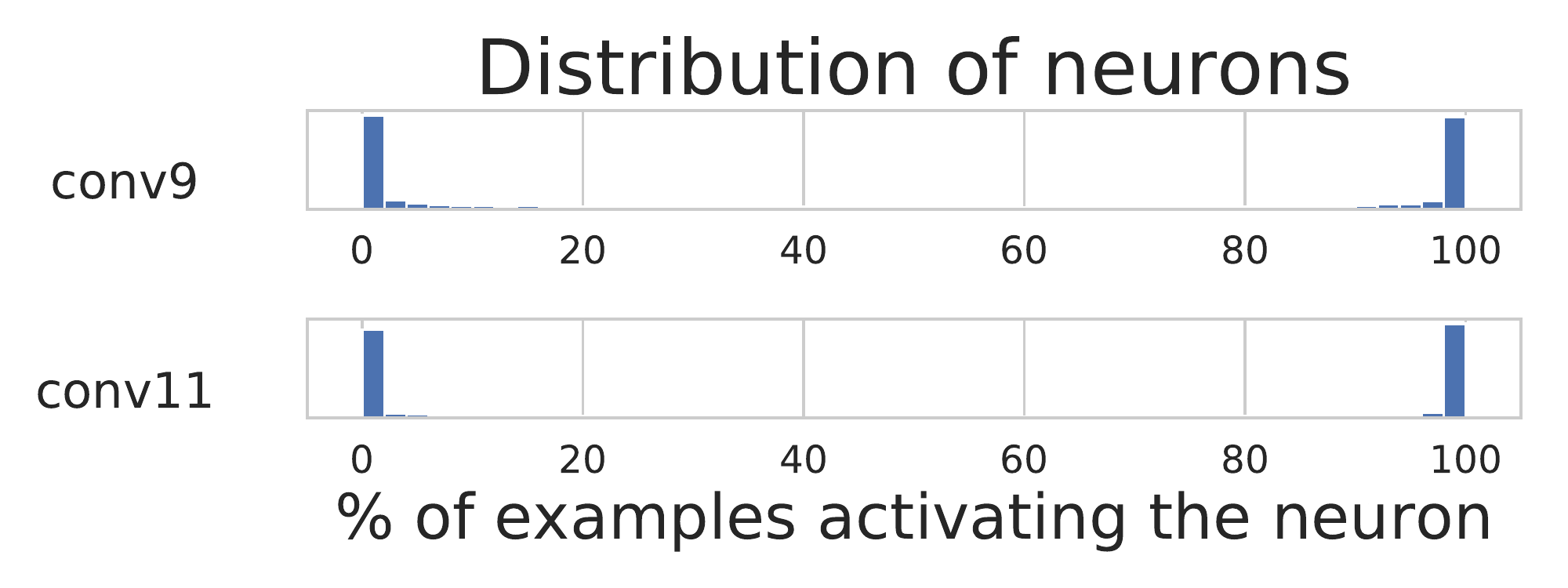}
  \includegraphics[width=.32\textwidth]{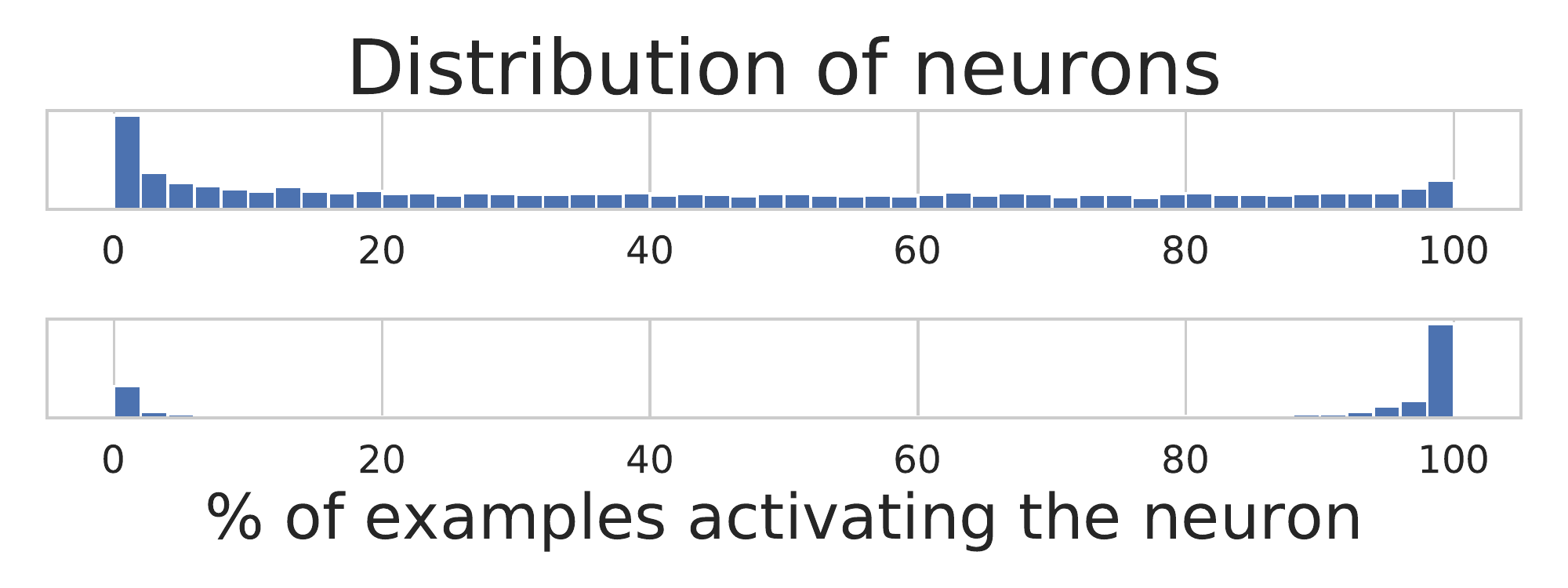}
  \includegraphics[width=.32\textwidth]{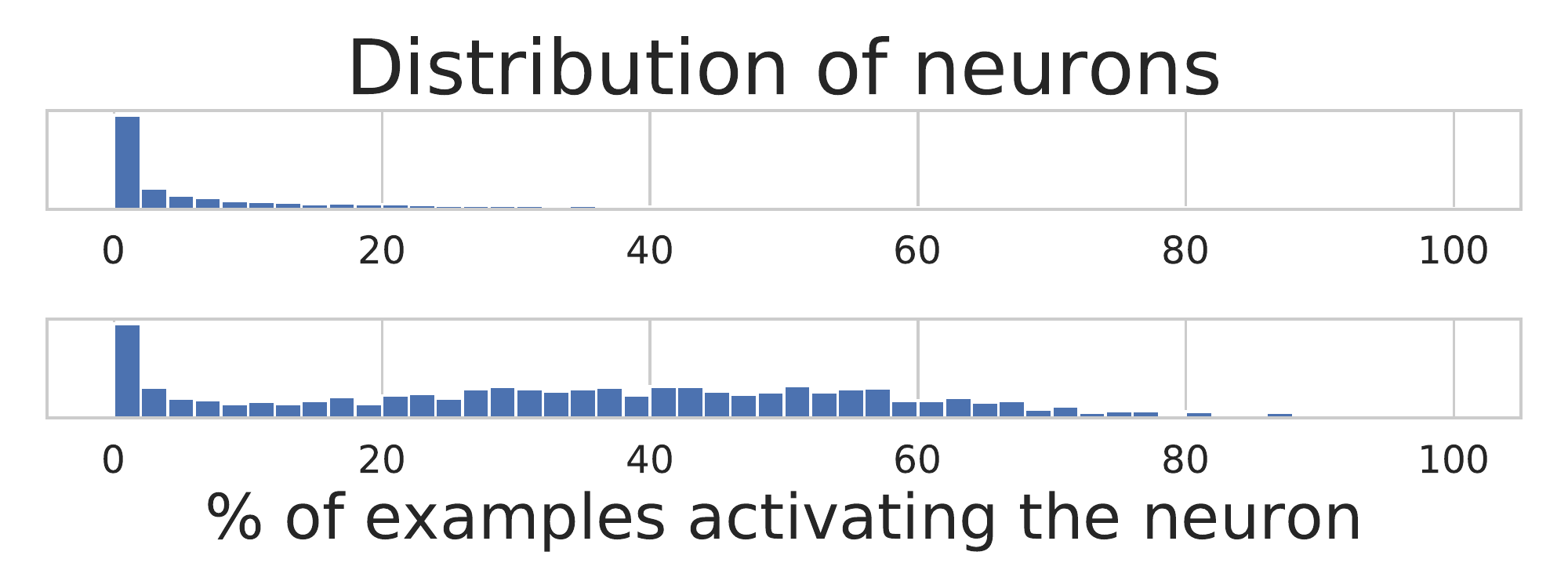}\\
  \includegraphics[width=.32\textwidth]{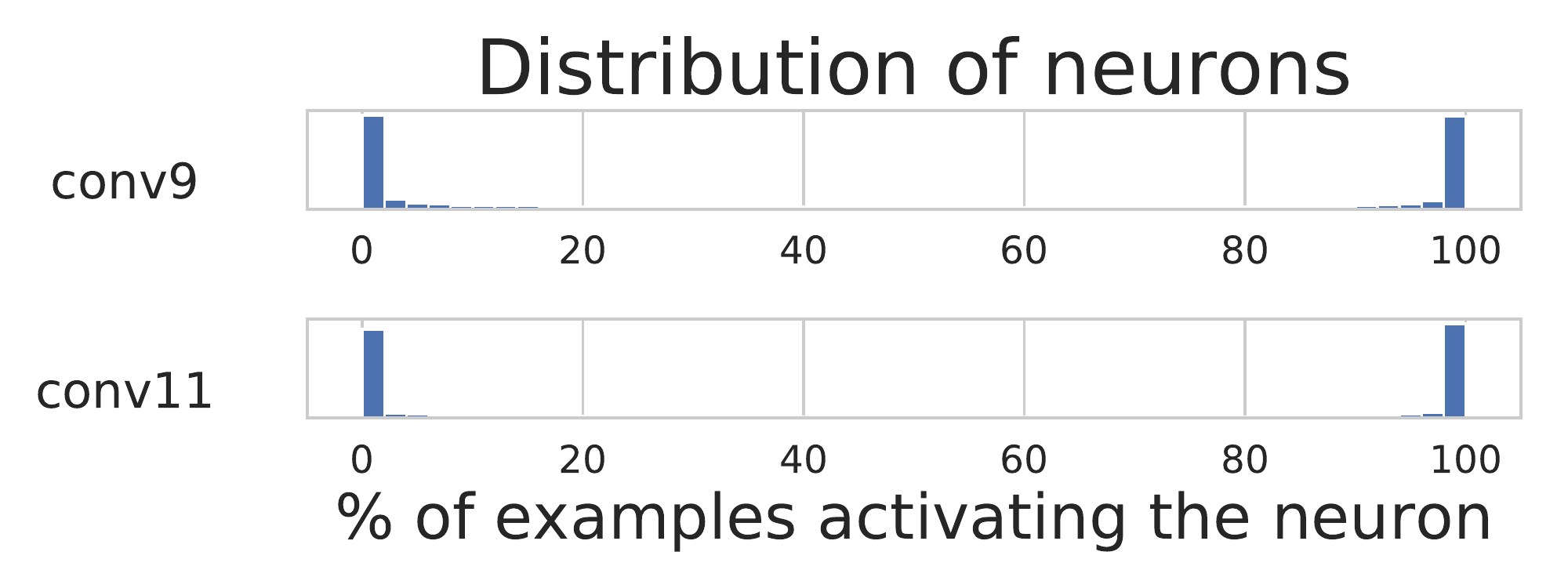}
  \includegraphics[width=.32\textwidth]{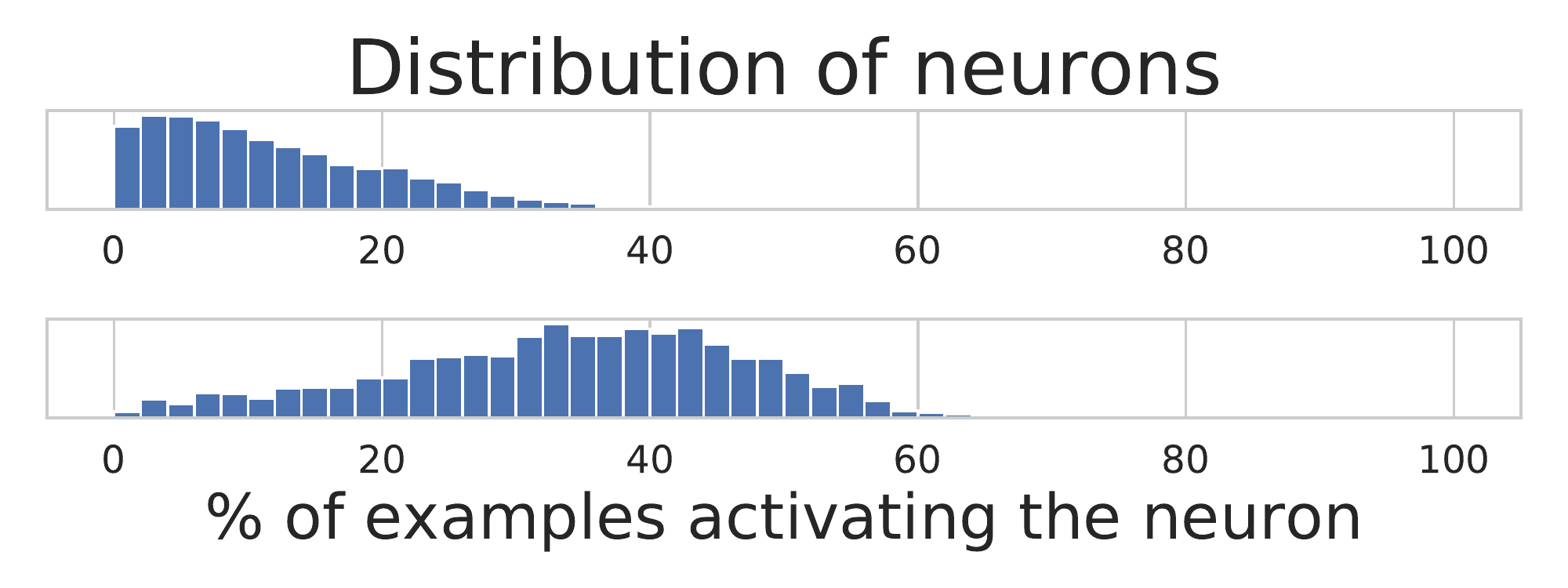}
  \includegraphics[width=.32\textwidth]{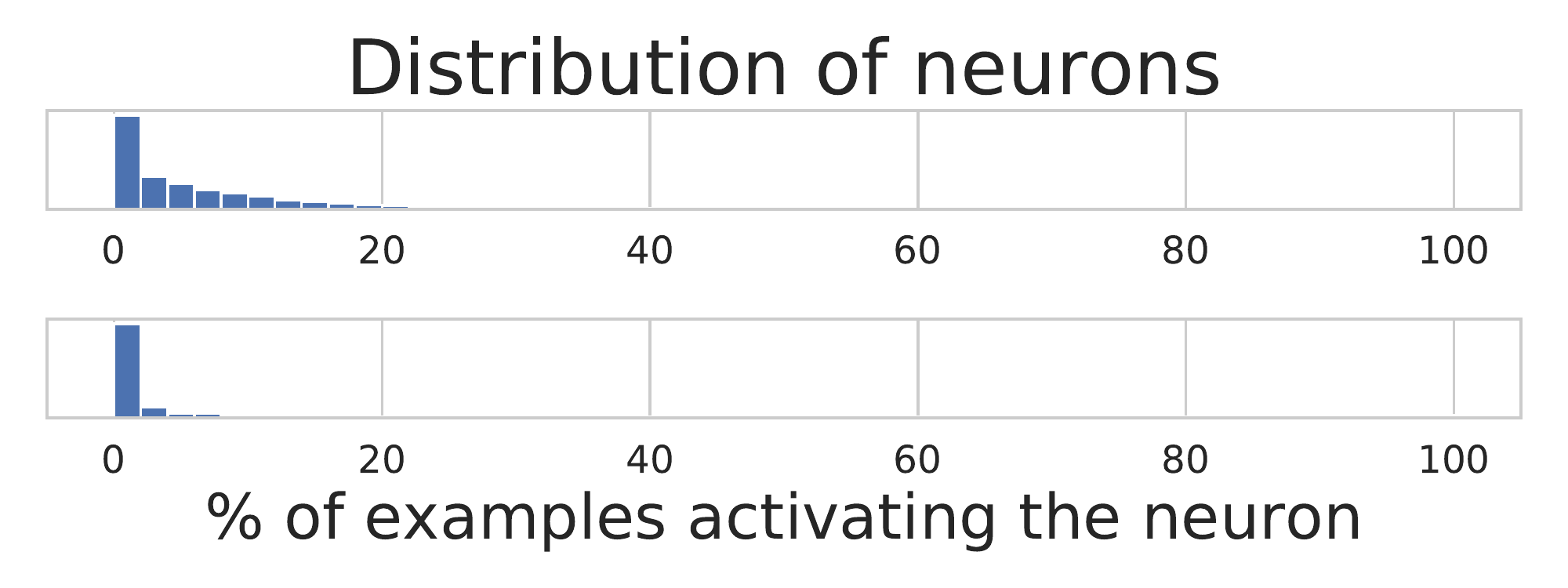}
  \caption{
    Activation plots for the two VGG16 models in Figure \ref{fig:image1}
    at initialization ({\sc left}),
    after pre-training 
    with random labels ({\sc center}),
    and after subsequently fine-tuning on fresh 
    examples with random labels ({\sc right}). Top row is for the positive transfer case; bottom row shows negative transfer. 
    Histograms depict distributions of neurons over the fraction of \emph{held out} 
    examples that activate them.
    The two histograms in each subplot correspond to two different 
    activation spaces.
    }
  \label{fig:VGGNeuronActivation}
\end{figure}

\begin{figure}[tb]
  \centering
  \includegraphics[width=.32\textwidth]{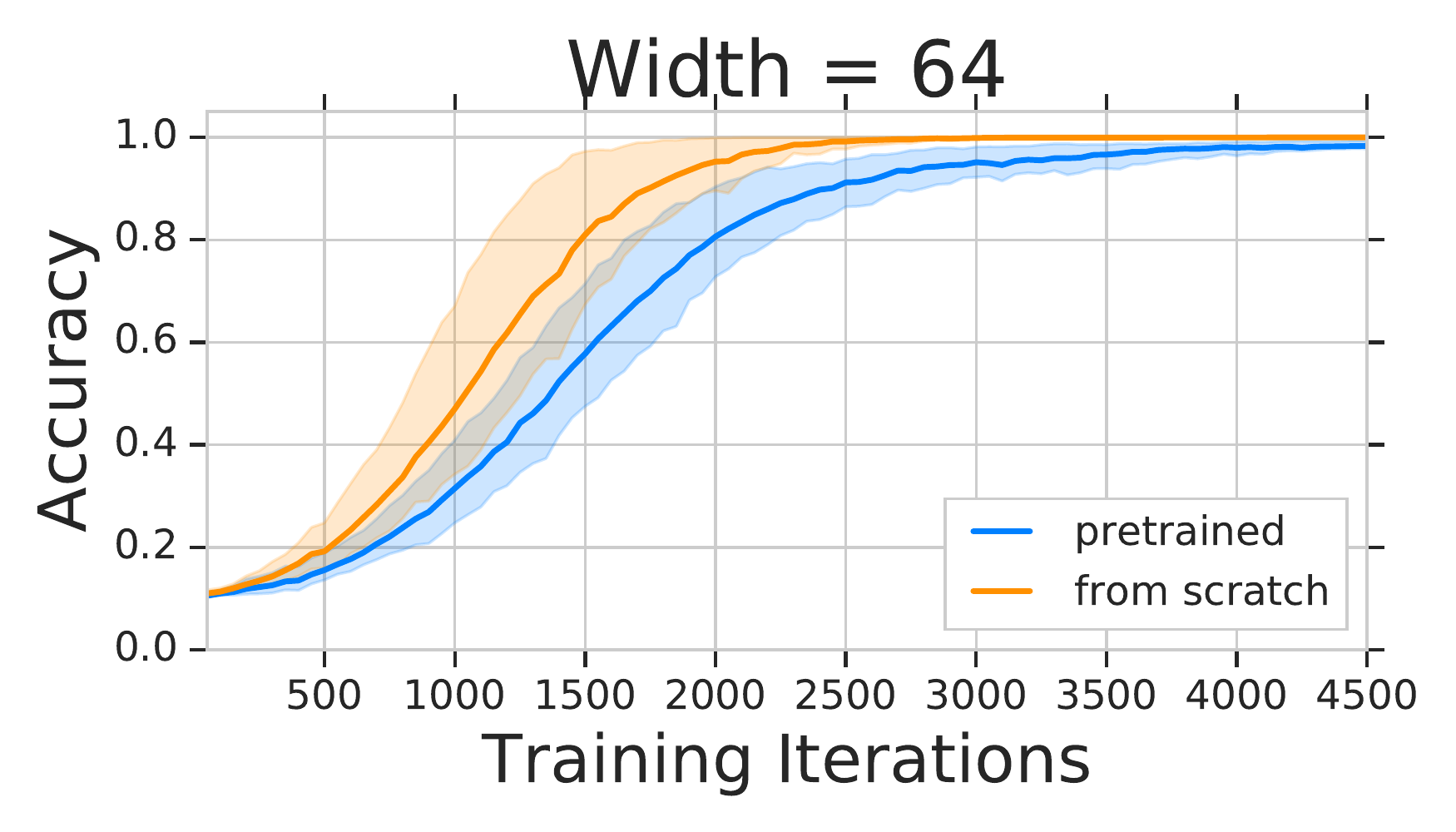}
  \includegraphics[width=.32\textwidth]{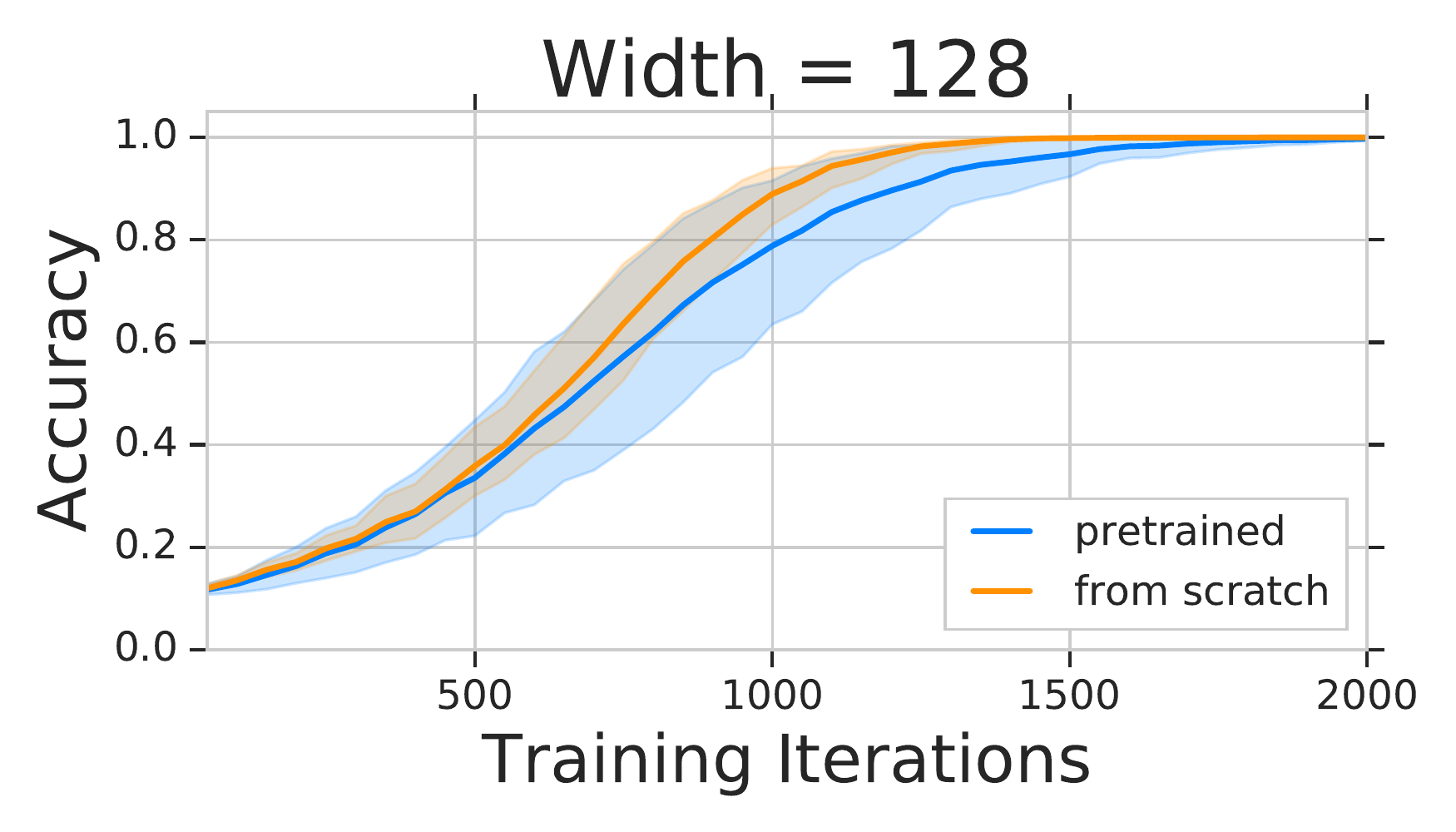}
  \includegraphics[width=.32\textwidth]{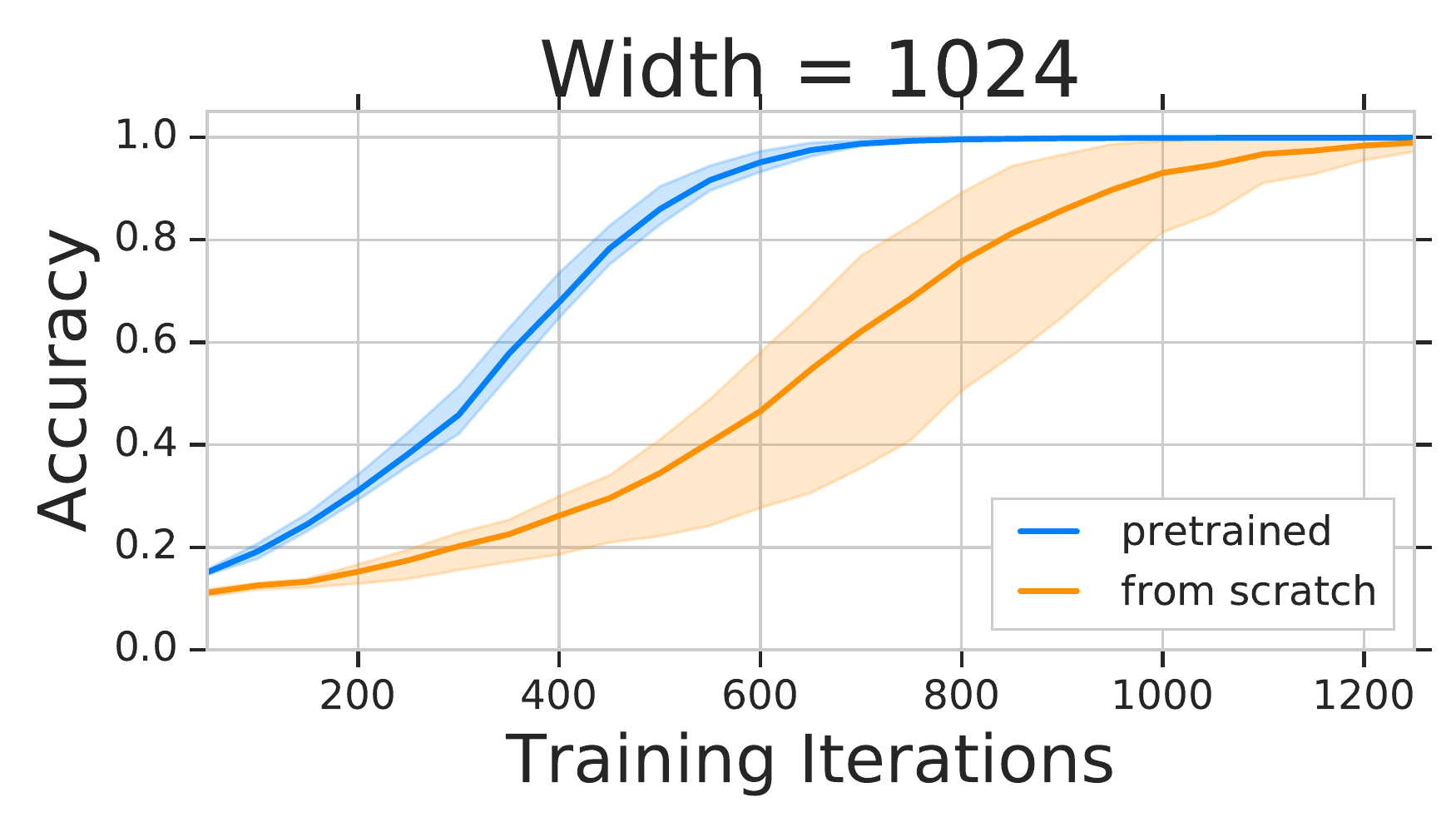}
  \caption{
    Increasing model width mitigates negative transfer.
    Simple CNN architectures with two convolutional layers and 64 ({\sc left}), 128 ({\sc center}), and 1024 ({\sc right}) units in the dense layer.
    }
  \label{fig:SimpleCNNIncreasingWIdth}
\end{figure}

One way to mitigate the effect of the inactive ReLU units is to increase the width so that the capacity remains sufficiently large for the downstream task. Figure~\ref{fig:SimpleCNNIncreasingWIdth} shows that increasing the width can indeed mitigate the negative transfer effect.  
While increased width seems to have general performance advantages~\cite{zagor2016}, it seems to be
also particularly useful in the case of transfer learning~\cite{alex2019big}.

\section{Concluding remarks}
The objective of this paper is to answer the  question of what neural networks learn when trained on random labels. 
We provide a partial answer by proving an alignment effect of principal components of network parameters and data and studying its implications, particularly for transfer learning. 
One important consequence is that second-order statistics of the earlier layers can be reduced to a one-dimensional function, which exhibits a surprising, regular structure. 
It remains an open question what the ``optimal'' shape of such function is, or whether it can be described analytically.  

The models used in this paper are taken from a large set of experiments that we conducted using popular network architectures and datasets, such as simple convolutional networks, VGG16, ResNet18-v2, CIFAR10,
and ImageNet, with wide range of hyperparameter settings (Appendix B).
These experiments show that pre-training on random labels
very often accelerates training on downstream tasks compared to training from scratch with the \emph{same hyperparameters}
and
rarely hurts the training speed.

By studying what is learned on random labels, we shed new insights into previous phenomena that have been reported in the literature. For instance, the alignment effect at the earlier layers explains the empirical observations of \cite{pondenkandath2018leveraging} that pre-training on random labels can accelerate training in downstream tasks and the observation of \cite{arpit2017closer} that the filters learned on random labels seem to exhibit some useful structure.  Also, our findings related to the inactive ReLU units at the later layers demonstrate how upper layers specialize early during training, which may explain why neural networks exhibit critical learning stages \cite{achille2017critical} and why increasing the width seems to be particularly useful in transfer learning \cite{alex2019big}. Both alignment and specialization are in agreement with the observation that  earlier layers generalize while later layers specialize, a conclusion that has been consistently observed in the literature when training on real labels \cite{ansuini2019intrinsic,arpit2017closer,cohen2018dnn,yosinski2014transferable}. We show that it holds for random labels as well.  


\section*{Acknowledgements}
The authors are grateful to
Alexander Kolesnikov, 
Alexandru Ţifrea,
Jessica Yung,
Larisa Markeeva,
Lionel Ngoupeyou Tondji, 
Lucas Beyer,
Philippe Gervais,
and other members of the Google Brain Team
for valuable discussions.

\section*{Broader Impact}
This work is partially theoretical  and contains experiments to study the theoretical results and related hypotheses. The paper aims at improving our understanding of how DNNs learn from data and therefore does not have a \textit{direct} impact on applications or society. Hence, speculating on its potential broader impact is difficult at this stage. Nevertheless, we hope that a better understanding of deep neural networks will lead to improvements in the future along the direction of building interpretable and explainable AI,
which are critical ingredients for the creation of socially-responsible AI systems.

\bibliographystyle{abbrvnat}  
\bibliography{main}

\newpage
\appendix

\setcounter{topnumber}{10} 
\setcounter{bottomnumber}{10} 
\setcounter{totalnumber}{20} 
\renewcommand{\topfraction}{1.0} 
\renewcommand{\bottomfraction}{1.0} 
\renewcommand{\textfraction}{0.0} 
\renewcommand{\floatpagefraction}{1.0} 

\section{Experimental details}
\label{sec:Experiments}

\subsection{Upstream and downstream datasets}
We use two datasets: CIFAR10 \cite{Krizhevsky09learningmultiple} and
ImageNet ILSVRC-2012 \cite{deng2009imagenet}.

For each run of a transfer experiment we 
\begin{enumerate}
    \item
    randomly sample \texttt{num\_examples\_upstream} examples uniformly from the training split of the dataset to form the (upstream) training set for pre-training;
    \item
    randomly sample \texttt{num\_examples\_downstream}
    examples from the remainder of the training split of the dataset. 
    They form the (downstream) training set for fine-tuning.
    This guarantees that the upstream and downstream examples do not intersect;
    \item
    upstream training examples are labelled with \texttt{num\_classes\_upstream} classes randomly and uniformly;
    \item 
    when fine-tuning on random labels, we randomly and uniformly label the downstream examples with \texttt{num\_classes\_downstream} classes. 
\end{enumerate}

\subsection{Neural network architectures}
Throughout the main text we use three different architectures: a simple convolutional architecture ``Simple CNN'', VGG16, and ResNet18.
The Simple CNN can be further configured by specifying the number of convolutional layers and units in the dense layer.
We use the same architectures for
both CIFAR10 and ImageNet datasets, by only adjusting the number of outputs (logits).

{\bf Simple CNN}
is a convolutional architecture consisting of:
\begin{enumerate}
    \item \texttt{num\_conv\_layers} convolutional layers with $3\times 3$ filters, each followed by the ReLU activation. Each convolutional layer contains \texttt{num\_filters} filters (with biases) that are applied with 
    stride 1.
    \item
    The outputs of the final convolutional layer are flattened and passed to a
    dense layer (with biases) of \texttt{num\_units} units, followed by the ReLU activation.
    \item
    The classifier head, i.e.\ a dense layer with \texttt{num\_output} units (logits).
\end{enumerate}

The {\bf VGG16} architecture that we use is ``Configuration D'' from Table 1 of \cite{simonyan2014very} with two dense layers (``FC-4096'') removed:
\begin{Verbatim}[fontsize=\small]
       conv0:   64 filters 
       conv1:   64 filters 
       maxpool 
       conv2:  128 filters 
       conv3:  128 filters 
       maxpool
       conv4:  256 filters 
       conv5:  256 filters 
       conv6:  256 filters 
       maxpool
       conv7:  512 filters 
       conv8:  512 filters 
       conv9:  512 filters 
       maxpool
       conv10: 512 filters 
       conv11: 512 filters 
       conv12: 512 filters 
       maxpool
       dense layer with num_outputs units (classifier head)
\end{Verbatim}
All convolutional filters (with biases) are of size $3\times3$ and applied with \texttt{SAME} padding and stride~1.
ReLU activation is applied after every convolutional layer.
Max-pooling is performed over a $2\times2$ pixel window,
with {\sc SAME} padding and stride 2.

{\bf ResNet18} We use vanilla ResNet-v2 architecture \cite{He2016}
with batch normalization \cite{ioffe2015} replaced by the group normalization \cite{wu2018}.

\subsection{Training}
All biases in the models are initialized with zeros, while all the other parameters (convolutional filters and weight matrices of the dense layers) are initialized using He normal algorithm~\cite{he2015} with \texttt{init\_scale} scaling.

Model outputs (logits) are passed through the softmax function and we minimize the cross-entropy (i.e. the negative log-likelihood) during training.
We use SGD with momentum 0.9 and batch size 256 to train our models.
We start training with the specified \texttt{learning\_rate} and divide it by 3 two times during the training: after $1/3\times$\texttt{total\_steps} and $2/3\times$\texttt{total\_steps} steps.

When training the models we report the accuracy on \emph{the entire training set}, not on the mini-batch.

We do not use data augmentation in our experiments.
We do not use weight decay, dropout, or any other regularization.
For CIFAR10 we scale the inputs to the $[-1, 1]$ range.
For ImageNet we resize the examples, take a $224\times224$ central crop, and scale inputs to the $[-1, 1]$ range.

\subsection{Transferring the model from the upstream to the downstream task}
After the upstream pre-training finishes, we replace the classifier head with a freshly initialized one.

Unless otherwise stated, for Simple CNN and VGG16 architectures we also re-scale the model parameters after the pre-training.
We store the per-layer parameter $\ell_2$ norms at the initialization and re-scale each layer of the trained model to match the stored $\ell_2$ norms.
Re-scaling does not affect the classifier accuracy and predictions, since it reduces to multiplying all parameters of a given layer by a strictly positive constant.
However, re-scaling changes the logits and, therefore, the cross-entropy loss.
It is not immediately obvious how to re-scale the residual architecture without changing its predictions so we decided not to re-scale ResNet18.

When fine-tuning downstream we use the same initial learning rate and schedule as in upstream pre-training.

\subsection{Figure \ref{fig:image1}: Positive and negative transfer with VGG16 on CIFAR10}

Experiments 1 and 2 use:
\begin{verbatim}
init_scale = 0.526 
learning_rate = 0.01
num_classes_upstream = 5
num_examples_upstream = 20000
epochs_upstream = 120
\end{verbatim}

Experiments 3 and 4 use:
\begin{verbatim}
init_scale = 0.612
learning_rate = 0.009
num_classes_upstream = 50
num_examples_upstream = 20000
epochs_upstream = 80
\end{verbatim}

Error bars correspond to $\pm 1$\,std.\:over 12 independent runs.

\subsection{Figure \ref{fig:Experiment_alignment}: Misalignment plots}
To measure misalignment, we use the misalignment score in Definition \ref{def:alignment} using the closed-from expression provided in Section \ref{subsec:Gaussian}. First, the eigenvectors of the $3\times 3$ patches of images are computed, which reside in a space of dimension 27 due to the three input channels. Let $v_i$ be the eigenvectors of data, which have distinct eigenvalues almost surely, the alignment score is then:
\begin{equation*}
    \sum_{i=1}^d \sqrt{(v_i^T\Sigma_wv_i)\cdot (v_i^T\Sigma_w^{-1}v_i)} - 1,
\end{equation*}
where the summation is taken over all 27 eigenvectors. To gain an intuition behind this formula, note that if $v_i$ is itself an eigenvector of $\Sigma_w$, then $(v_i^T\Sigma_wv_i)\cdot (v_i^T\Sigma_w^{-1}v_i)=1$. The covariance of weights $\Sigma_w$ is estimated in a single run by computing the covariance of the filters in the first layer. 

To ensure that alignment depends indeed on the eigenvectors of data, we also compute alignment in which the set of data eigenvectors $\{v_i\}$ is replaced by some random orthonormal basis of the plane. 
The network architecture is a 2-layer CNN: convolutional followed by a fully-connected layer. The experiment setup used to produce Figure \ref{fig:Experiment_alignment} is:
\begin{verbatim}
num_conv_layers = 1
num_filters = 256
num_units = 64
learning_rate = 0.01
num_classes = 10
num_examples = 50000
epochs_upstream = 40
\end{verbatim}

\subsection{Figure \ref{fig:filters}: ResNet convolutional filter alignment}
\label{sec:fig3settings}
Figure~\ref{fig:filters} was produced from 70 runs of a wide ResNet~\cite{zagor2016} on CIFAR10 with random labels. We used 4 blocks per group and a width factor of 4. As mentioned in the main text, we replaced the initial $3\times 3$ convolution with a $5\times 5$ convolution for better visualization. Training was run with batch normalization \cite{ioffe2015} and the following parameters:
\begin{verbatim}
init_scale = 0.01
learning_rate = 0.005
num_examples = 50000
epochs = 1800
\end{verbatim}

\subsection{Figure \ref{fig:Experiment2_f_sigma}: Plots of the function \texorpdfstring{$f(\sigma)$}{f(sigma)}}
These figures were generated using the procedure described in Section \ref{subsect::mapping_eigenvalues}. For every eigenvalue $\sigma_i^2$ of the data with eigenvector $v_i$, the corresponding $\tau_i^2$ is computed using Eq. (\ref{eq:general_tau}). The pairs $(\sigma_i^2,\tau_i^2)$ define a mapping from $\mathbb{R}^+$ to $\mathbb{R}^+$, which is plotted in Figure \ref{fig:Experiment2_f_sigma}. 

The network architecture is a 2-layer CNN: convolutional followed by a fully-connected layer. The experiment setup used to produce Figure \ref{fig:Experiment_alignment} is:
\begin{verbatim}
num_conv_layers = 1
num_filters = 256
num_units = 64
learning_rate = 0.01
num_classes = 10
num_examples = 50000
epochs_upstream = 40
\end{verbatim}

\subsection{Figure \ref{fig:PCA_init}: Explaining the positive transfer}
Network: {SimpleCNN}
\begin{verbatim}num_conv_layers = 1
num_filters = 64
num_units = 256
\end{verbatim}

\pagebreak
Training:
\begin{verbatim}
num_examples_upstream = 10000
num_classes_upstream = 10
num_examples_downstream = 10000
num_classes_downstream = 10
init_scaling = 1.0
learning_rate = 0.0005
total_steps = 10000
\end{verbatim}

For all curves with the exception of the ``no rescaling'' curve, each layer is scaled down by a factor 
between pre-training and training to match the $\ell_2$ norm of the weights at initialization.

For all runs, the head layer is re-initialized after pre-training.
For ``pretrained conv'' and ``pretrained conv, no bias'' also the fully connected layer is re-initialized.
For ``pretrained conv, no bias'' also the bias of the convolutional layer is reset to zero. This is the
fairest comparison to ``covariance'': For ``covariance'' the filters in the convolutional layer
are random samples from a Gaussian distribution with mean 0 and the covariance obtained
from training on random labels. The bias is set to zero and the dense and head layers are
initialized randomly.

Each line is the average of 12 runs that differ in the random initializations.
In the left and center image, the ``covariance'', ``pretrained'', and ``from scratch'' 
curves are surrounded by a colored area indicating $\pm 1$ standard deviations.
The ``no rescaling'' curve is without this area, it would touch the ``pretrained'' curve below it.
In the right image no error bounds are plotted since the curves are too close together.

\subsection{Figure \ref{fig:TransferKLayers}: Transferring more layers improves the downstream performance}

\begin{figure}[tbp]
  \centering \sffamily \footnotesize
  \begin{minipage}{0.32\textwidth}
    \centering 
    Simple CNN    
  \end{minipage}
  \begin{minipage}{0.32\textwidth}
    \centering
    VGG16
  \end{minipage}
  \begin{minipage}{0.32\textwidth}
    \centering
    ResNet18
  \end{minipage}\\
  \includegraphics[width=0.32\textwidth]{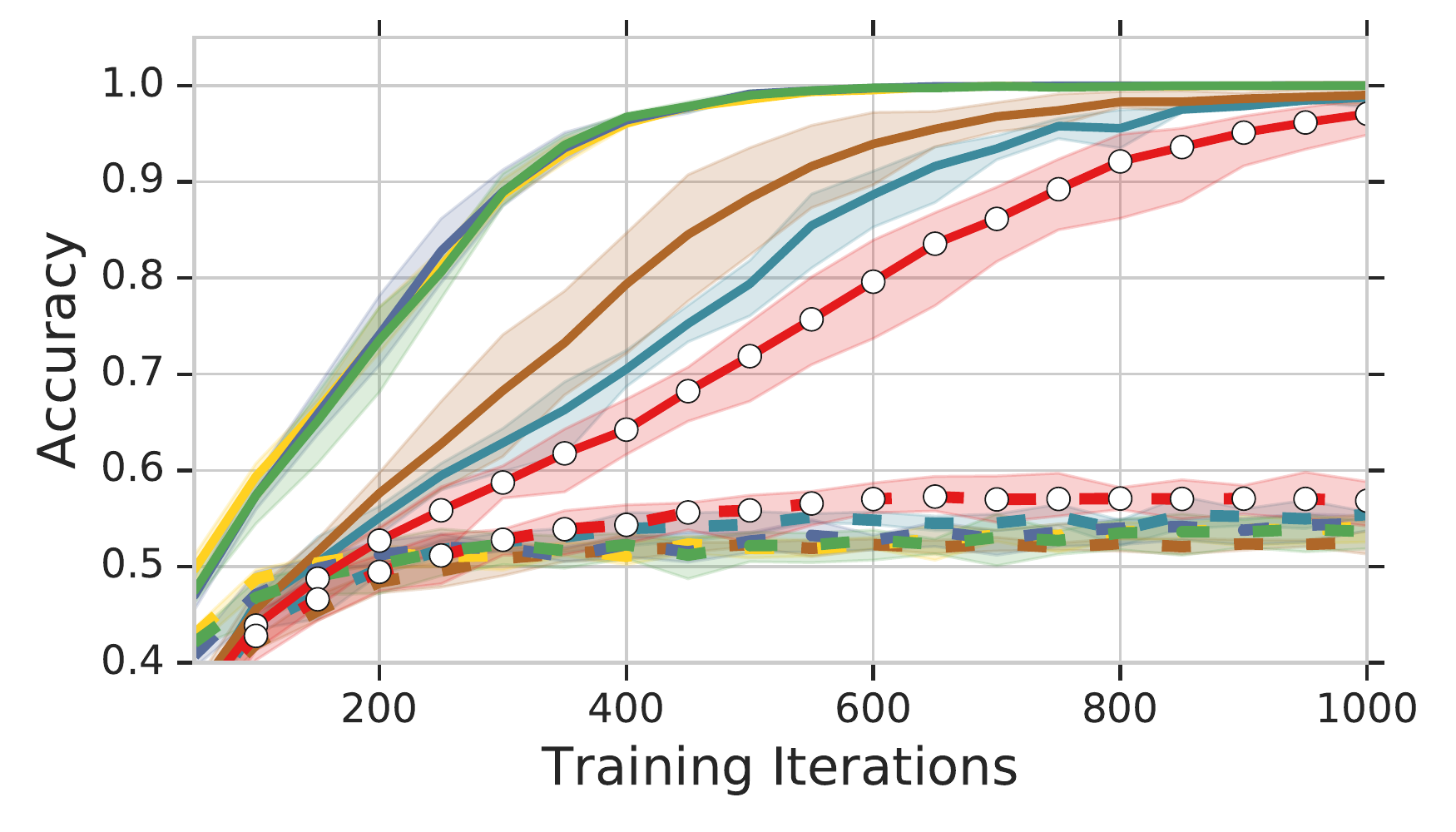}\hspace{.32cm}
  \includegraphics[width=0.32\textwidth]{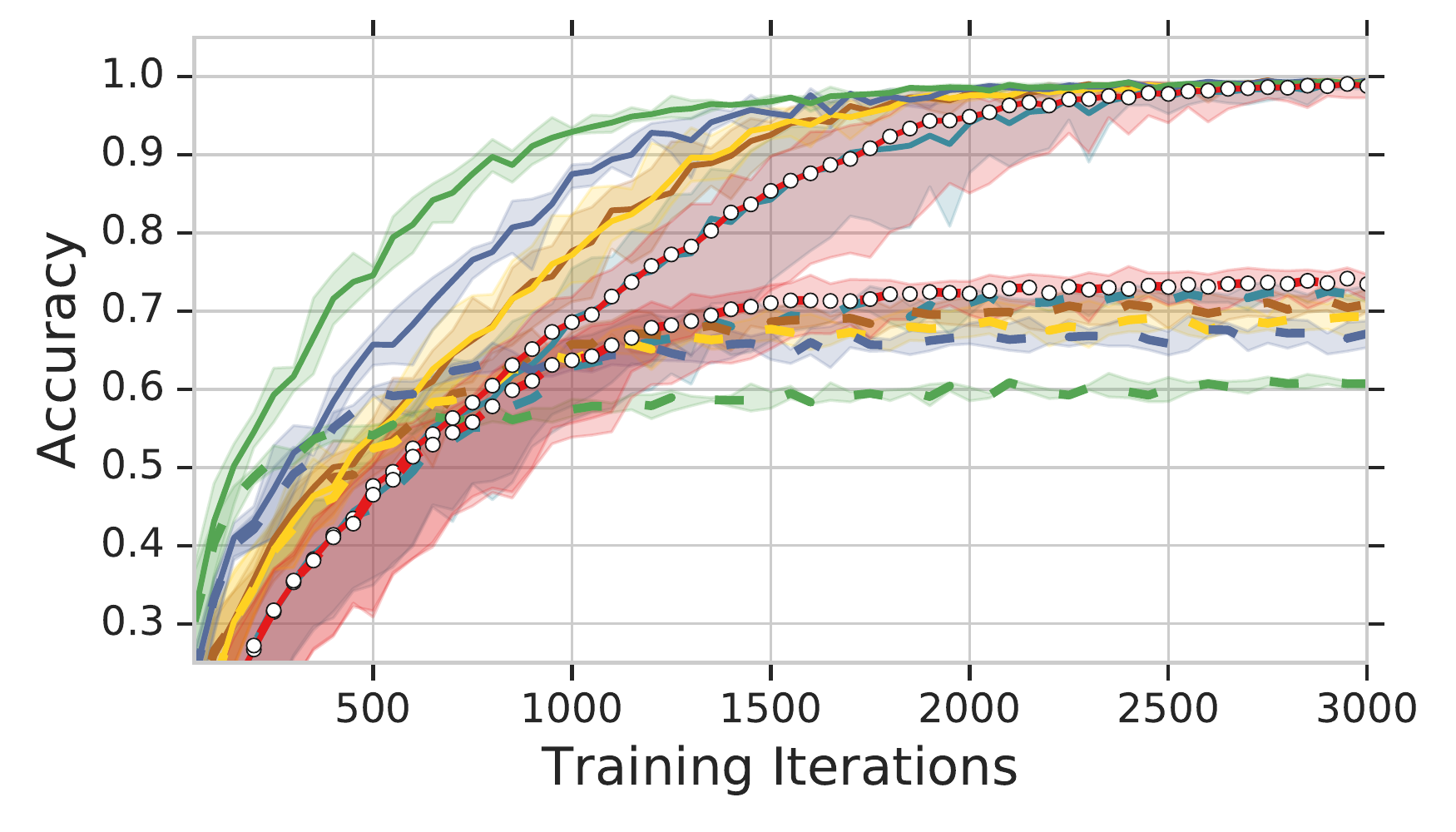}
  \includegraphics[width=0.32\textwidth]{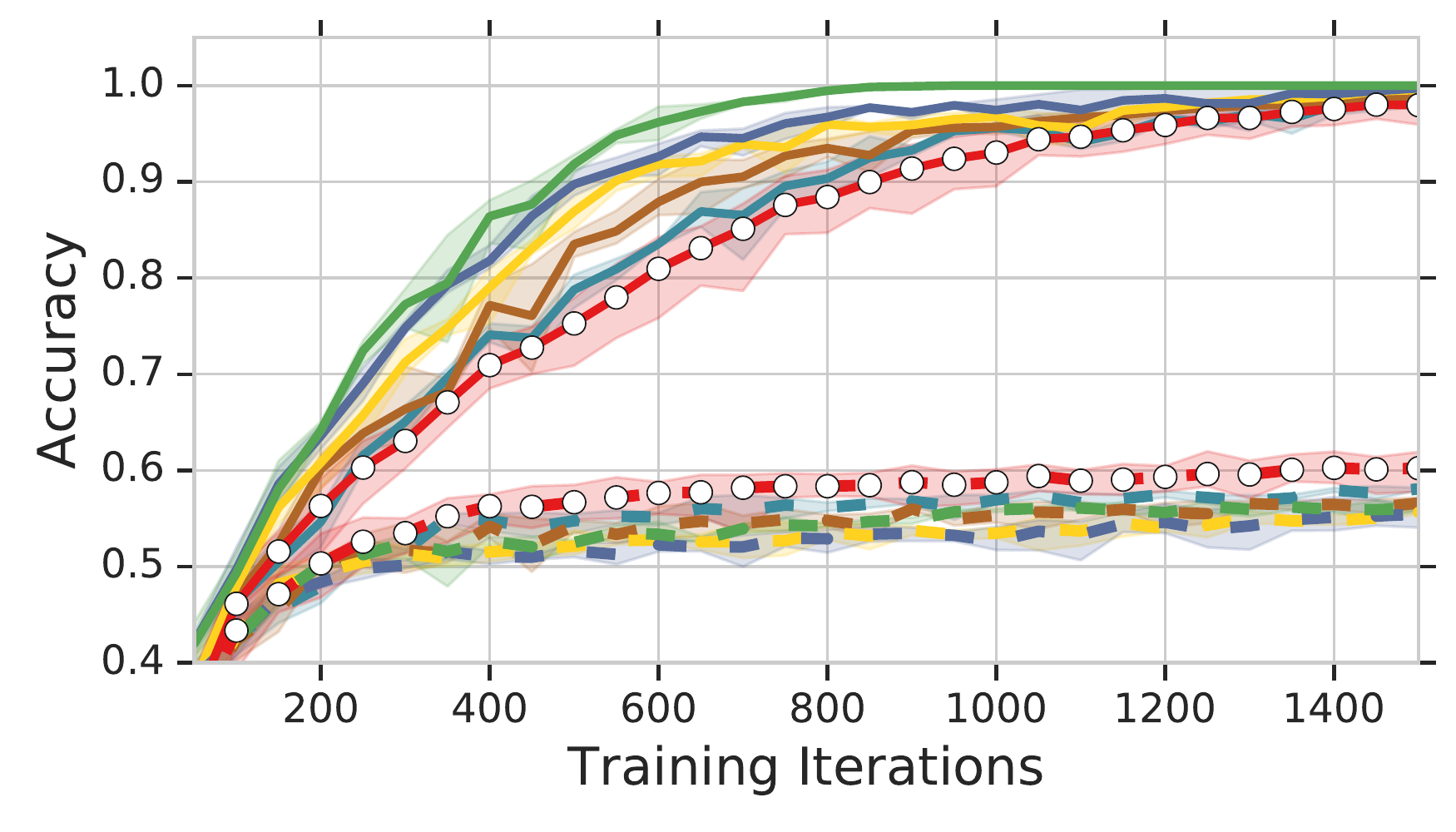}\\
  \includegraphics[width=0.32\textwidth]{images/cifar10_simplecnn_random_k_layers.pdf}
  \includegraphics[width=0.32\textwidth]{images/cifar10_vgg16_random_k_layers.pdf}
  \includegraphics[width=0.32\textwidth]{images/cifar10_resnet18_random_k_layers.pdf}
  \caption{Transferring more layers improves the downstream performance.
  Simple CNN architecture with 3 conv.\:layers ({\sc left}),
  VGG16 ({\sc center}), and
  ResNet18 ({\sc right})
  pre-trained on CIFAR10 examples with random labels and subsequently
  fined-tuned on 25k fresh CIFAR10 examples with real labels ({\sc top}) and 10 random labels ({\sc bottom}).
  Lines with circular markers correspond to training from scratch.
  Error bars correspond to min/max over 3 independent runs.
  }
  \label{fig:TransferKLayersSupplementary}
\end{figure}

Figure \ref{fig:TransferKLayersSupplementary} is the extended version of Figure \ref{fig:TransferKLayers} that includes the models fine-tuned on real labels (apart from the models fine-tuned on the random labels).

The left column (Simple CNN architecture) uses:
\begin{verbatim}
init_scale = 0.518
learning_rate = 0.01
num_classes_upstream = 25
num_examples_upstream = 15000
epochs_upstream = 40
num_conv_layers = 3
num_filters = 16
num_units = 1024
\end{verbatim}

The center column (VGG16) corresponds to the same setup as in Experiments 1 and 2 in Figure \ref{fig:image1}. It~uses:
\begin{verbatim}
init_scale = 0.526
learning_rate = 0.01
num_classes_upstream = 5
num_examples_upstream = 20000
epochs_upstream = 120
\end{verbatim}

The right column (ResNet18) uses:
\begin{verbatim}
init_scale = 0.671
learning_rate = 0.01
num_classes_upstream = 50
num_examples_upstream = 25000
epochs_upstream = 80
\end{verbatim}

\subsection{Figure \ref{fig:VGGNeuronActivation}: Neuron activation plots}\label{subsect::nuron_act_plots_vgg}

\begin{figure}[tbp]
  \centering
  \includegraphics[width=.46\textwidth]{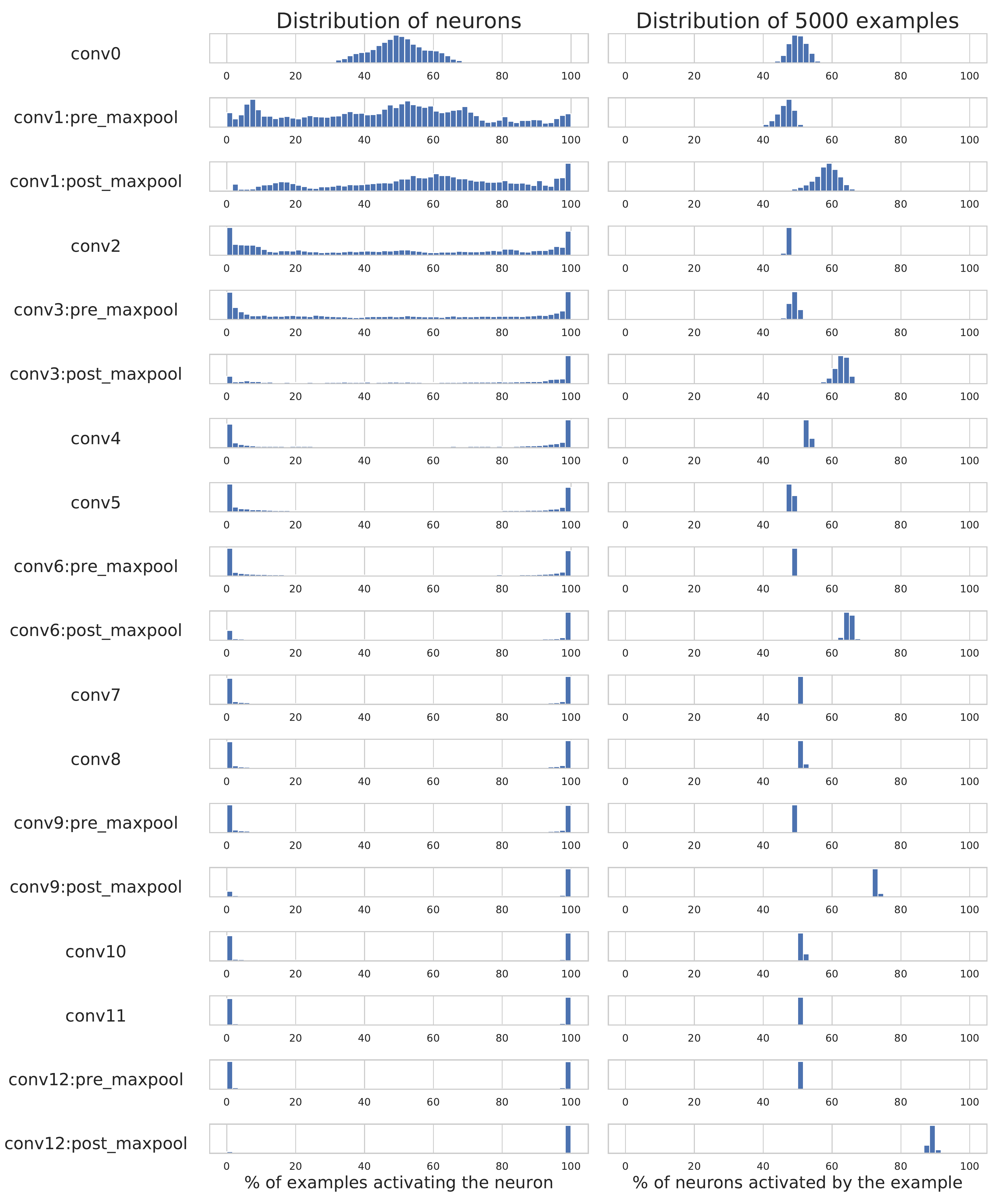}\hfill
  \includegraphics[width=.46\textwidth]{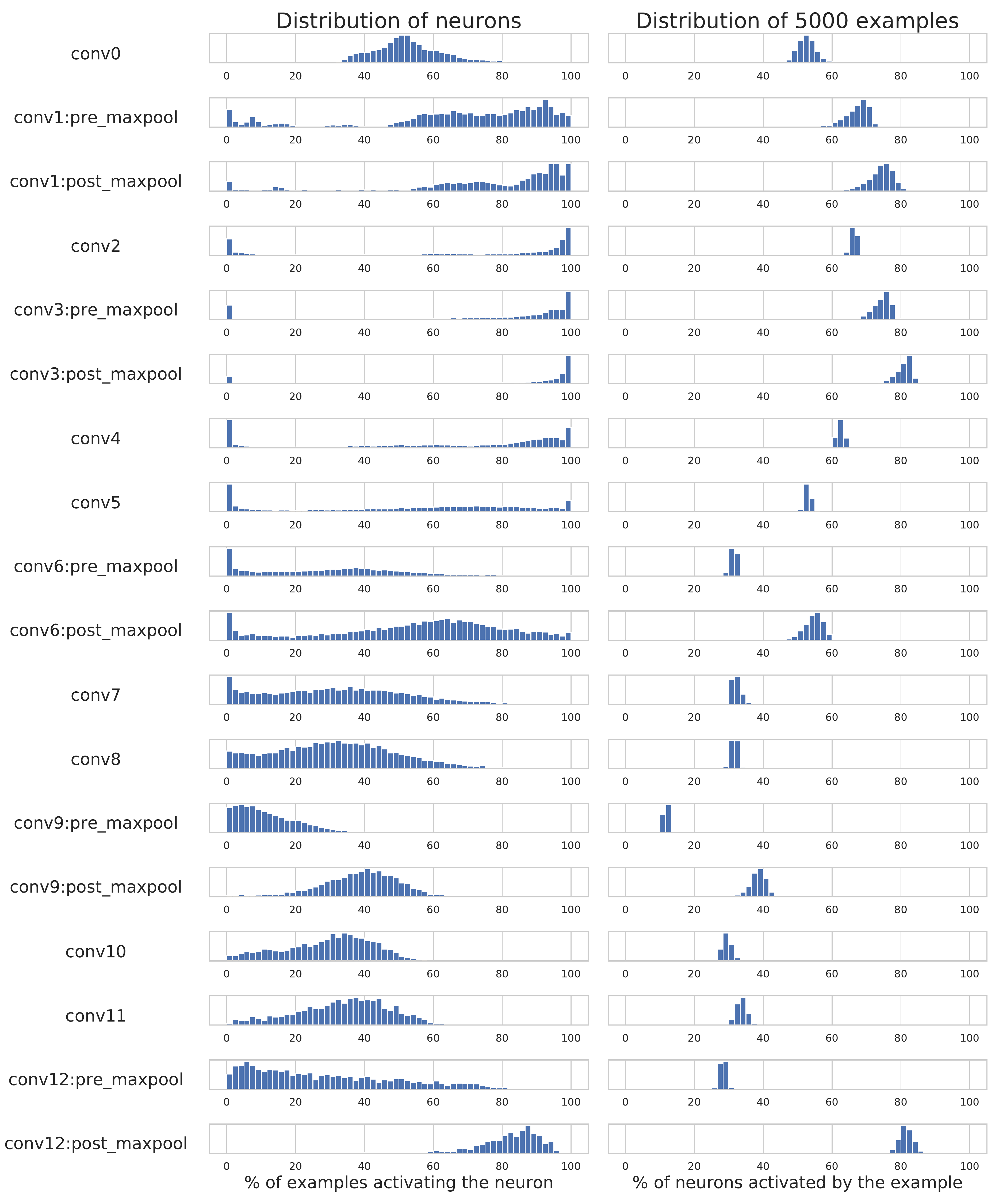}\\[1ex]
  \includegraphics[width=.46\textwidth]{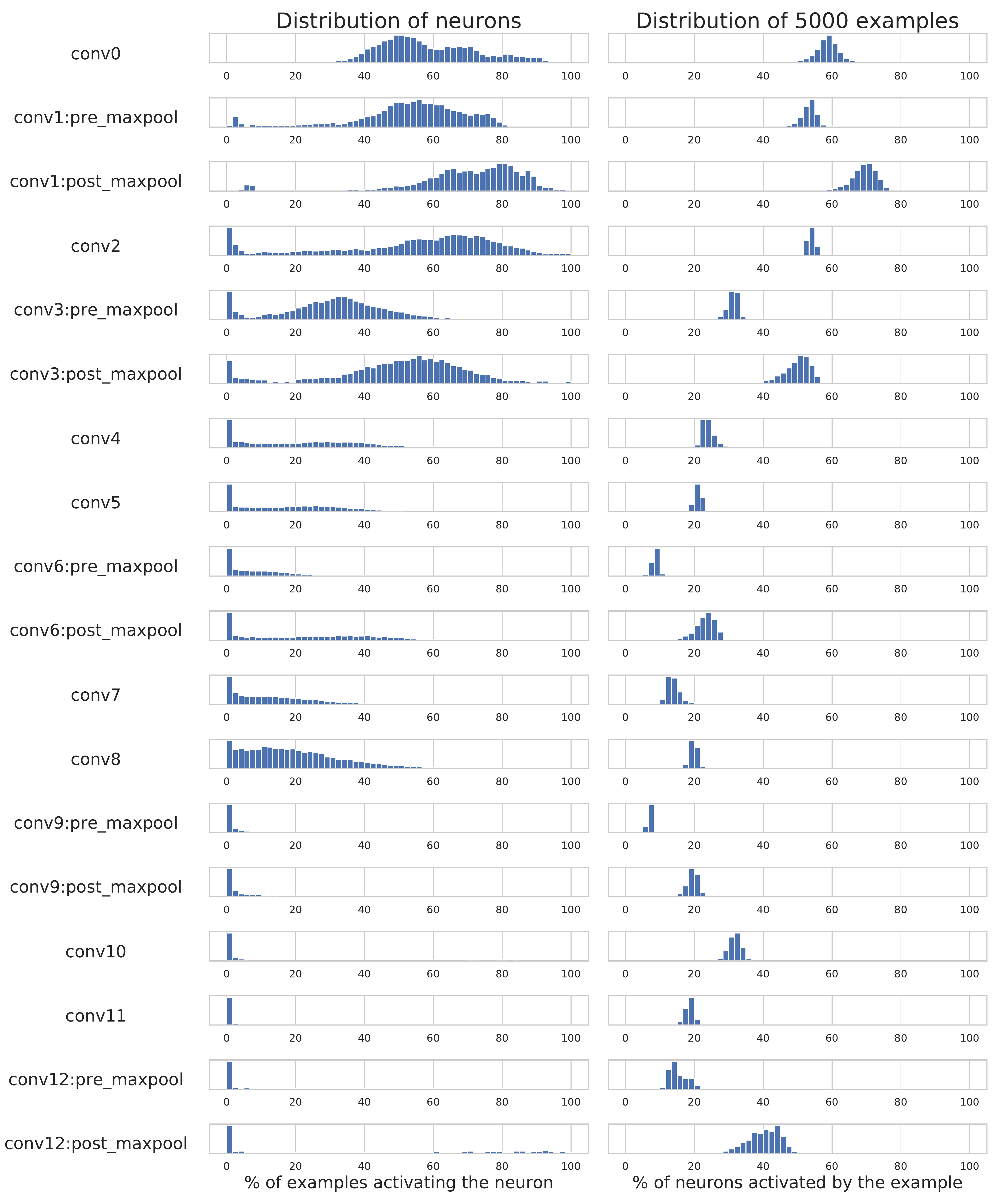}\hfill
  \includegraphics[width=.46\textwidth]{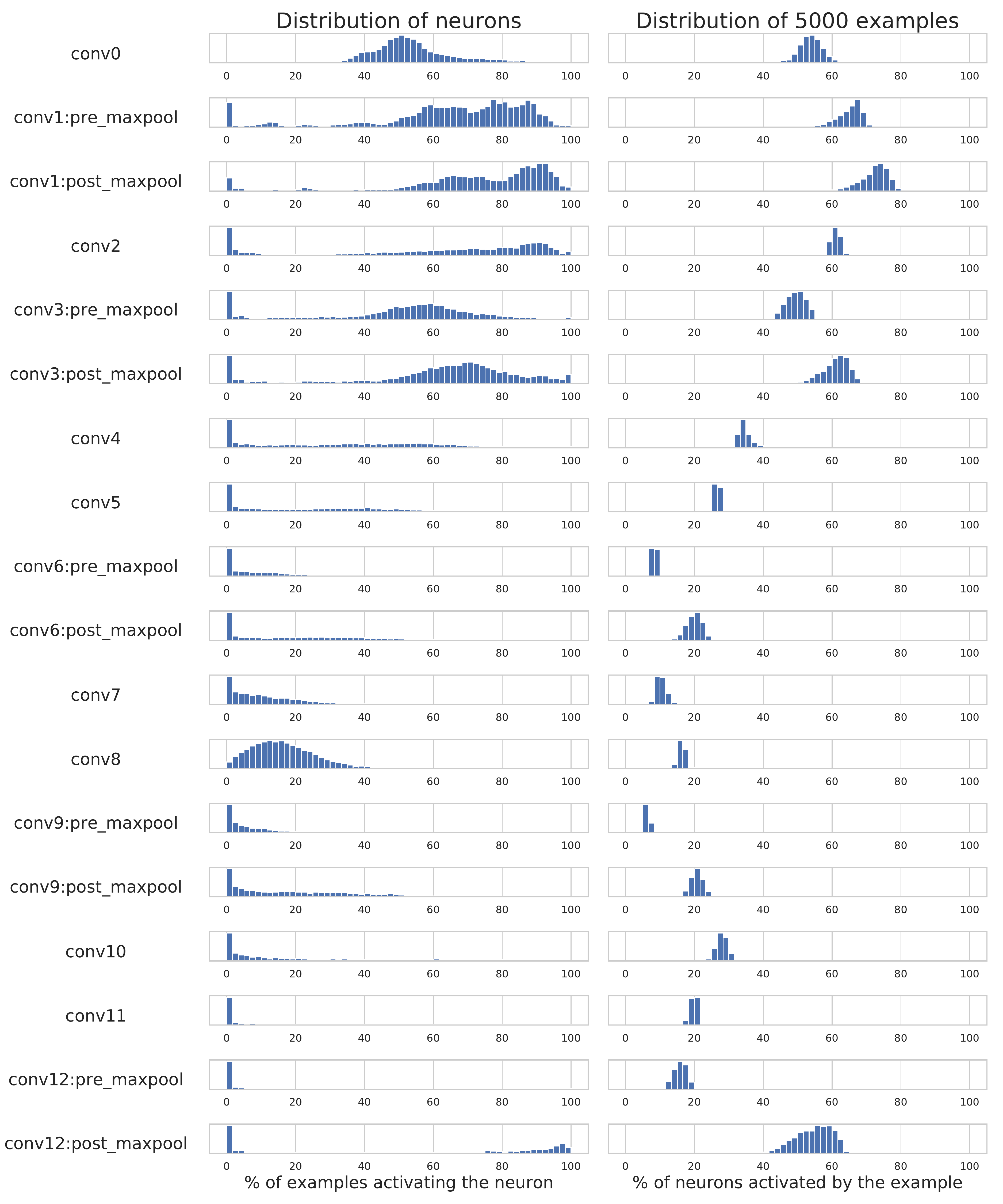}
  \caption{
    Neuron activations in case of the negative transfer.
    The VGG16 model
    pre-trained for 1 ({\sc top-left}) and 6240 ({\sc top-right}) training iterations on 20k examples from CIFAR10 with 50 random classes and
    subsequently fine-tuned for 200 epochs on fresh 25k examples from CIFAR10 with real labels ({\sc bottom-left}) and 10 random labels ({\sc bottom-right}).
    In each subplot, the left column depicts distributions of neurons over the fraction of input examples that activate them.
    The right column depicts distribution of the input examples over the fraction of neurons that are activated by them.
    The 5k input examples were taken from the holdout test split of CIFAR10.
    }
  \label{fig:VGGNeuronActivationSupplementary}
\end{figure}

\begin{figure}[tbp]
  \centering
  \includegraphics[width=.46\textwidth]{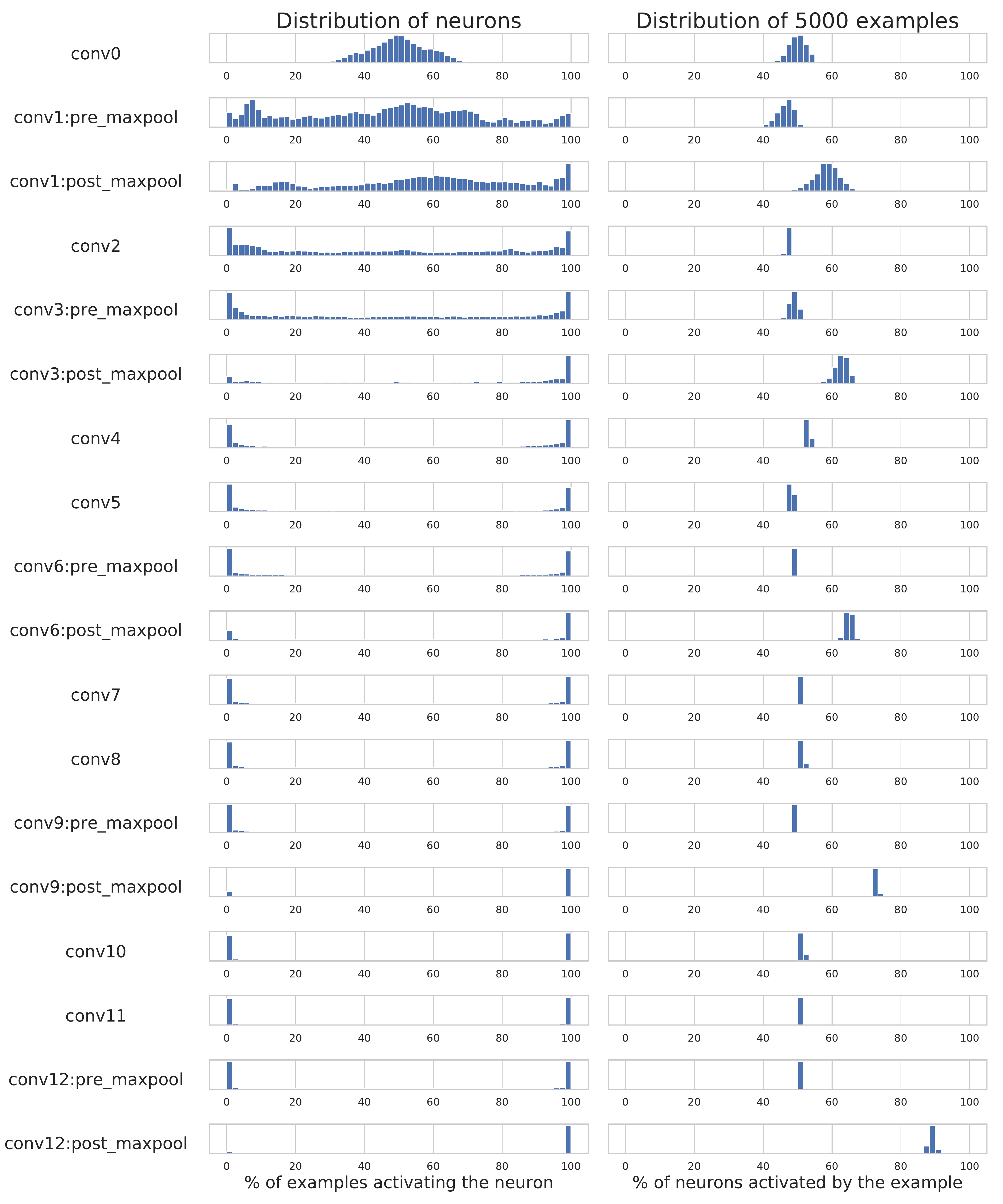}\hfill
  \includegraphics[width=.46\textwidth]{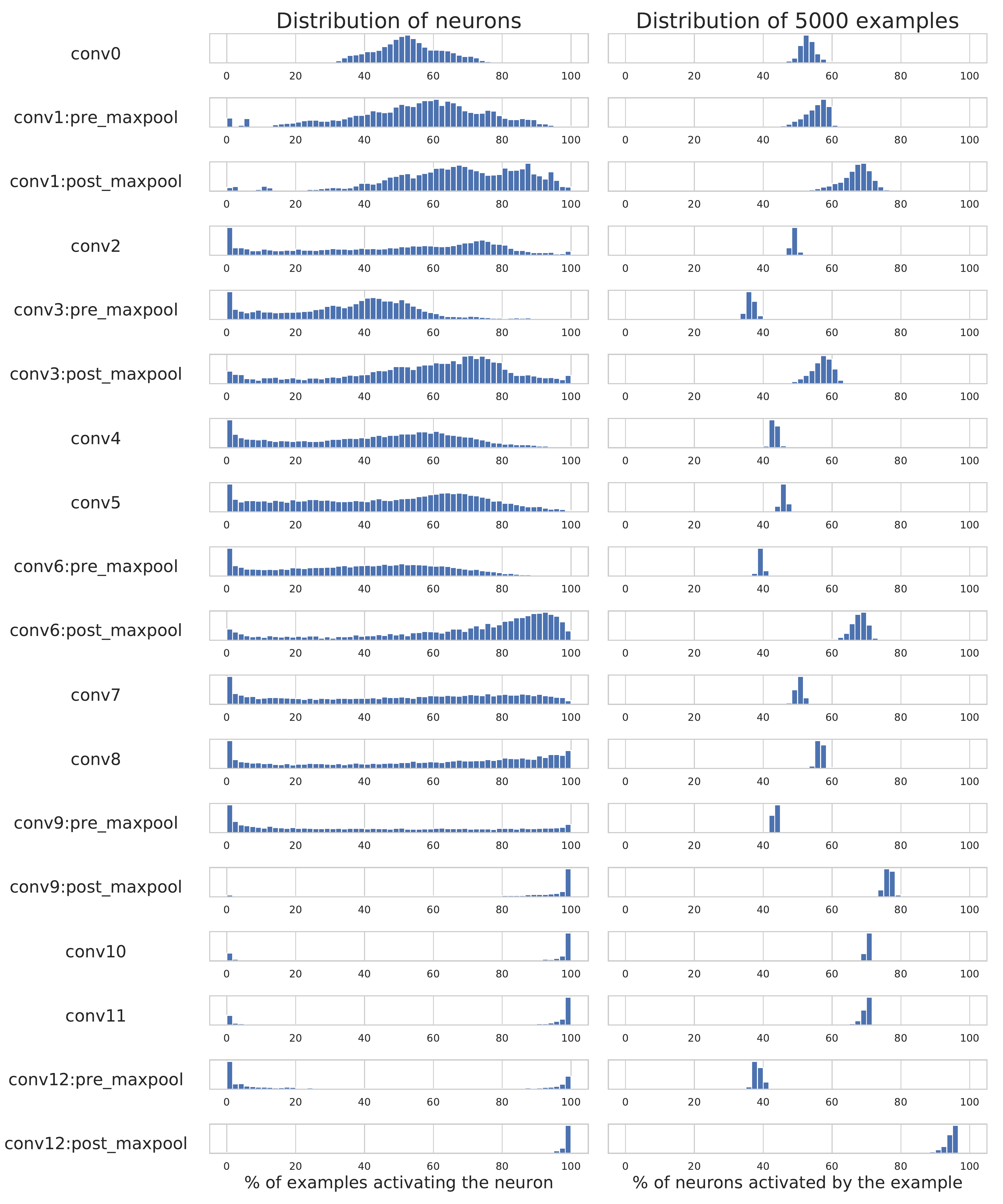}\\[1ex]
  \includegraphics[width=.46\textwidth]{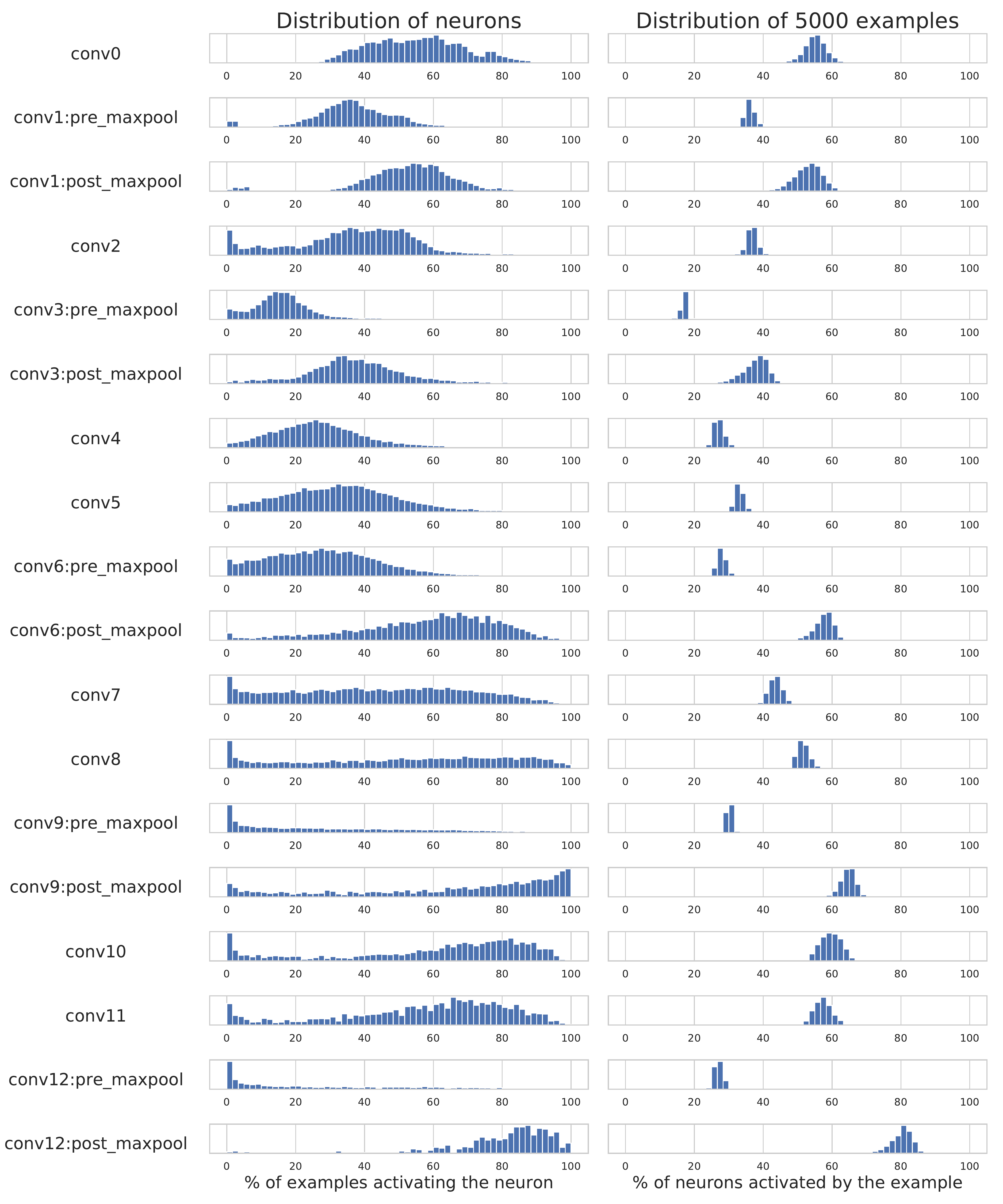}\hfill
  \includegraphics[width=.46\textwidth]{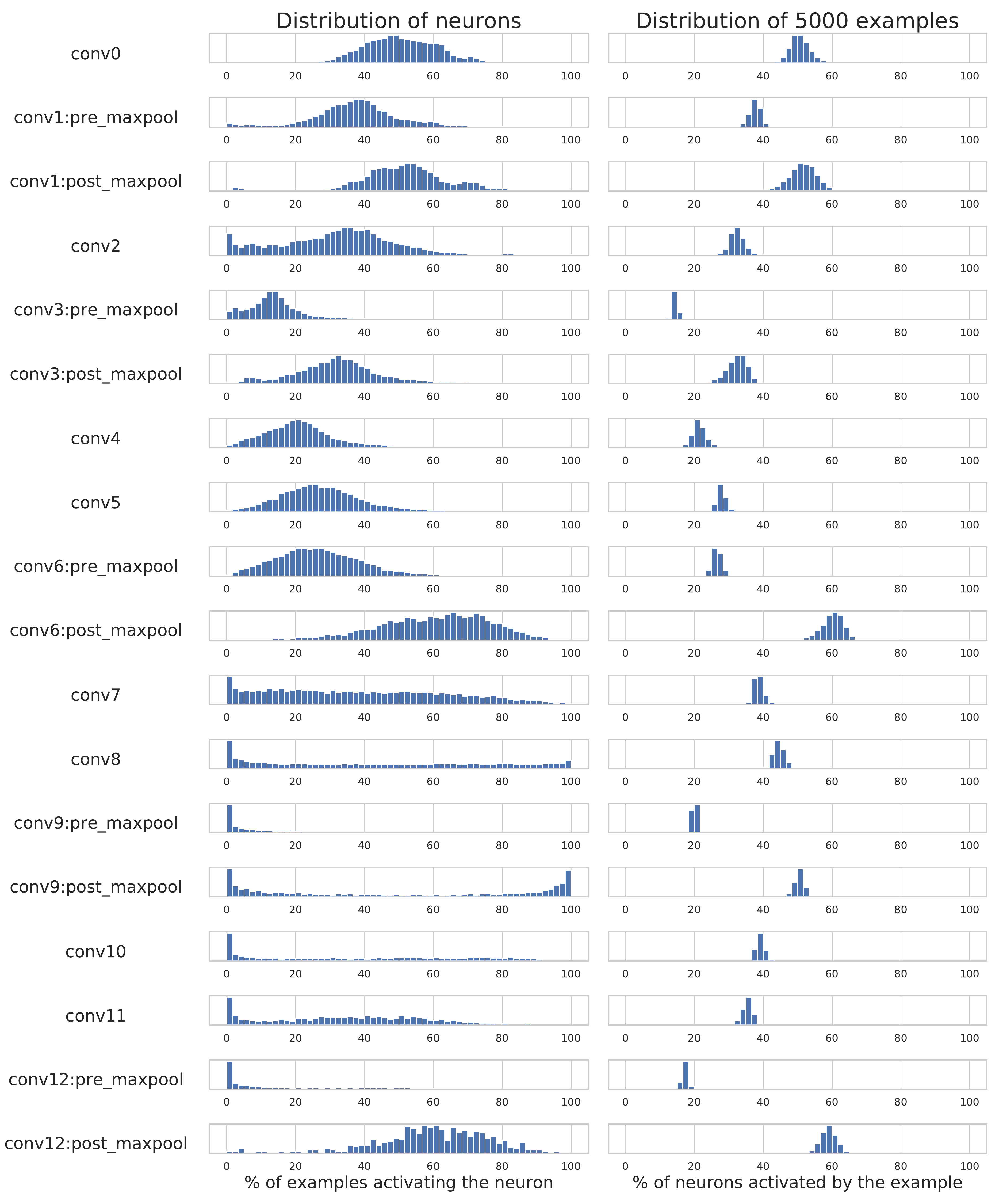}
  \caption{
    Neuron activations in case of the positive transfer.
    The VGG16 model
    pre-trained for 1 ({\sc top-left}) and 9360 ({\sc top-right}) training iterations on 20k examples from CIFAR10 with 5 random classes and
    subsequently fine-tuned for 200 epochs on fresh 25k examples from CIFAR10 with real labels ({\sc bottom-left}) and 10 random labels ({\sc bottom-right}).
    In each subplot, the left column depicts distributions of neurons over the fraction of input examples that activate them.
    The right column depicts distribution of the input examples over the fraction of neurons that are activated by them.
    The 5k input examples were taken from the holdout test split of CIFAR10.
    }
  \label{fig:VGGNeuronActivationSupplementaryPositive}
\end{figure}

Figures \ref{fig:VGGNeuronActivationSupplementary} and \ref{fig:VGGNeuronActivationSupplementaryPositive} are the extended versions of Figure~\ref{fig:VGGNeuronActivation}.
They include activation plots for \emph{all intermediate layers} of VGG16 models
(i) at initialization, 
(ii) in the end of the pre-training, 
(iii) in the end of fine-tuning on real labels,
(iv) in the end of fine-tuning on 10 random labels.
Figures \ref{fig:VGGNeuronActivationSupplementary} and \ref{fig:VGGNeuronActivationSupplementaryPositive} illustrate the negative and positive transfer examples respectively.

The VGG16 model from Figure \ref{fig:VGGNeuronActivationSupplementary} (top row in Figure \ref{fig:SimpleCNNIncreasingWIdth}) uses the same setup as in Experiments 3 and 4 in Figure \ref{fig:image1}:
\begin{verbatim}
init_scale = 0.612
learning_rate = 0.01
num_classes_upstream = 50
num_examples_upstream = 20000
epochs_upstream = 80
\end{verbatim}

The VGG16 model from Figure \ref{fig:VGGNeuronActivationSupplementaryPositive} (bottom row in Figure \ref{fig:SimpleCNNIncreasingWIdth}) uses the same setup as in Experiments 1 and 2 in Figure \ref{fig:image1}:
\begin{verbatim}
init_scale = 0.526 
learning_rate = 0.01
num_classes_upstream = 5
num_examples_upstream = 20000
epochs_upstream = 120
\end{verbatim}

\subsection{Figure \ref{fig:SimpleCNNIncreasingWIdth}: Increasing the width mitigates the negative transfer}

The models from all three subplots use Simple CNN architecture and share the same parameters:
\begin{verbatim}
init_scale = 1.218
learning_rate = 0.012
num_classes_upstream = 2
num_examples_upstream = 10000
epochs_upstream = 100
num_conv_layers = 2
num_filters = 16
\end{verbatim}
The models from {\sc left}, {\sc center}, and {\sc right} subplots use \texttt{num\_units} of 64, 128, and 1\,024, respectively.
Error bars correspond to max/min over 3 independent runs.

\clearpage
\section{Empirical evidence with diverse real-world settings}

We argued theoretically that the alignment effect holds under certain idealized conditions (Proposition~\ref{prop:1})
and demonstrated that it entirely explains the positive transfer in a real-world setting (Figure~\ref{fig:PCA_init}).
In this appendix, we confirm that the positive transfer is reproducible and can be frequently observed in common real-world settings with various popular 
network architectures,
datasets, and different
hyperparameters.

{\bf Experimental setup}
We run multiple transfer experiments, repeating each run with several different seeds.
We consider three network architectures: 
(a) a simple convolutional architecture with a configurable number of conv.\:layers, filters, and units in one dense layer,
(b) VGG16 \cite{simonyan2014very} with two final dense layers of width 4096 removed, 
and (c) ResNet-v2 architecture \cite{He2016}.
The \emph{disjoint} upstream and downstream training sets are sampled randomly from one of the two datasets: 
CIFAR10~\cite{Krizhevsky09learningmultiple} or
ImageNet ILSVRC-2012~\cite{deng2009imagenet}. 
We never transfer \emph{between different} datasets. 
The models are pre-trained with random labels and fine-tuned for a fixed number of epochs with either real or random labels using SGD with momentum 0.9. 
We randomly explore various 
initial learning rates,
initialization types and scales,
numbers of random classes and examples upstream,
duration of the pre-training,
and configurations of the Simple CNN architecture.

We collect two sets of experiments for CIFAR10 (Experiments A and B) and two sets for ImageNet (Experiments A and B) reported below.
Each set consists of multiple \emph{groups} of experiments. 
Experiments are gathered in a \emph{group} if they share same 
(1)~architecture, 
(2)~learning rate,
(3)~initialization type and scale,
(4)~number of examples upstream, 
(5)~number of classes upstream, and
(6)~number of epochs upstream. 
For each configuration of these 6~parameters we explore \emph{all possible combinations} of the following choices:
(a)~the random seed,
(b)~pre-train / train from scratch on downstream,
(c)~transfer one layer / all layers [optional],
(d)~train downstream with real / random labels.
Gathering experiments in groups like this allows to compare the downstream performance of the pre-trained models to that of models trained from scratch \emph{with same hyperparameters}.

{\bf Visualizing the experiments}
In order to visualize the experiments we summarize each group with two numbers: one characterizing the downstream performance of the pre-trained models, and one for the models trained from scratch.
In order to capture the speed of training we use the \emph{area under the curve} (AUC) to sketch the training with a single number. 
For instance, the pre-trained model in Experiment 2 of Figure \ref{fig:image1}  (blue line) trains faster than the one trained from scratch (orange line).
Accordingly, the area under the blue curve is larger than the area under the orange curve.
We can now visualize all the groups of experiments on a single scatter plot.
Depending on what exact curves we summarize with AUC, we get three scatter plots: 
(1)~training accuracy when using real labels downstream (areas under the solid lines in Experiments 1 and 3 in Figure \ref{fig:image1}),
(2)~test accuracy when using real labels downstream (areas under the dotted lines in Experiments 1 and 3 in Figure \ref{fig:image1}),
(3)~training accuracy when using random labels downstream (areas under the solid lines in Experiments 2 and 4 in Figure \ref{fig:image1}).

\subsection{CIFAR10, Experiments A}
This set of experiments counts 20 groups per architecture.
We use 2 different random seeds, resulting in 16 experiments per group, or $3 \times 20 \times 16 =960$ trained models in total.

The following parameters are shared between all runs:
\begin{verbatim}
num_examples_downstream = 25000
num_epochs_downstream = 200
init_algorithm = "he"
\end{verbatim}
For each group of experiments we sample the following parameters randomly:
\begin{verbatim}
init_scale = random log_uniform(0.5, 1.35)
learning_rate = random log_uniform(0.008, 0.0125)
num_examples_upstream = random uniform([5000, 10000, 15000, 20000, 25000])
num_classes_upstream = random uniform([2, 5, 10, 25, 50, 100])
num_epochs_upstream = random uniform([20, 40, 80, 100, 120, 150, 200])
\end{verbatim}
For the Simple CNN architecture we also randomly sample
\begin{verbatim}
num_conv_layers = random uniform([1, 2, 3, 4, 5])
num_units = random uniform([64, 128, 256, 512, 1024])
num_filters = random uniform([16, 32])
\end{verbatim}

Exmeriments A are depicted on Figure \ref{fig:CIFAR10SweepV2}.
Two clusters of points in the {\sc center} row correspond to models with VGG16 architecture ($y > 0.7$) and models with SimpleCNN or ResNet18 architecture ($y \leq 0.7$).
Surprisingly, several experiments with transferring one (first) layer of VGG16 when fine-tuning with real labels lead to improved test accuracy compared to the same models trained from scratch (blue squares, {\sc center} row).

\subsection{CIFAR10, Experiments B}

This set of experiments contains 20 groups per architecture.
We use 2 different random seeds, resulting in 16 experiments per group, or $3 \times 20 \times 16 =960$ trained models in total.

We use slightly wider hyperparameter ranges compared to the CIFAR10 Experiments A.
Also, we include the orthogonal initialization algorithm \cite{saxe2014}.
The following parameters are shared between all runs:
\begin{verbatim}
num_examples_downstream = 25000
num_epochs_downstream = 200
\end{verbatim}
For each group of experiments we sample the following parameters randomly:
\begin{verbatim}
init_algorithm = random uniform(["he", "orthogonal"])
init_scale = random log_uniform(0.5, 2.)
learning_rate = random log_uniform(0.005, 0.02)
num_examples_upstream = random uniform([5000, 10000, 15000, 20000, 25000])
num_classes_upstream = random uniform([2, 5, 10, 25, 50, 100])
num_epochs_upstream = random uniform([20, 40, 80, 100, 120, 150, 200])
\end{verbatim}
For the Simple CNN architecture we also randomly sample
\begin{verbatim}
num_conv_layers = random uniform([1, 2, 3, 4, 5])
num_units = random uniform([512, 1024, 2048])
num_filters = random uniform([16, 32])
\end{verbatim}

Exmeriments B are depicted on Figure \ref{fig:CIFAR10SweepV1}.
The models with $y > 0.7$ on {\sc center} row again correspond to the VGG16 architecture.
This time we do not observe improvements in the test performance when pre-training.

\subsection{ImageNet, Experiments A}
\label{section:imagenet-experiments-a}

This set of experiments counts 10 groups per architecture.
We use 5 different random seeds.
We always transfer all layers.
This results in 20 experiments per group, or $3 \times 10 \times 20 = 600$ trained models in total.

The following parameters are shared between all runs:
\begin{verbatim}
num_examples_downstream = 200000
num_epochs_downstream = 100
num_epochs_upstream = 100
num_classes_upstream = 1000
init_algorithm = "he"
\end{verbatim}
For each group of experiments we sample the following parameters randomly:
\begin{verbatim}
init_scale = random log_uniform(0.5, 1.5)
learning_rate = random log_uniform(0.008, 0.0125)
num_examples_upstream = random uniform([50000, 100000])
\end{verbatim}
For the Simple CNN architecture we also randomly sample
\begin{verbatim}
num_conv_layers = random uniform([1, 2, 3])
num_units = random uniform([32, 64, 128, 256, 512, 1024])
num_filters = random uniform([16, 32, 64])
\end{verbatim}

ImageNet Exmeriments A are depicted on Figure \ref{fig:ImageNetSweepV1}.
All the ResNet18 runs result in positive transfer in terms of the downstream training performance, i.e. all the green points on {\sc top} and {\sc bottom} plots are below the diagonal.
This may be caused by the fact that we do not re-scale ResNet18 models after pre-training, which means the effective downstream learning rate of the pre-trained ResNets18 may be affected.

\subsection{ImageNet, Experiments B}

This set of experiments counts 10 groups per architecture.
We use 5 different random seeds.
We always transfer all layers.
This results in 20 experiments per group, or $3 \times 10 \times 20 = 600$ trained models in total.

We use smaller upstream and downstream training sizes compared to ImageNet Experiments A.

The following parameters are shared between all runs:
\begin{verbatim}
num_examples_downstream = 50000
num_epochs_downstream = 100
num_epochs_upstream = 100
num_classes_upstream = 1000
init_algorithm = "he"
\end{verbatim}
For each group of experiments we sample the following parameters randomly:
\begin{verbatim}
init_scale = random log_uniform(0.5, 1.5)
learning_rate = random log_uniform(0.008, 0.0125)
num_examples_upstream = random uniform([25000, 50000])
\end{verbatim}
For the Simple CNN architecture we also randomly sample
\begin{verbatim}
num_conv_layers = random uniform([1, 2, 3])
num_units = random uniform([32, 64, 128, 256, 512, 1024])
num_filters = random uniform([16, 32, 64])
\end{verbatim}

ImageNet Exmeriments B are depicted on Figure \ref{fig:ImageNetSweepV2}.
We see again that all the ResNet18 runs result in positive transfer in terms of the downstream training performance.
However, this time we do not observe the positive transfer with architectures other than ResNet18.

\begin{figure}[tbp]
  \centering \sffamily
  {\large Real Labels Downstream (Training Accuracy)}\\
  \includegraphics[width=\textwidth]{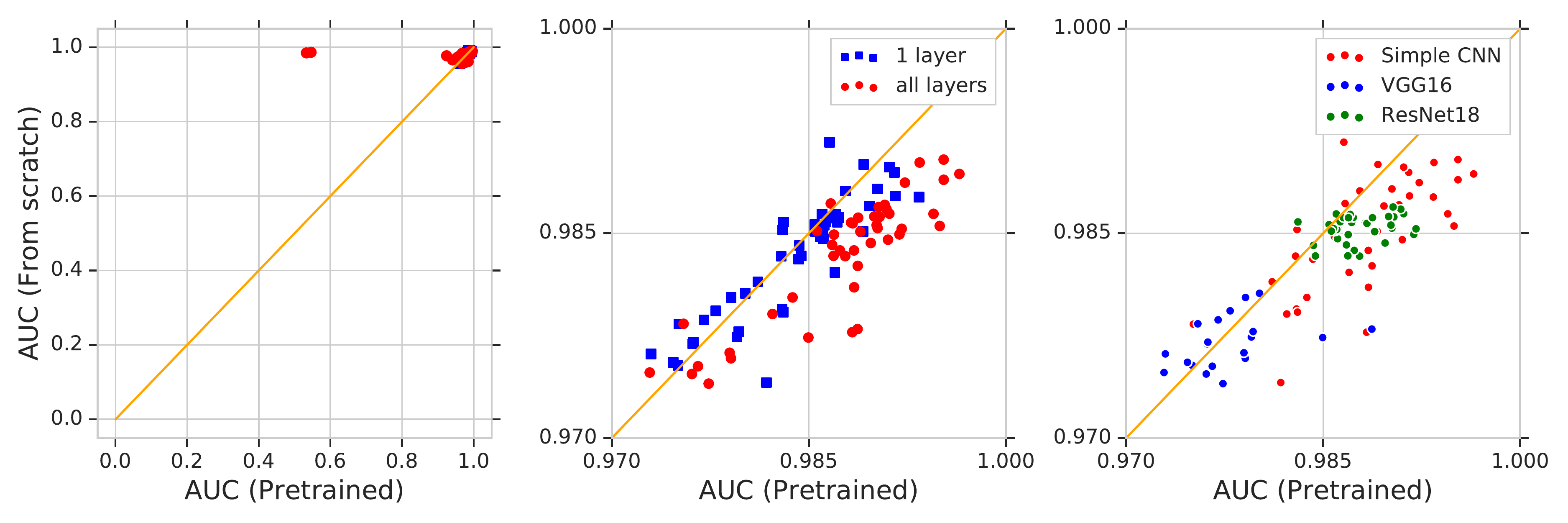}\\
  {\large Real Labels Downstream (Test Accuracy)}\\
  \includegraphics[width=\textwidth]{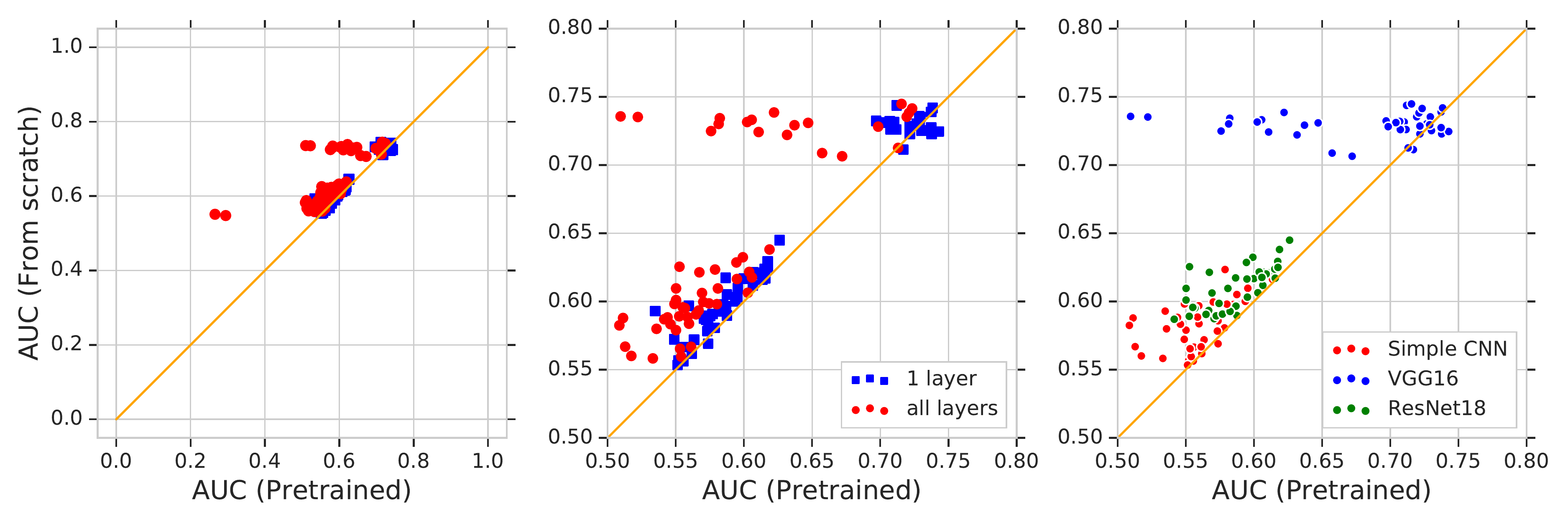}\\
  {\large Random Labels Downstream (Training Accuracy)}\\
  \includegraphics[width=\textwidth]{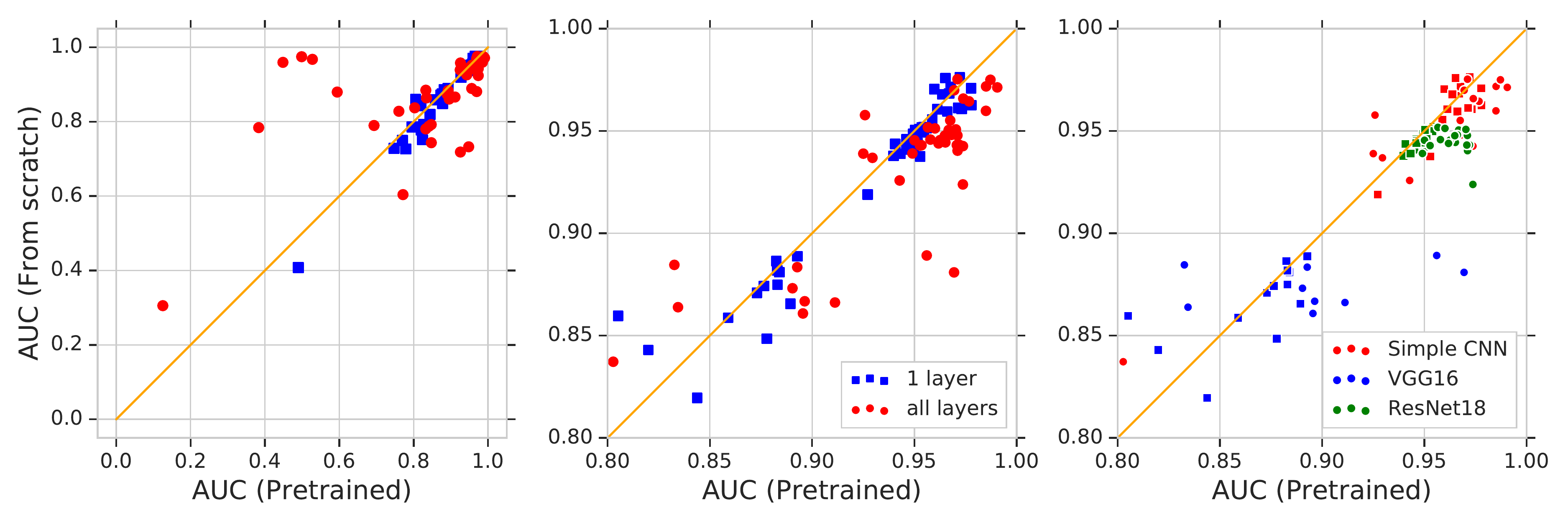}
  \caption{
    Scatter plots for CIFAR10 Experiments A.
    Models fine-tuned on real labels ({\sc top} and {\sc center}) and on random labels ({\sc bottom}).
    Each dot corresponds to one \emph{group of experiments}.
    Points below the orange $x=y$ line on {\sc top} correspond to experiments where models pre-trained with random labels train faster downstream (with real labels) compared to models trained from scratch \emph{with same hyperparameters}.
    {\sc Center and right} columns contain zoomed in versions of the plots from {\sc left} column. 
    }
  \label{fig:CIFAR10SweepV2}
\end{figure}

\clearpage
\begin{figure}[tbp]
  \centering \sffamily
  {\large Real Labels Downstream (Training Accuracy)}\\
  \includegraphics[width=\textwidth]{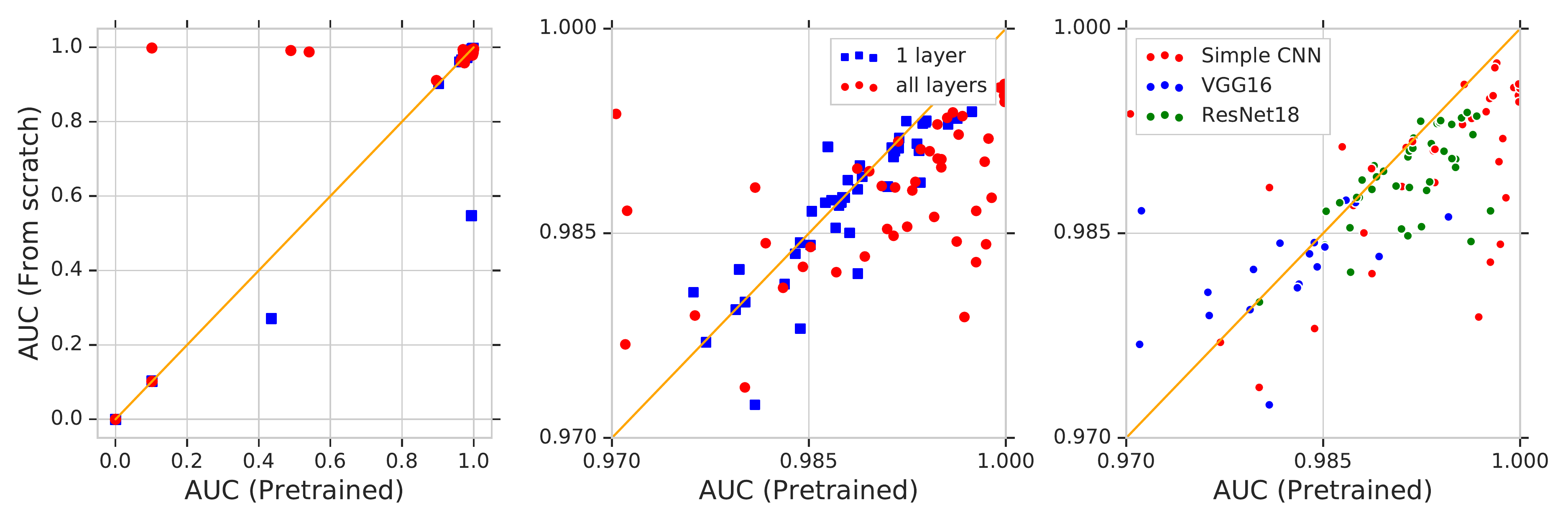}\\
  {\large Real Labels Downstream (Test Accuracy)}\\
  \includegraphics[width=\textwidth]{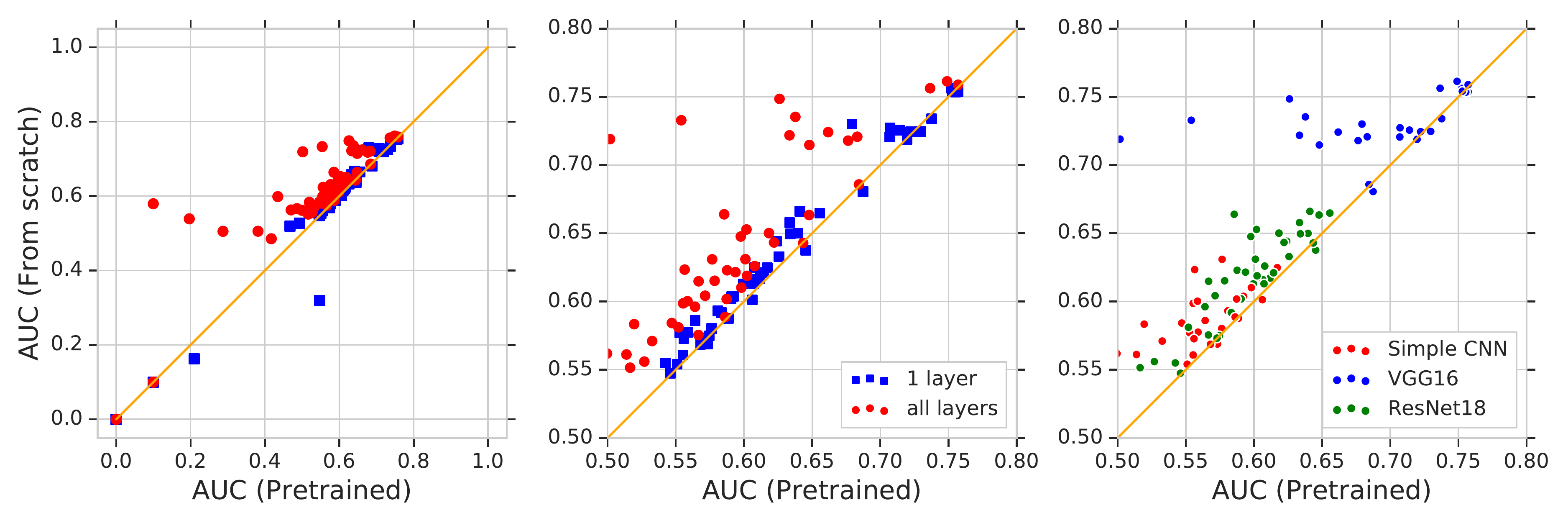}\\
  {\large Random Labels Downstream (Training Accuracy)}\\
  \includegraphics[width=\textwidth]{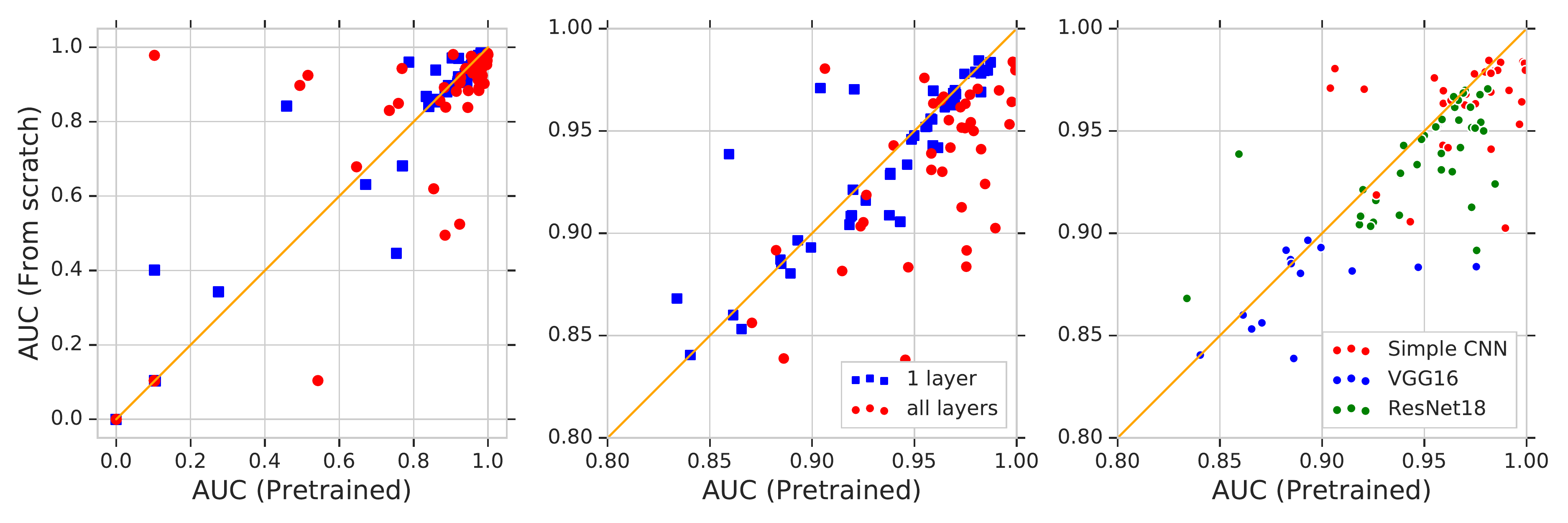}
  \caption{
    Scatter plots for CIFAR10 Experiments B.
    Models fine-tuned on real labels ({\sc top} and {\sc center}) and on random labels ({\sc bottom}).
    Each dot corresponds to one \emph{group of experiments}.
    Points below the orange $x=y$ line on {\sc top} correspond to experiments where models pre-trained with random labels train faster downstream (with real labels) compared to models trained from scratch \emph{with same hyperparameters}.
    {\sc Center and right} columns contain zoomed in versions of the plots from {\sc left} column.
    }
  \label{fig:CIFAR10SweepV1}
\end{figure}

\clearpage
\begin{figure}[tbp]
  \centering \sffamily
  {\large Real Labels Downstream (Training Accuracy)}\\
  \includegraphics[width=.8\textwidth]{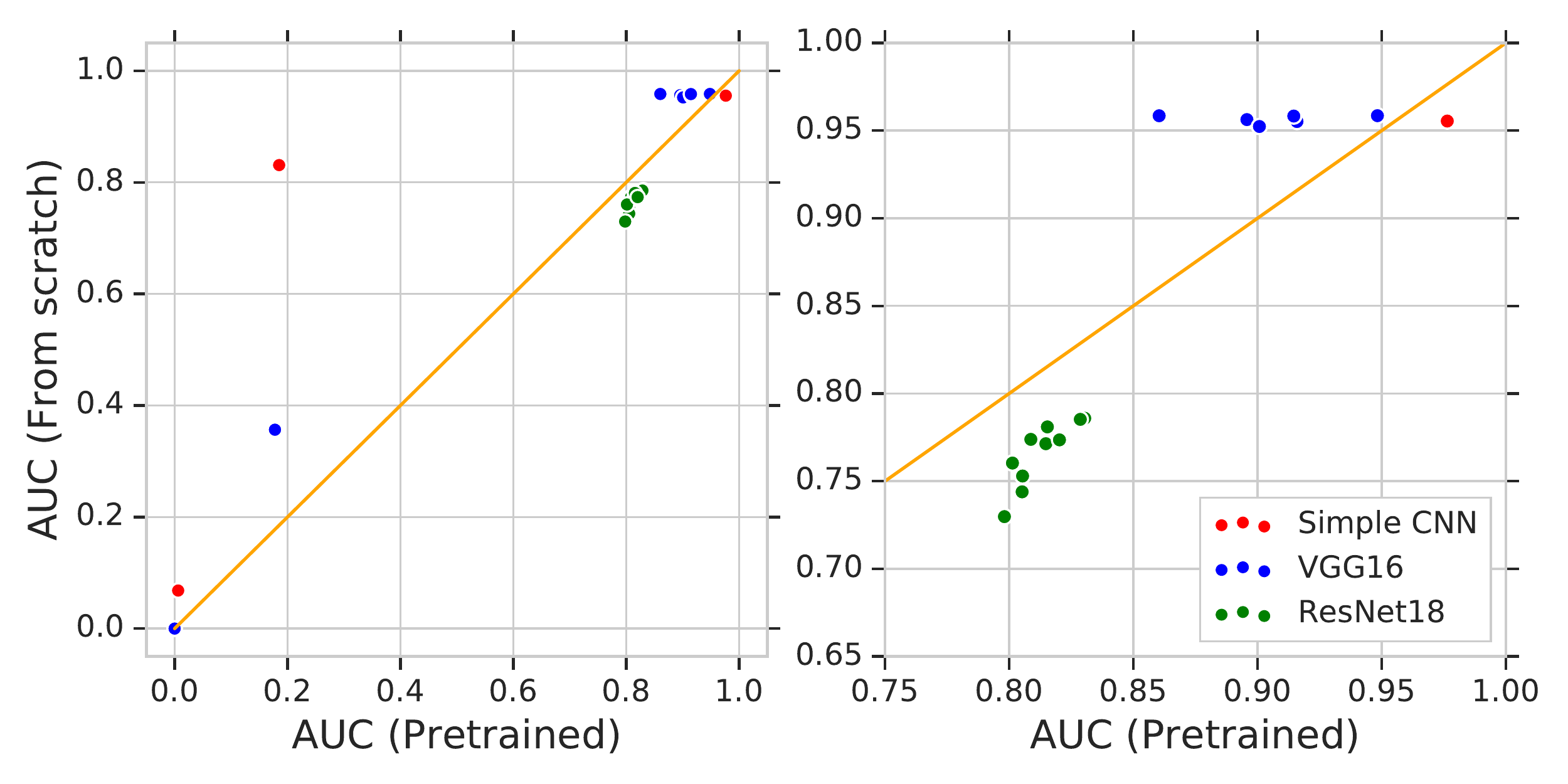}\\
  {\large Real Labels Downstream (Test Accuracy)}\\
  \includegraphics[width=.8\textwidth]{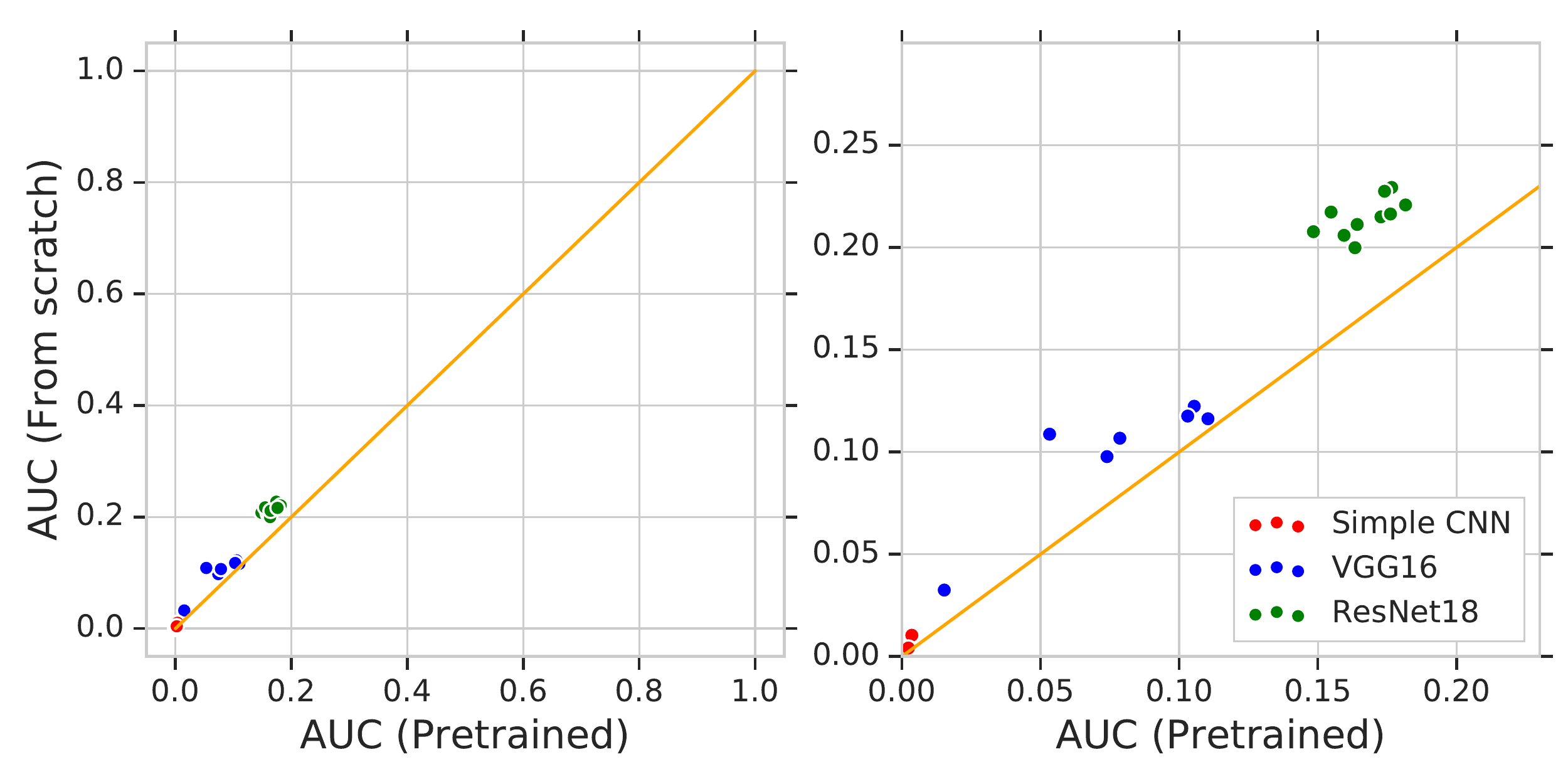}\\
  {\large Random Labels Downstream (Training Accuracy)}\\
  \includegraphics[width=.8\textwidth]{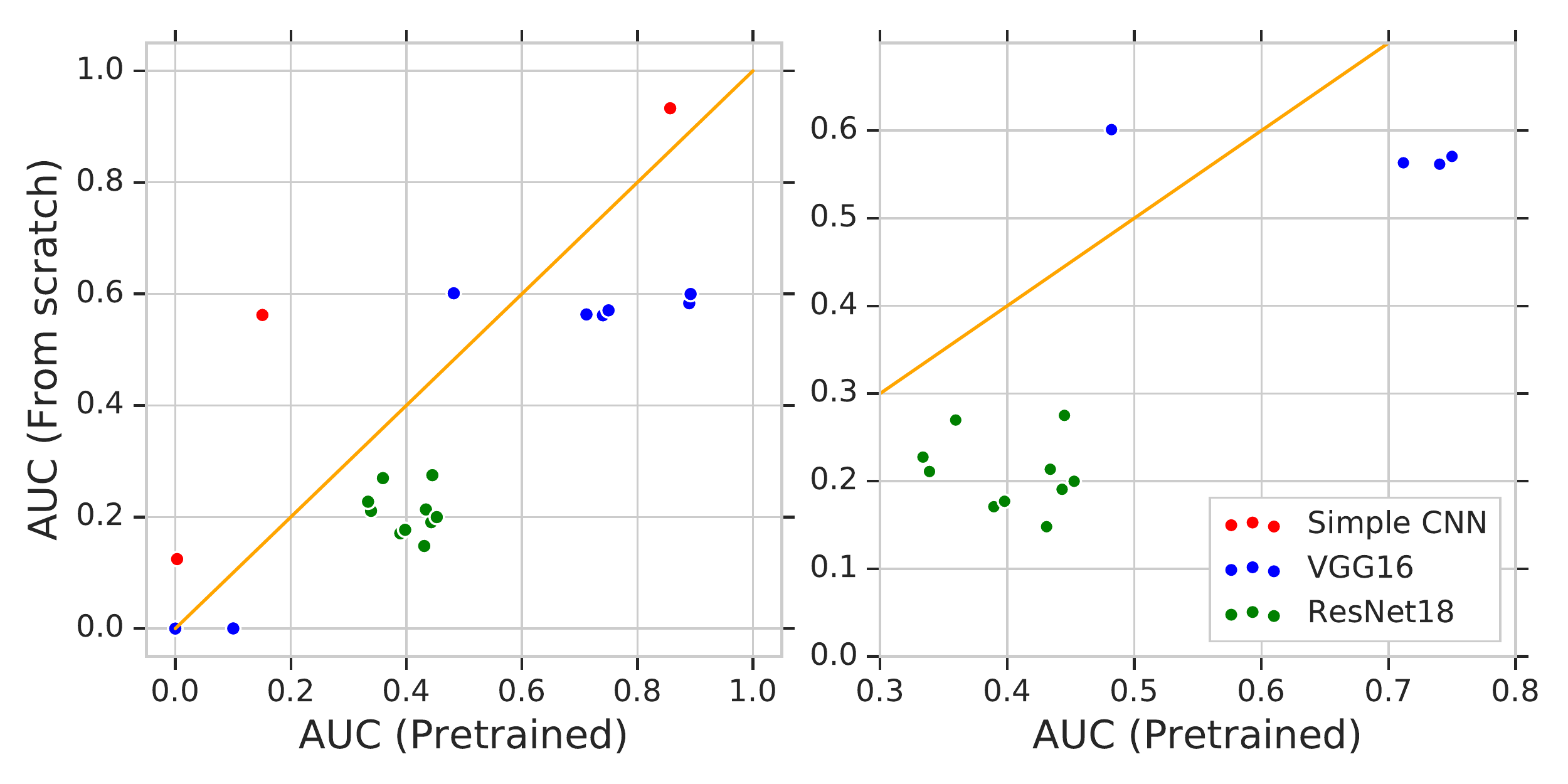}
  \caption{
    Scatter plots for ImageNet Experiments A.
    Models fine-tuned on real labels ({\sc top} and {\sc center}) and on random labels ({\sc bottom}).
    Each dot corresponds to one \emph{group of experiments}.
    Points below the orange $x=y$ line on {\sc top} correspond to experiments where models pre-trained with random labels train faster downstream (with real labels) compared to models trained from scratch \emph{with same hyperparameters}.
    {\sc Right} column contains zoomed in versions of the plots from {\sc left} column.
    }
  \label{fig:ImageNetSweepV1}
\end{figure}

\clearpage
\begin{figure}[tbp]
  \centering \sffamily
  {\large Real Labels Downstream (Training Accuracy)}\\
  \includegraphics[width=.8\textwidth]{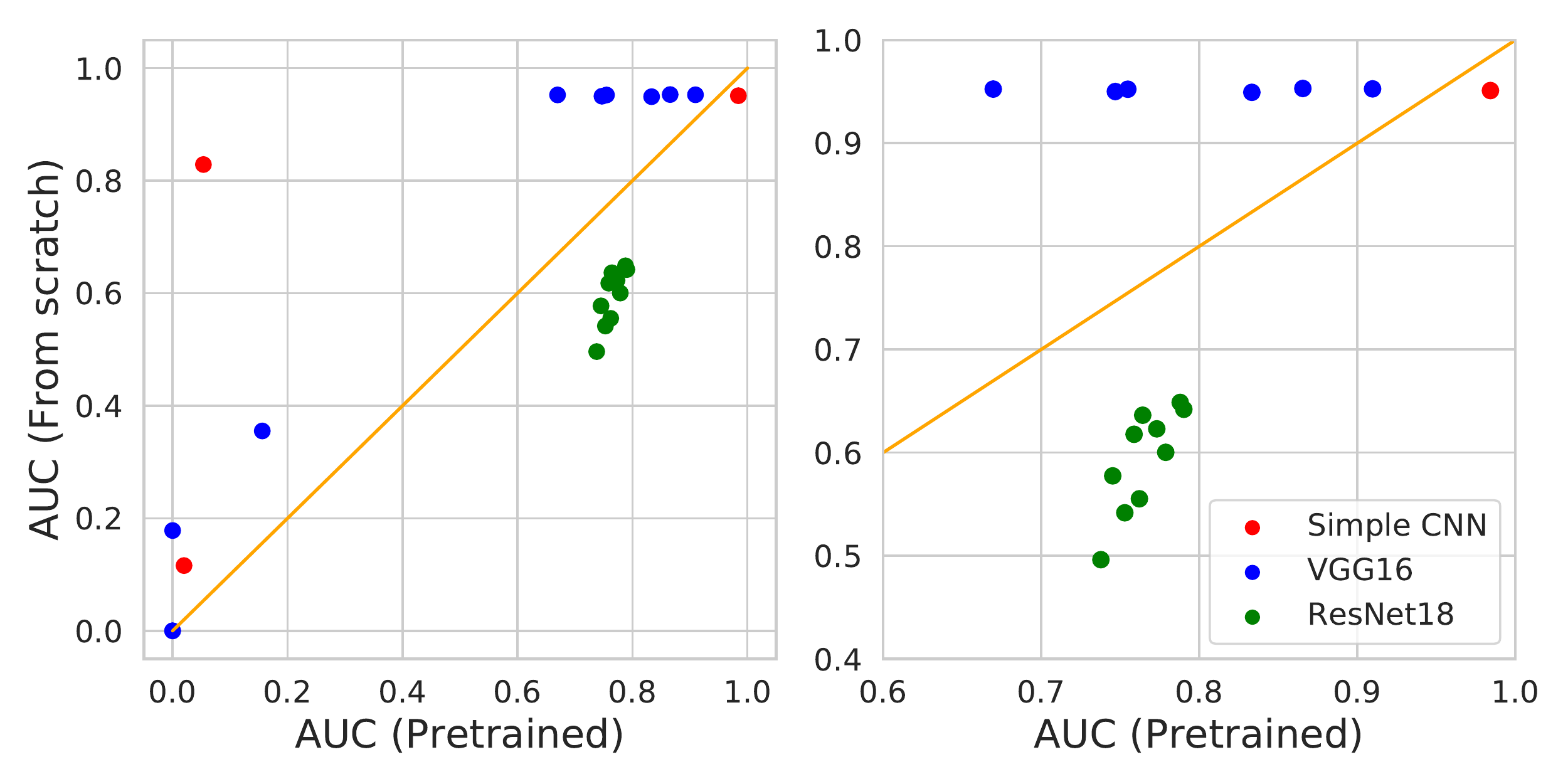}\\
  {\large Real Labels Downstream (Test Accuracy)}\\
  \includegraphics[width=.8\textwidth]{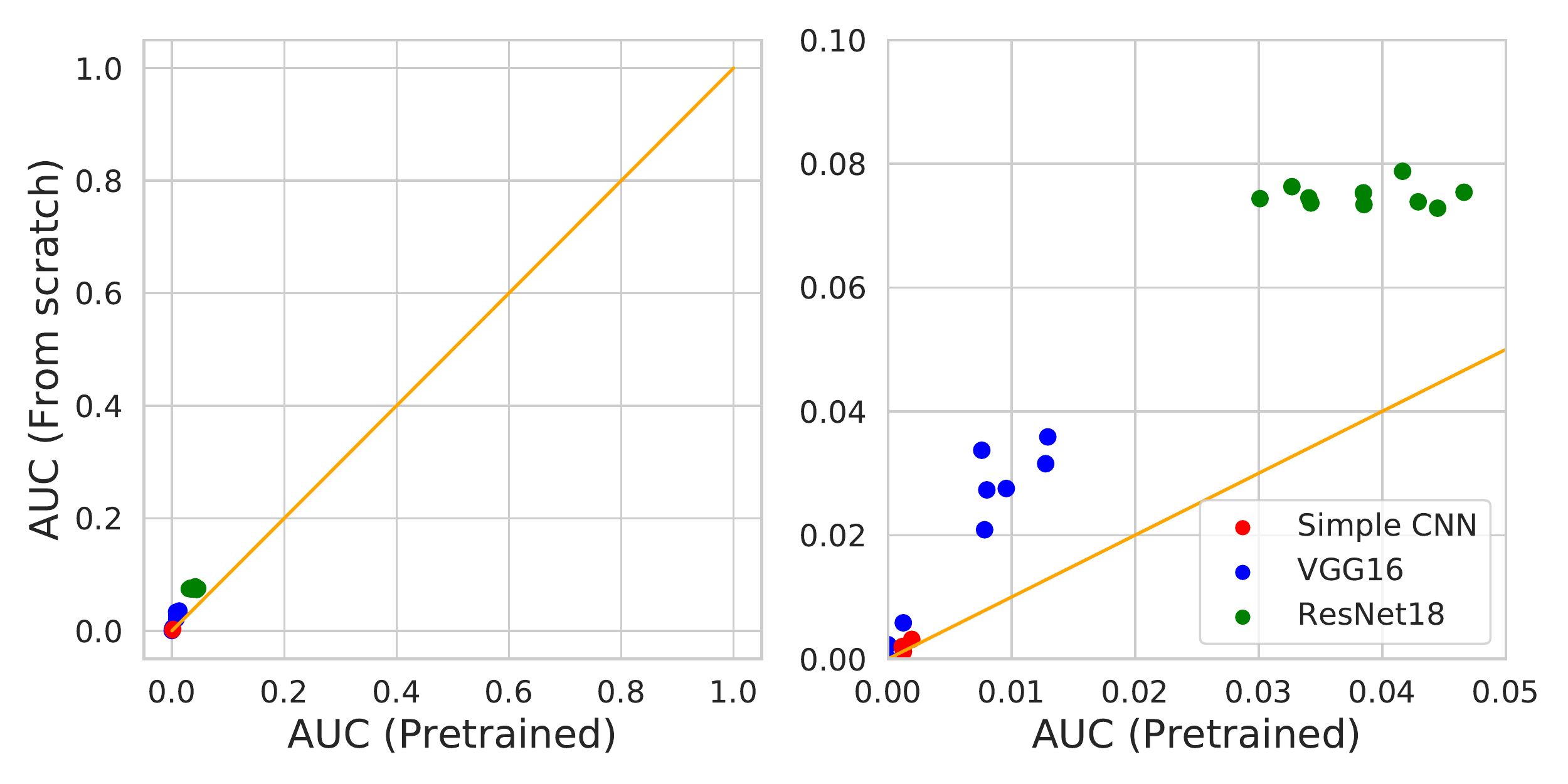}\\
  {\large Random Labels Downstream (Training Accuracy)}\\
  \includegraphics[width=.8\textwidth]{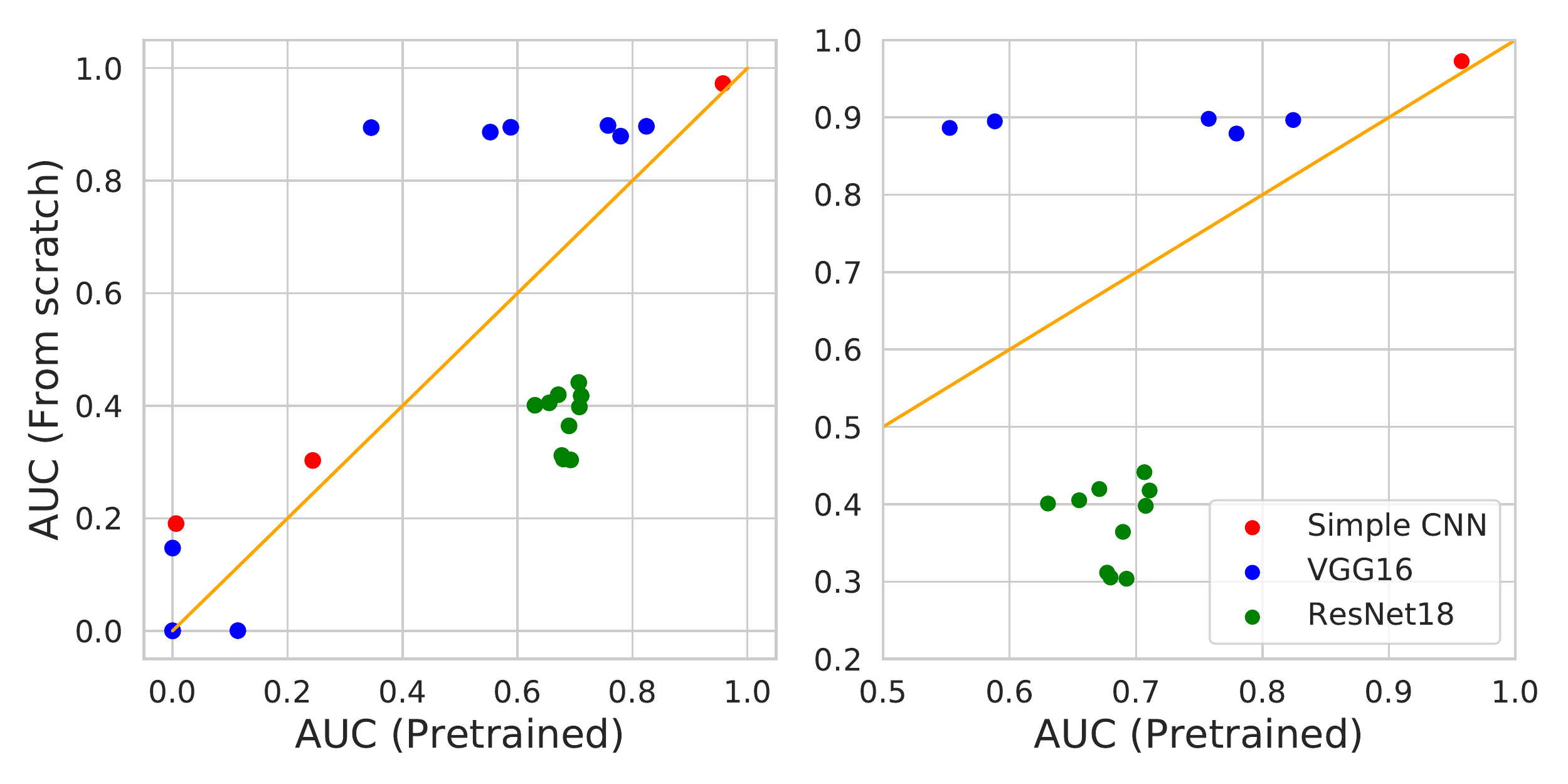}
  \caption{
    Scatter plots for ImageNet Experiments B.
    Models fine-tuned on real labels ({\sc top} and {\sc center}) and on random labels ({\sc bottom}).
    Each dot corresponds to one \emph{group of experiments}.
    Points below the orange $x=y$ line on {\sc top} correspond to experiments where models pre-trained with random labels train faster downstream (with real labels) compared to models trained from scratch \emph{with same hyperparameters}.
    {\sc Right} column contains zoomed in versions of the plots from {\sc left} column.
    }
  \label{fig:ImageNetSweepV2}
\end{figure}

\clearpage
\subsection{Summary of the experiments}
Analysis of experiments presented in the previous four sections reveals a large number of positive transfers where the pre-trained model \emph{trains faster} on the downstream task compared to the model trained from scratch using the \emph{same hyperparameters}.
Generally it looks like transferring more layers leads to stronger positive effects.
All ResNet18 runs performed on ImageNet (and most of the ones on CIFAR10) result in positive transfer (in terms of the downstream training accuracy), which may be caused by the fact that we do not re-scale them after pre-training.

Most likely there are other important factors, yet to be discovered, playing a role in all these experiments. 
However, the evidence presented earlier (Figures~\ref{fig:Experiment_alignment} and~\ref{fig:filters}, Table~\ref{table:experiment_3conv}) suggests that alignment is responsible for the observed positive transfer in at least some of these cases.

\subsection{Detailed results for some of the experiments}

Figure \ref{fig:imagenet-resnet-1} reports experiments with ResNet18 on ImageNet where pre-training helps. These experiments differ only in the in initialization scale and otherwise use:
\begin{verbatim}
learning_rate = 0.01
num_classes_upstream = 1000
num_examples_upstream = 500000
num_epochs_upstream = 100
num_examples_downstream = 500000
num_epochs_downstream = 100
\end{verbatim}

\def\pltwidth{0.35\textwidth}
\begin{figure}[tbp]
  \centering
  \sffamily
  \begin{minipage}{\pltwidth}
    \centering 
    Pre-training helps\\
    Init scale 1.3\\
    (real labels)
  \end{minipage}
  \begin{minipage}{\pltwidth}
    \centering
    Pre-training helps\\
    Init scale 1.3\\
    (random labels)
  \end{minipage}\\
  \includegraphics[width=\pltwidth]{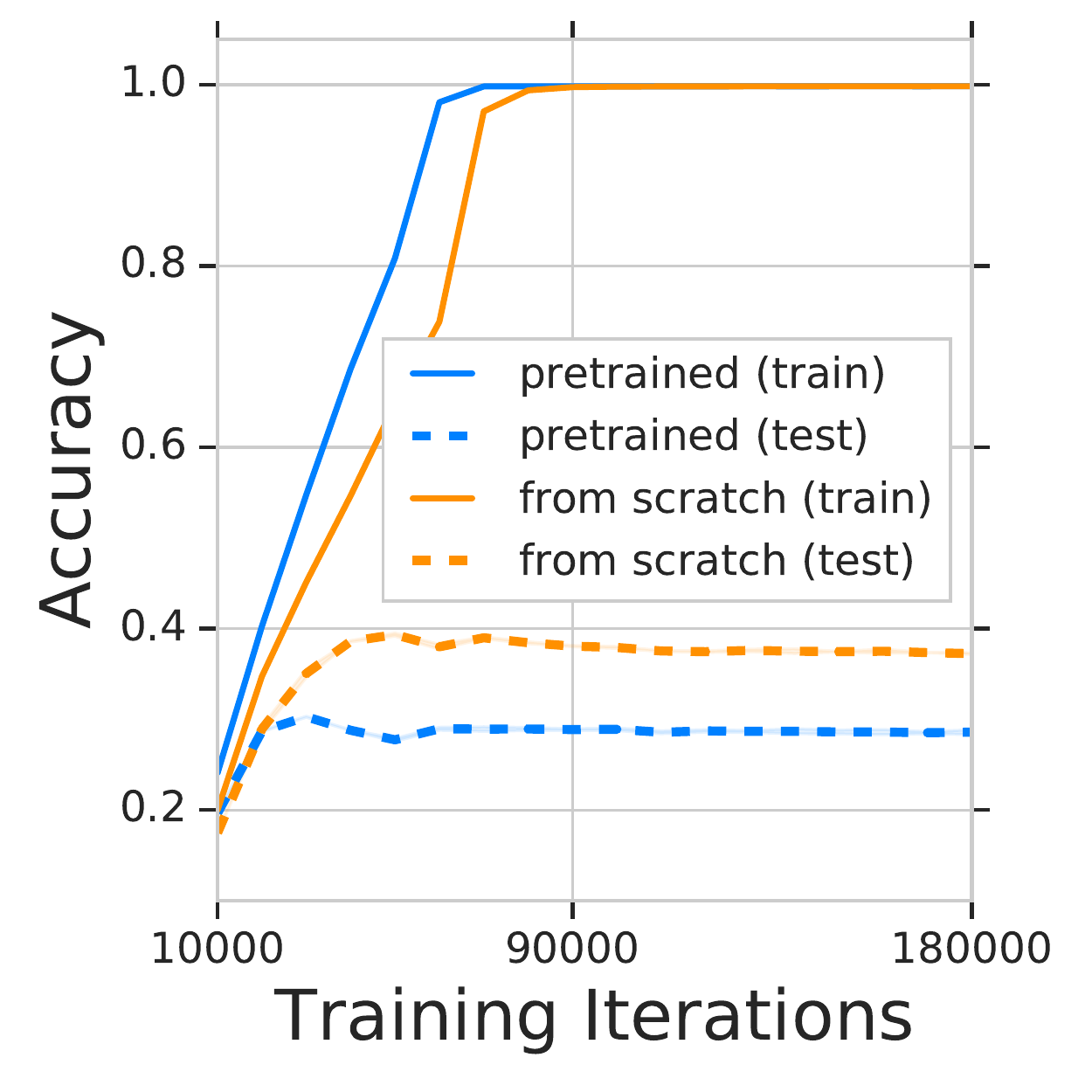}%
  \includegraphics[width=\pltwidth]{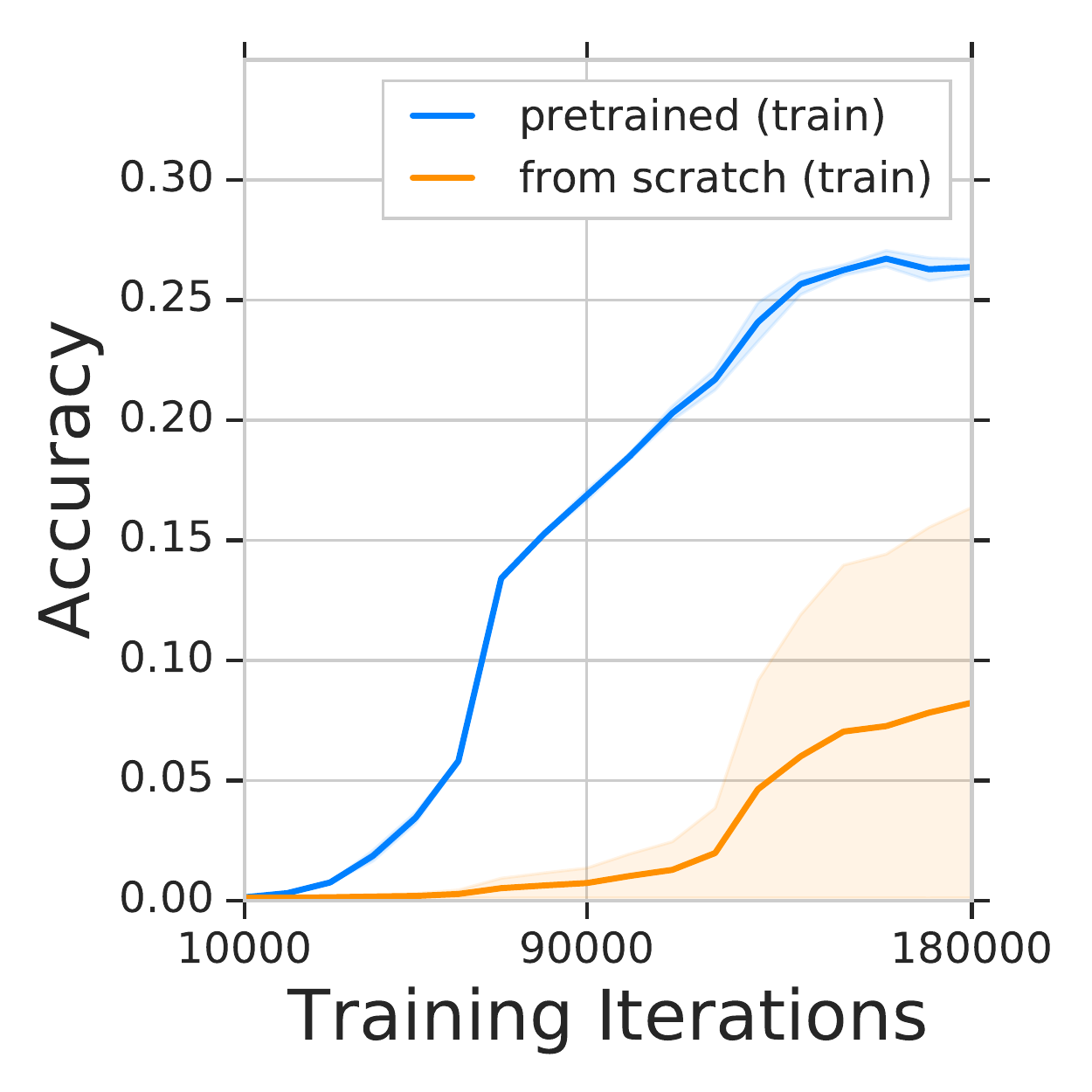}\\
  \begin{minipage}{\pltwidth}
    \centering 
    Pre-training helps\\
    Init scale 1.6\\
    (real labels)
  \end{minipage}
  \begin{minipage}{\pltwidth}
    \centering
    Pre-training helps\\
    Init scale 1.6\\
    (random labels)
  \end{minipage}\\
  \includegraphics[width=\pltwidth]{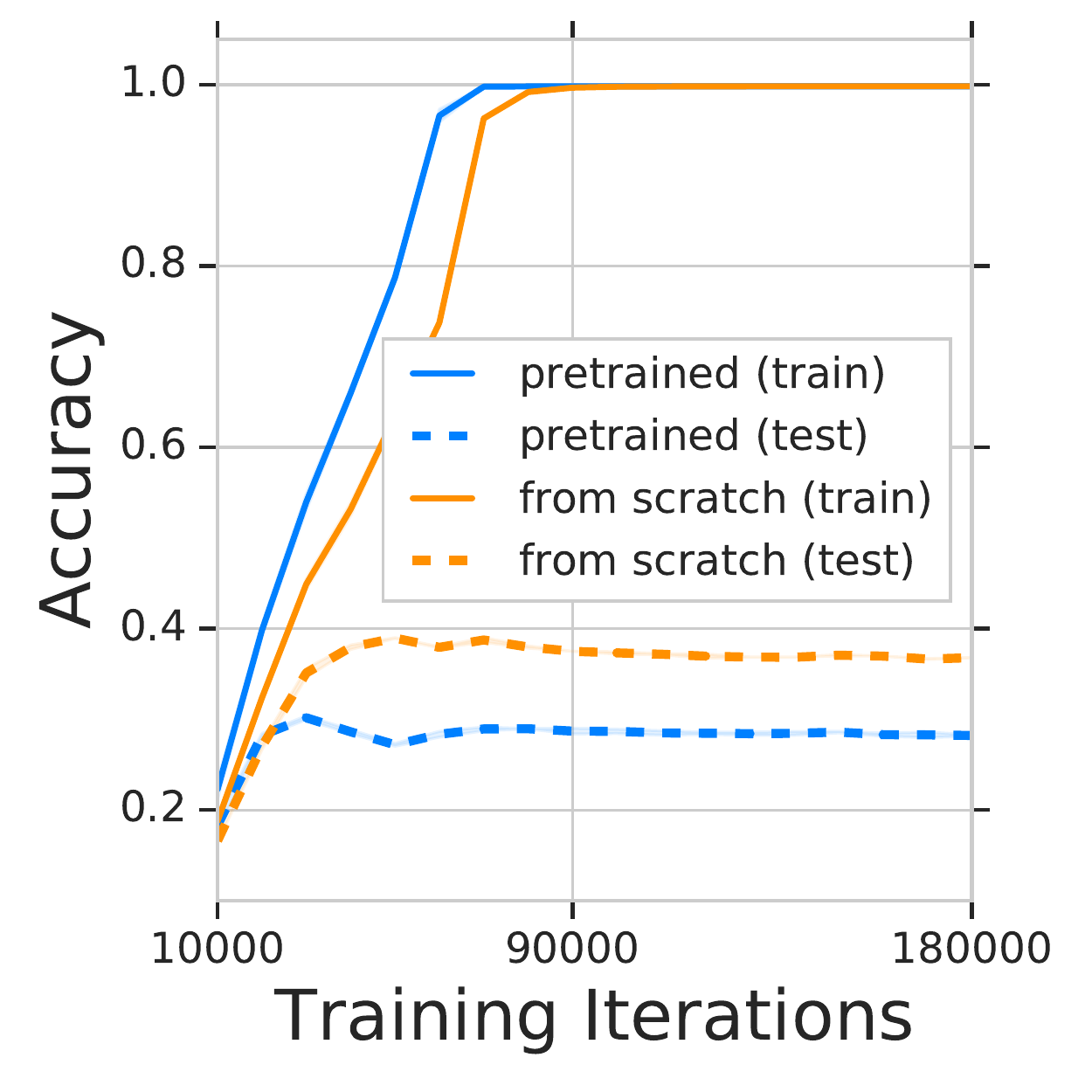}%
  \includegraphics[width=\pltwidth]{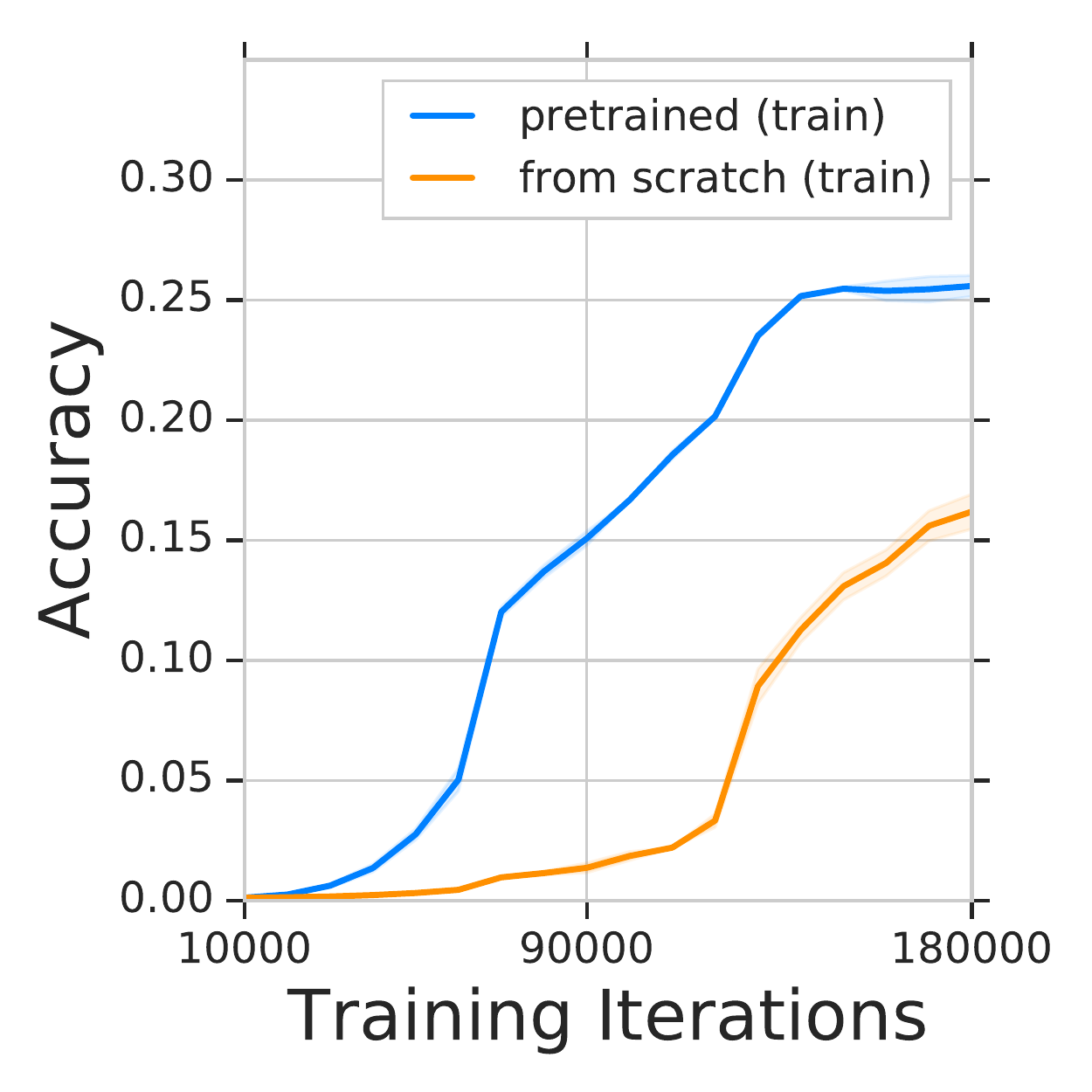}\\
  \begin{minipage}{\pltwidth}
    \centering 
    Pre-training helps\\
    Init scale 1.9\\
    (real labels)
  \end{minipage}
  \begin{minipage}{\pltwidth}
    \centering
    Pre-training helps\\
    Init scale 1.9\\
    (random labels)
  \end{minipage}\\
  \includegraphics[width=\pltwidth]{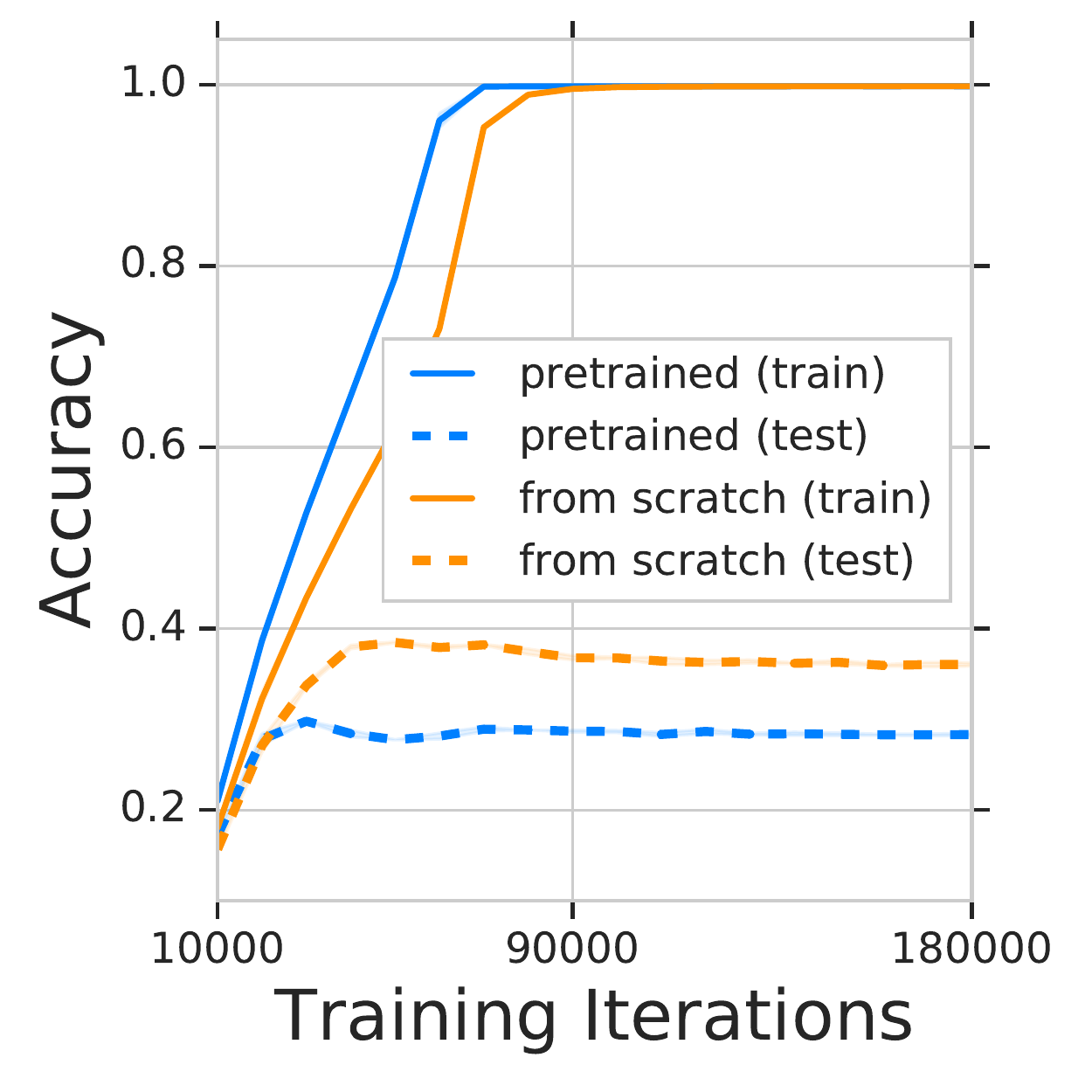}%
  \includegraphics[width=\pltwidth]{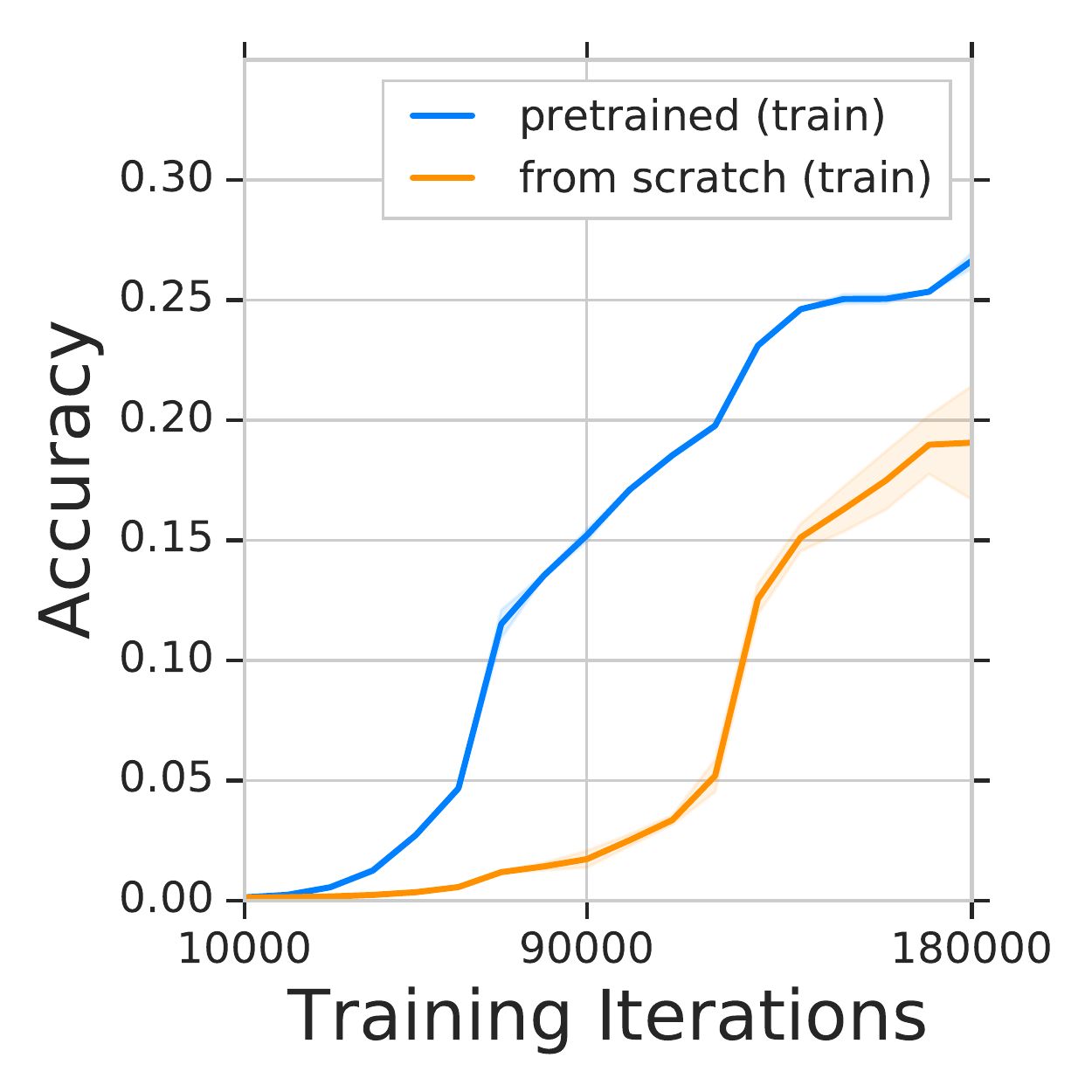}
  \caption{Pre-training on random labels often accelerates the downstream training compared to training from scratch with same hyperparameters.
  ResNet18 models are pre-trained on 500k
  ImageNet examples with 1000 random labels and subsequently fine-tuned on the fresh 500k ImageNet examples with either real labels or 1000 random labels using learning rate 0.01 and different initialization scales.
  Error bars correspond to $\pm 1$\:std.\:over 2 independent runs. 
  }
  \label{fig:imagenet-resnet-1}
\end{figure}

\begin{figure}[tbp]
  \centering \sffamily
  \begin{minipage}{\pltwidth}
    \centering 
    Pre-training slows down training\\
    (real labels)
  \end{minipage}
  \begin{minipage}{\pltwidth}
    \centering
    Pre-training helps\\
    (random labels)
  \end{minipage}\\
  \includegraphics[width=\pltwidth]{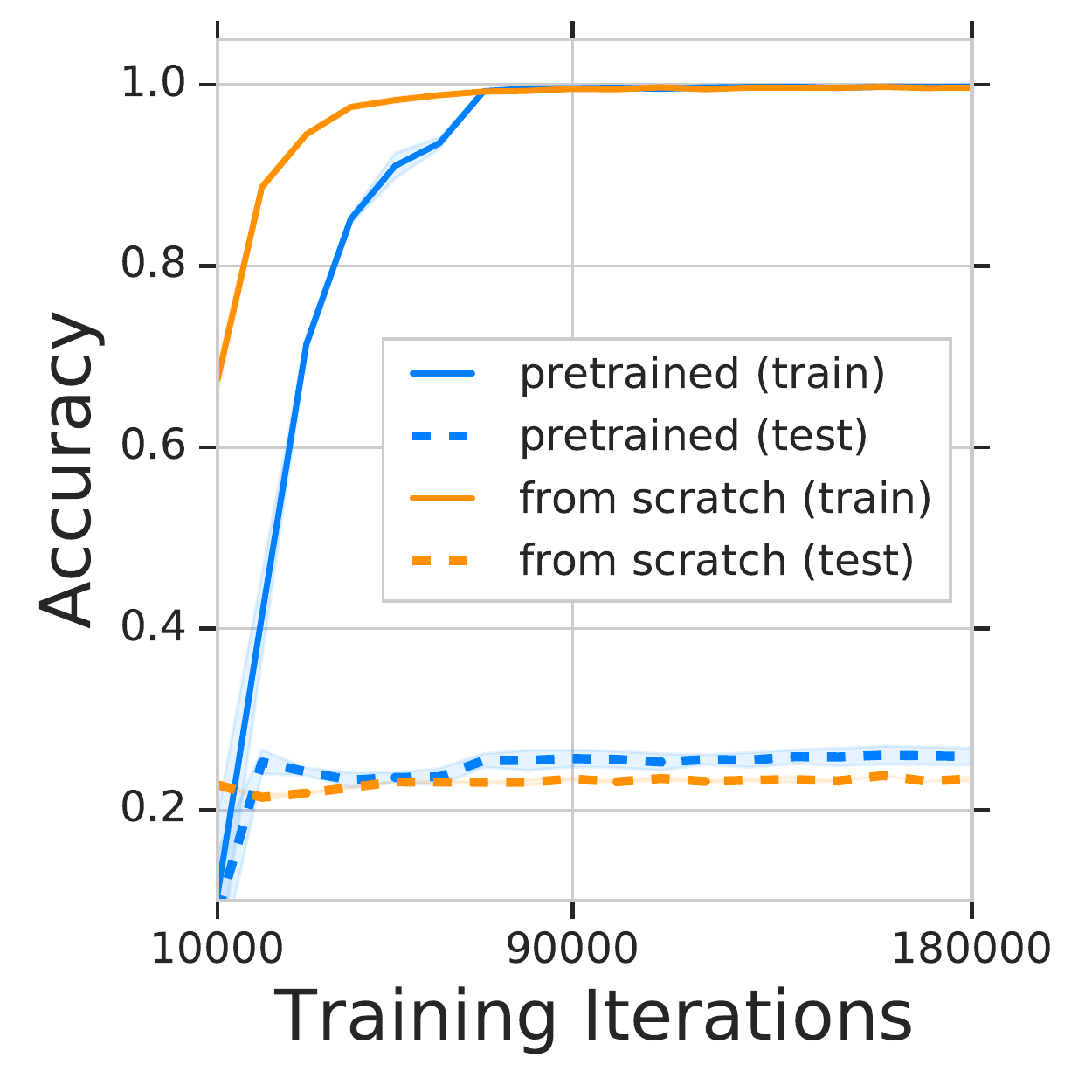}%
  \includegraphics[width=\pltwidth]{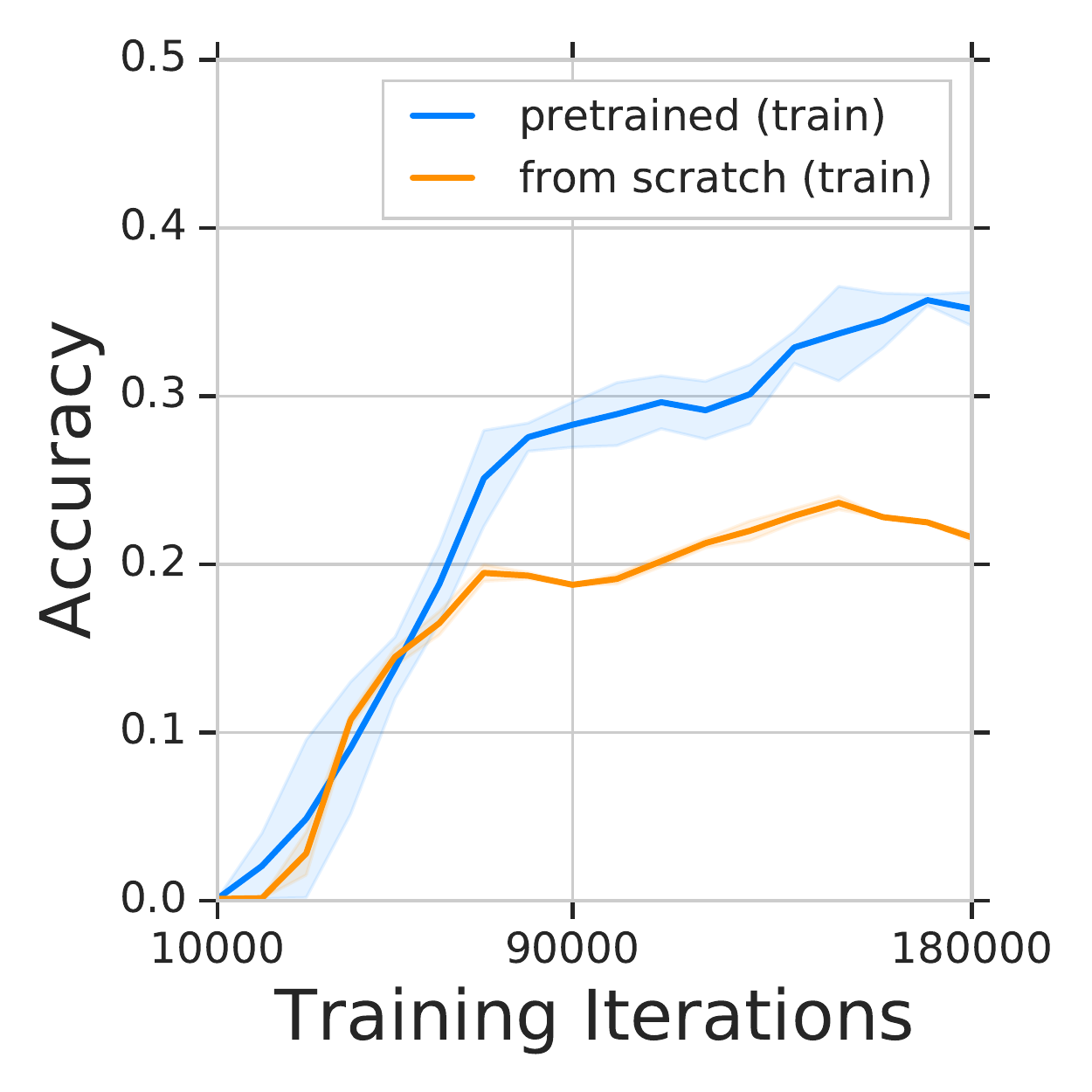}
  \caption{Pre-training on random labels may both accelerate or slow down the downstream training compared to training from scratch with same hyperparameters.
  VGG16 models are pre-trained on 500k
  ImageNet examples with 1000 random labels and subsequently fine-tuned on the fresh 500k ImageNet examples with either real labels or 1000 random labels.
  {\sc Left:} Surprisingly, pre-training improves the holdout test accuracy.
  Error bars correspond to $\pm 1$\:std.\:based on 2 independent runs.
  }
  \label{fig:imagenet-vgg-1}
\end{figure}

\begin{figure}[tbp]
  \centering \sffamily
  \begin{minipage}{\pltwidth}
    \centering 
    Pre-training hurts\\
    (real labels)
  \end{minipage}
  \begin{minipage}{\pltwidth}
    \centering
    Pre-training hurts\\
    (random labels)
  \end{minipage}\\
  \includegraphics[width=\pltwidth]{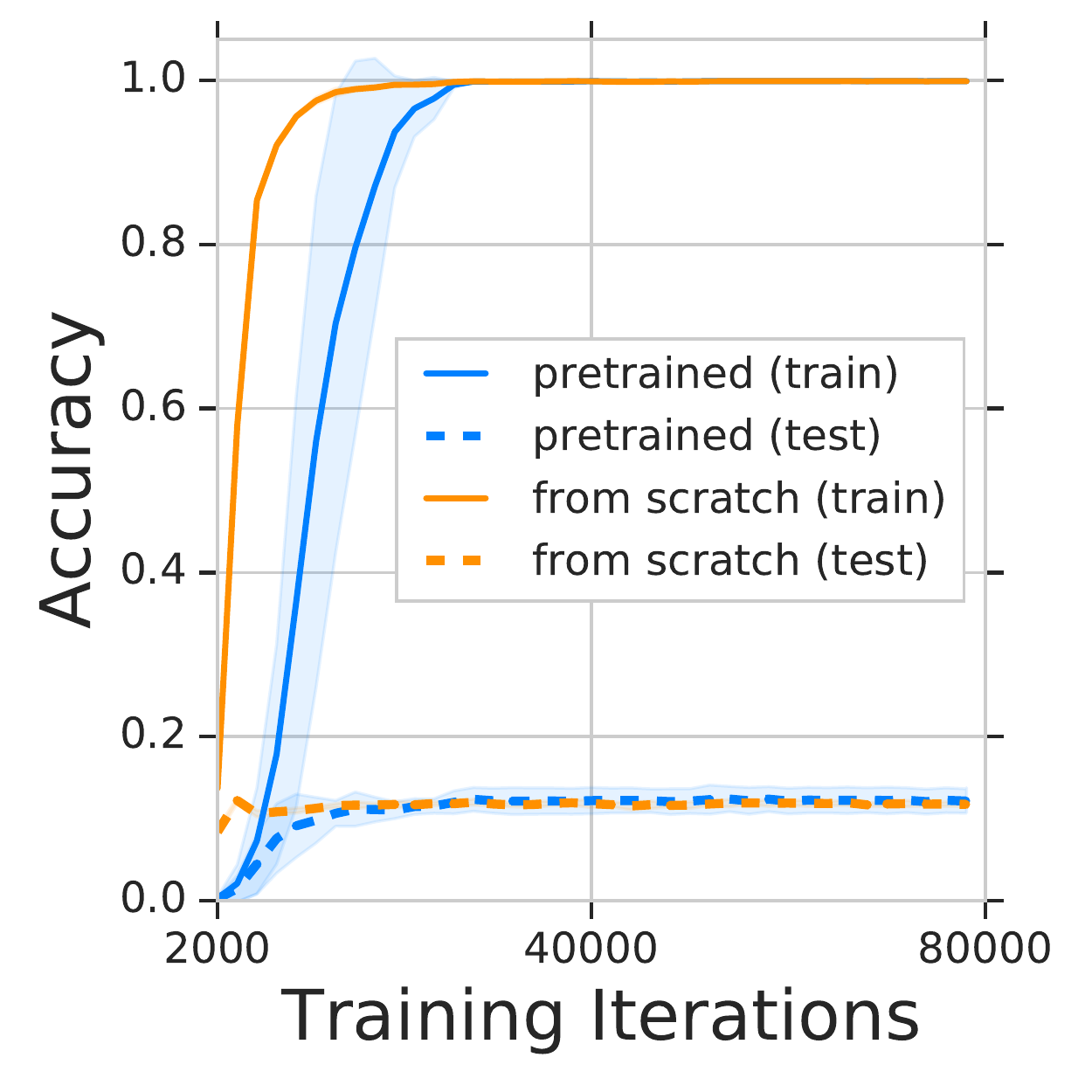}%
  \includegraphics[width=\pltwidth]{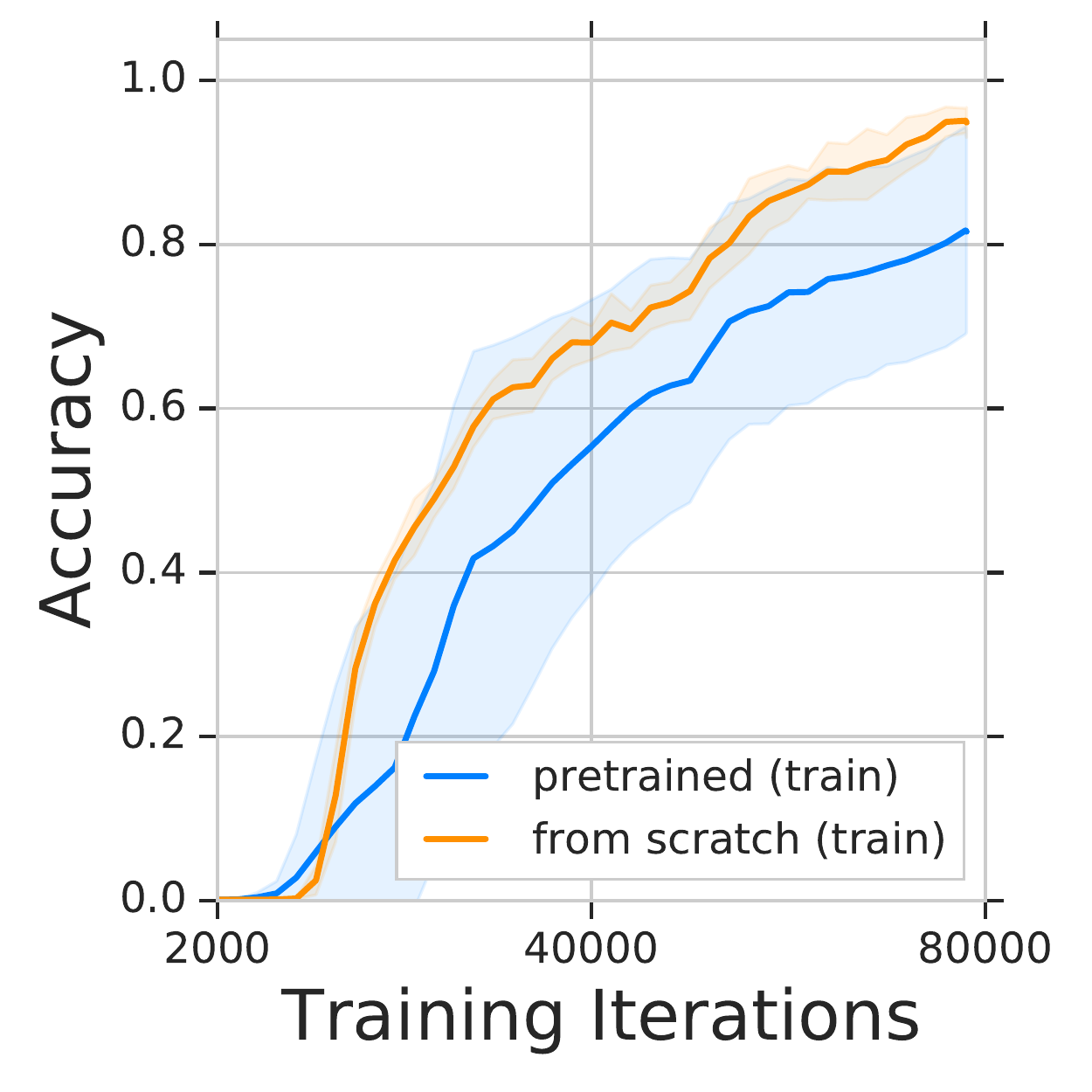}
  \caption{
  VGG16 models are pre-trained on 50k
  ImageNet examples with 1000 random labels and subsequently fine-tuned on 200k fresh ImageNet examples with real labels or 1000 random labels.
  }
  \label{fig:imagenet-vgg-2}
\end{figure}

Figure \ref{fig:imagenet-vgg-1} reports experiments with VGG16 on ImageNet where pre-training both helps and hurts.
It may speed up the training, or make it slower.
Surprisingly, in some cases it may slightly improve the holdout test accuracy.
These experiments use:
\begin{verbatim}
init_scale = 1.0
learning_rate = 0.01
num_classes_upstream = 1000
num_examples_upstream = 500000
num_epochs_upstream = 100
num_examples_downstream = 500000
num_epochs_downstream = 100
\end{verbatim}

Figure \ref{fig:imagenet-vgg-2} reports experiments with VGG16 on ImageNet where pre-training hurts both real and random label downstream tasks.
These experiment are taken from the ImageNet Experiments A reported in Section \ref{section:imagenet-experiments-a}.
These experiments use:
\begin{verbatim}
init_scale = 0.747
learning_rate = 0.008
num_classes_upstream = 1000
num_examples_upstream = 50000
num_epochs_upstream = 100
num_examples_downstream = 200000
num_epochs_downstream = 100
\end{verbatim}

\clearpage

\section{Proof of Proposition 1}
\label{app:Theorem1}

Proposition 1 makes the following assumptions:
\begin{enumerate}
    \item The first layer is either
    \begin{itemize}
        \item fully connected, and the input from $\BR^d$ is normally distributed with mean $ \mu_x=0$ and covariance $\Sigma_x$, or
        \item convolutional, with patches that are not overlapping and the data in
        each position is independent normally distributed with mean 
        $ \mu_x=0$ and covariance $\Sigma_x$.
    \end{itemize}
    \item The first layer weights $ w\in\BR^d$ are initialized i.i.d.\:at random from
        some distribution on $\BR^d$ that is invariant under the orthogonal group $O(d)$. 
    \item The sampled inputs are labelled randomly according to some distribution over the target set $\mathcal{Y}=\{1,2,\ldots,c\}$, independently of the input sample.
\end{enumerate}

Let $\calG := \{ G \in O(d)\, | \, G^T \Sigma_x G = \Sigma_x \}$ be the group of
orthogonal matrices that leave the distribution of the data $ x_i$ invariant.
Proposition 1 follows from these two claims:

\begin{claim}\label{claim:1}
A probability distribution $\calD_w$ on $\BR^d$ 
with mean $\mu_w$ and covariance matrix $\Sigma_w$ that is 
invariant under $\calG$ has $\mu_w = 0$ and $\Sigma_w$ aligned with $\Sigma_x$.
\end{claim}

\begin{claim} \label{claim:2}
The probability distribution of the weights $ w$ in the first
layer after $t$ iterations of training is invariant under $\calG$ (for any $t=0,1,...$).
\end{claim}

\begin{proof}[Proof of Claim \ref{claim:1}]
 Since $-I \in \calG$, we have $\BE[ w] = \BE[- w]$, so we must have 
 $\BE[ w]=0$. For the second part, assume
 \[
   \BR^d = V_1 \oplus V_2 \oplus ... \oplus V_r
 \]
 is the orthogonal decomposition of $\BR^d$ into eigenspaces of $\Sigma_x$ and
 $d_i := \dim(V_i)$ are the dimensions of its parts. Then 
 \[
    \calG = O(d_1) \times O(d_2) \times ... \times O(d_r)
 \]
 where each $O(d_i)$ operates in the canonical way on $V_i$ and leaves the other
 parts invariant.
 
 By definition of $\calG$ we have $G^T \Sigma_w G = \Sigma_w$ for all $G\in \calG$.
 Then each $G\in\calG$  must also leave eigenspaces of $\Sigma_w$ invariant since
 for an eigenvector $u$ with $\Sigma_w u=\lambda u$, we have
 \[
    \Sigma_w (G u) = (\Sigma_w G) u = (G \Sigma_w) u = \lambda \cdot G u.
 \]
 We have to show that each $V_i$ is contained in an eigenspace of $\Sigma_w$.
 Assume to the contrary that a particular $V_i$ is not contained in any eigenspace
 of $\Sigma_w$.
 Since the eigenvectors of $\Sigma_w$ span $\BR^d$, there must be 
 an eigenvector $ u$ in an eigenspace $U$ of $\Sigma_w$ that is not orthogonal
 to $V_i$. We will show that $V_i \subseteq U$: 
 Let $M_i \in \calG$ be the matrix that is $I$ on $V_i$ and $-I$ on all other 
 $V_j$, $j\neq i$, and let $pr_i:\BR^d\arrow V_i$ be the orthogonal projection 
 onto $V_i$. Then by  assumption $pr_i(u) \neq 0$, but since 
 $pr_i(u) = (u + M_i u)/2$, we must also have $pr_i(u)\in U$. Since $O(d_i)$
 operates transitively on the set of all vectors $v\in V_i$ of a given length,
 and $\{I\}\times ... \times O(d_i) \times ... \times \{I\} \subseteq \calG$, all vectors
 of length $|pr_i(u)|$ in $V_i$ must as well be in $U$. Since $U$ is closed under
 scalar multiplication and $pr_i(u) \neq 0$, this means that indeed $V_i\subseteq U$.
\end{proof}

\begin{proof}[Proof of Claim \ref{claim:2}]
Each run of the given network is determined by sampled initial data: 
\begin{itemize}
    \item The initial weights $ w_1,..., w_M\in W = \BR^d$ in the first layer,
    \item the inputs $ x_1,..., x_N\in V = \BR^d$,
    \item the initial weights in the later layers, and biases in all layers,\\
          we will write them as one large vector $w' \in W' = \BR^{d'}$,
    \item the targets $y_1,...,y_N\in\calY$.
\end{itemize}
We will consider the targets as fixed and write the other initial data as
\[
   ( w_i, x_j, w')\in \calV := W^M \times V^N\times W'.
\]
The group $\calG\subseteq O(d)$ operates on $V=\BR^d$ and $W=\BR^d$, we 
also define its operation on $\calV$ by
\begin{equation}
  G: \calV \arrow \calV, 
  ( w_i,  x_j, w') \mapsto (G  w_i, G  x_j, w')
  \label{eq:GradientDescent}
\end{equation}
By our assumptions and the definition of $\calG$, this operation 
leaves the distribution of initial data invariant for any $G\in \calG$.
This proves the ``$t=0$'' case in Claim \ref{claim:2}, it remains to show that this invariance
is kept when we do one step of Gradient Descent, this will follow if we can show
that the update commutes with the operation of $\calG$.

At each step of the training the current state is given by a point 
in the  vector space $\calV$, the loss is a function
$   L:\calV \arrow \BR$
and one step of Gradient Descent is given by
\begin{equation*}
  w_i \mapsto w_i - \eps \cdot \nabla_{w_i} L(w_i, x_j,  w'), \qquad  w' \mapsto  w''
\end{equation*}
for some $w''\in W'$.
We have to show that this update commutes with the operation of $\calG$, 
i.e. for all $G\in\calG$
\begin{equation}
   G w_i \mapsto G w_i - \eps \cdot G \nabla_{w_i} L (w_i, x_j,  w'), \qquad
    w' \mapsto  w''
  \label{eq:updateEquivariant}
\end{equation}
for the same $w''$.

Since the first layer only gets information about the $w_i$ and $x_j$ via 
their scalar products $\langle w_i, x_j \rangle$, and for any $G\in O(d)$
we have $\langle G w_i, G x_j \rangle = \langle w_i, x_j \rangle$, the operation 
\eqref{eq:GradientDescent} of $G\in\calG$ on
$\calV$ leaves the loss function $L:\calV\arrow \BR$ invariant. For the same
reason also the updates of the biases and later
weights $w'$ are unaffected by the operation of $\calG$.
So we only have to prove the update equations for $G w_i$, i.e. show that
\[
   \nabla_{w_i} L (G w_i, G x_j,  w') = G \nabla_{w_i} L(w_i, x_j,  w')
\]
The gradients $\nabla_{w_i} L$ in the $w_i$--directions are part of the full 
gradient $\nabla L$ (which is a vector in $\calV$, so it also contains 
the derivatives with respect
to the points $x_j$ and the weights $w'$ of the later layers).
Hence it is sufficient (or even stronger) if we can show
\begin{equation}
   \nabla L(G p) = G \nabla L(p)\qquad \hbox{for each point}\ p\in\calV
   \ \hbox{and}\ G\in\calG.
   \label{eq:GNabla}
\end{equation}

Since $L$ is invariant under $G$, also the differential form $dL$ is
invariant: $G^* dL = dL$.

The standard Euclidean metric on $\calV=\BR^{d\times M + d \times N + d''}$ 
provides a translation between differential forms and vector fields that 
determines $\nabla L$ from $dL$ and vice versa by the condition
\[
   \langle \nabla L (p), v \rangle = dL|_p(v) \qquad \hbox{for all }\ v\in\calV.
\]
Since the Euclidean metric on $\calV$ is invariant under $G\in\calG$, we have for each
vector $v\in\calV$
\[
   \langle \nabla L (G p) , G v \rangle
   = dL|_{G p} (G v)
   = dL|_p (v)
   = \langle \nabla L (p), v \rangle
   = \langle G \nabla L (p), G v \rangle
\]
Since $G\in \calG$ is invertible, $Gv$ can be any vector in $\calV$, so 
\eqref{eq:GNabla} follows.
\end{proof}

Note that the same proof argument holds for many other optimization techniques, not only for gradient descent. For example, it does not matter whether we use Momentum, AdaGrad, Adam, Nesterov, weight decay, etc. in the gradient descent.
However, the optimization needs to respect symmetries from $O(d)$, i.e. it cannot make 
use of the special coordinate system for the input. This excludes for example Exponentiated Gradient \cite{kivinen1997exponentiated}.

Similarly, the proof is also independent of the loss function used, if that
loss only involves the output of the network: Since the proof only makes use of 
symmetries $\langle G x, G  w \rangle = \langle  x,  w \rangle$ in
the first layer, it
is not affected by what we do in the later layers or at the output level.
However, if the loss function includes regularization terms that involve the 
weights in the first layer, these terms also need to be invariant under the 
orthogonal group. This means the proof still applies if we use $L2$ regularization, but $L1$ regularization does make use of the special coordinate system and is not
invariant under rotations, so the proof arguments would not hold anymore if we 
used it in the first layer.

\clearpage
\section{Measuring Alignment}

\label{sec:MeasuringAlignment}
\subsection{How well can we measure eigenvectors?}
We do not observe the covariance matrices directly, but estimate them from samples 
of the corresponding distributions. This estimate has a variance that results in a 
variance of the computed eigenvectors. As a first intuitive experiment we compare
the eigenvectors obtained from two disjoint samples:

{\bf Data}
    We compare the covariance matrices $\Sigma_x$ and $\Sigma_x'$ obtained from $5\times 5$ patches of the first and second 30\,000 images
    in CIFAR10. The first eigenvectors are extremely well aligned:

\vskip0.8ex    
\centerline{\small
\begin{tabular}{@{}rcc@{}}
\hline 
  i & $\sigma_i^2$ & $|\langle e_i, e'_i \rangle|$ \\
\hline
1 & 12.3\phantom{55} & 0.9999996 \\
2 & 1.45 & 0.99995 \\
3 & 1.18 & 0.99992 \\
4 & 1.06 & 0.99994 \\
5 & 0.37 & 0.99969 \\
~6 & 0.32 & 0.99970 \\
\hline
\end{tabular}\hspace{1em}
\begin{tabular}{@{}rcc@{}}
\hline 
  i & $\sigma_i^2$ & $|\langle e_i, e'_i \rangle|$ \\
\hline
7 & 0.28\phantom{5} & 0.99985 \\
8 & 0.23\phantom{5} & 0.99988 \\
9 & 0.128 & 0.99010 \\
10 & 0.126 & 0.99030 \\
11 & 0.115 & 0.99969 \\
12 & 0.096 & 0.99999 \\
\hline
\end{tabular}
}

Also the other eigenvectors are very well aligned, all scalar products obtained were 
above 0.975. So 30,000 images are sufficient to estimate the eigenvectors of $\Sigma_x$ to a high accuracy.

{\bf Weights}
We use the same ResNet experiment as for Figure~\ref{fig:filters}, described in Section~\ref{sec:fig3settings}. We use two disjoint groups 
of 30 randomly initialized networks, trained in the same way on CIFAR10 with random labels.

\centerline{
\small
\begin{tabular}{@{}rcc@{}}
\hline
  i & $\sigma_i^2$ & $|\langle e_i, e'_i \rangle|$ \\
\hline
1 & 0.281 & 0.974 \\
2 & 0.202 & 0.958 \\
3 & 0.154 & 0.916 \\
4 & 0.141 & 0.934 \\
5 & 0.107 & 0.861 \\
~6 & 0.104 & 0.861 \\
\hline
\end{tabular}\hspace{1em}
\begin{tabular}{@{}rcc@{}}
\hline
  i & $\sigma_i^2$ & $|\langle e_i, e'_i \rangle|$ \\
\hline
7 & 0.085 & 0.906 \\
8 & 0.078 & 0.946 \\
9 & 0.077 & 0.924 \\
10 & 0.072 & 0.811 \\
11 & 0.063 & 0.598 \\
12 & 0.059 & 0.645 \\
\hline
\end{tabular}
}

So we can only measure the most important eigenvectors with reasonable accuracy.
Of course it is not unexpected that we can determine the covariance of image patches much better 
than the covariance of weights: In the above example, we used 
$30000 \times 30 \times 30$ = 27 million image patches, but only $30\times 64 = 1920$ weight 
vectors (filters). However, apart from the number of examples, also the type of information
we want to extract from the covariance matrix determines the accuracy of our measurement.

In particular, when two eigenvalues are close together (like $\sigma_5, \sigma_6$ above), it may be
difficult to determine the exact eigenvector. However, the 2-dim space spanned by both eigenvectors
is relatively stable -- in the above example the expansion of $e_5, e_6$ in terms of the basis
$e'_i$ is
\begin{eqnarray*}
   e_5 &=& 0.86 e'_5 - 0.36 e'_6 + ...(\mathrm{smaller\ terms})...\\
   e_6 &=& 0.36 e'_5 + 0.86 e'_6 + ...(\mathrm{smaller\ terms})...
\end{eqnarray*}

A similar observation can be made in figure \ref{fig:filters}: 
The first eigenvalues of $\Sigma_w$ 
%
are 0.019, 0.017, 0.011, 0.010, 0.009,..., with the first two close together; the corresponding
two eigenvectors of $\Sigma_w$ given on the right of figure \ref{fig:filters} were:

\hskip 1cm
\includegraphics[width=0.05\textwidth]{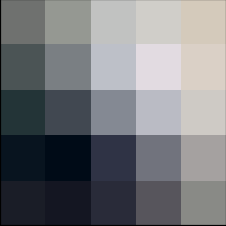} 
\raisebox{7pt}{\ $\approx$\ }
\includegraphics[width=0.05\textwidth]{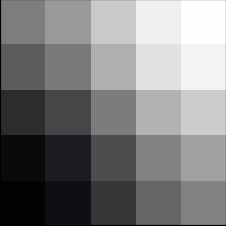}
\raisebox{7pt}{$= \ \ 0.65\  \times$} \includegraphics[width=0.05\textwidth]{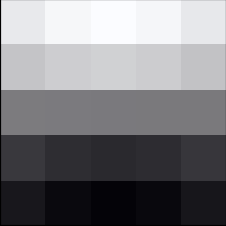}
\raisebox{7pt}{$ - \ 0.71\ \times\ $}  
\includegraphics[width=0.05\textwidth]{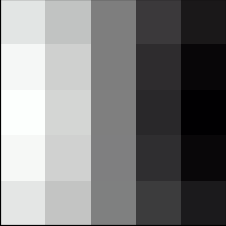}\ \hskip 1cm and

\hskip 1cm
\includegraphics[width=0.05\textwidth]{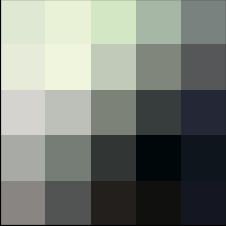} 
\raisebox{7pt}{\ $\approx$\ }
\includegraphics[width=0.05\textwidth]{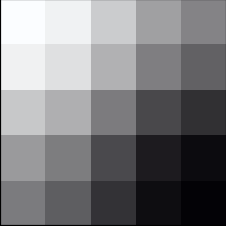}
\raisebox{7pt}{$ = \ \ 0.71 \times$} \ \includegraphics[width=0.05\textwidth]{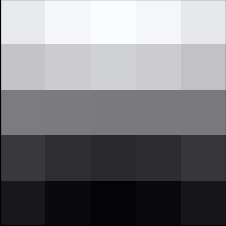}
\raisebox{7pt}{$ + \ 0.67 \times\ $} \ 
\includegraphics[width=0.05\textwidth]{images/pca_new/pca-1-3.png}

So the vector space of these two eigenvectors of $\Sigma_w$ is well aligned to the vector space 
of two eigenvectors of $\Sigma_x$, but its given basis is rotated compared to the eigenvectors 
of $\Sigma_x$.

This is a general problem - if two eigenvalues are close to each other, the
direction of eigenvectors is not measurable in practice, or if they are even
the same, also theoretically individual eigenvectors are not defined, 
only the higher dimensional eigenspace is. 
We will investigate this problem of statistical sampling uncertainty 
closer in \ref {subsec:FisherRiemann}.

Apart from the statistical sampling uncertainty we may also consider noise
in the images themselves. However, given that the two computations of the
eigenvectors agreed so well for the data covariance, we can expect that 
this error is small compared to the sampling error for the weight covariance.

\subsection{Motivational definition of misalignment}

Any attempt of defining ``misalignment measures'' naively by using the 
direction of eigenvectors (e.g.\ ``misalignment = sum of angles
between corresponding eigenvectors'') runs into the problem mentioned above:
It would be undefined when we have eigenspaces of dimension $>1$ or at least
discontinuous / not measurable / not informative in practice when two of the
eigenvalues are close. So instead we will construct a ``distance to an aligned
matrix'' that takes into account what can actually be measured.

In \cite{Helmholtz1896} a Riemannian metric is defined on the space of color
perceptions. In this metric, the length of a path in color space is given by the minimal
number of colors on this path connecting start and end such that each
color is indistinguishable (by a human) from the next one.

We can use a similar construction for covariance matrices: Fix a (large) $n$, then 
we call two matrices indistinguishable, if after sampling $n$ samples
from $\calN(0,\Sigma_k)$ we most of the time cannot reject the hypothesis that 
the correct covariance matrix was $\Sigma_{k+1}$.
The minimal number of ``indistinguishable'' matrices connecting start to end of a 
curve is (in good approximation) proportional
to $\sqrt n$, to get a path length independent of $n$, we divide the number 
by $\sqrt n$ and take the limit $n\arrow \infty$.
The misalignment could now be defined as the minimal path length of a path between $B$ 
and a matrix $\Sigma$ that is aligned to $A$. (The exact definition
would need to specify ``indistinguishable'', i.e. the fraction of times we can reject 
the hypothesis and the confidence level used in this rejection. The path lengths corresponding
to different definitions would differ by a multiplicative constant.)

The Riemannian metric defined in this way is given locally by the Fisher Information, which
can be easily computed in this case. However, it seems there is not a simple known formula
for the resulting (global) distance between two points. 
To simplify computations and proofs, we will use instead the upper bound to the square length given by
the symmetrized Kullback--Leibler divergence. Like the square length of the shortest path between
$A$ and $B$ it can be expressed as an integral over the Fisher Information, but the path is not
the usual shortest path (geodesic) that is traversed with constant speed, 
but the straight line $(1-t)\cdot A + t \cdot B$. For small distances this is a good approximation,
but it will give a larger value in general.

\subsection{Fisher information and Riemannian manifold of symmetric positive definite matrices}
\label{subsec:FisherRiemann}

As a concrete toy example for the problem of close eigenvalues observed above, 
we take the 2-dimensional normal distribution with $\mu=0$ and
$\Sigma$ the diagonal matrix with entries $(1,\lambda)$.
We generate $n=100$ i.i.d.\ samples $\vec w_1,...,\vec w_n$ of this distribution, measure
their empirical covariance $\hat\Sigma_w$, and plot the direction of their eigenspaces.
Repeating this 50 times, we can visualize the distribution of these directions (Figure~\ref{fig:close-ev2}, first row).

\begin{figure}[tbp]
    \centering
\includegraphics[width=0.85\textwidth]{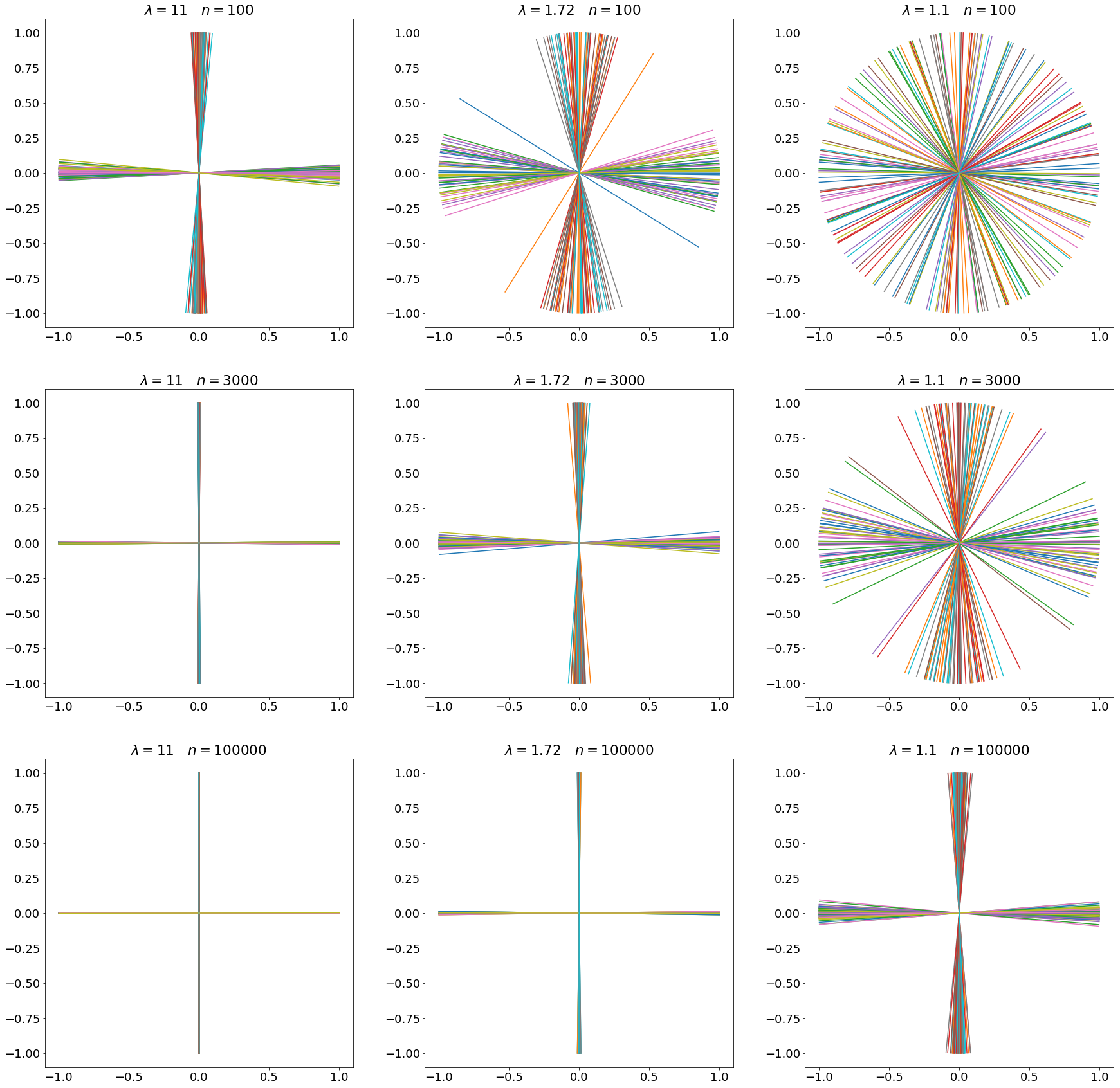}
    \caption{Visualization of the variation of estimating the directions of the eigenspaces for 2-dimensional normal distribution with $\mu=0$ and
$\Sigma$ the diagonal matrix with entries $(1,\lambda)$ for three different values of $\lambda$ and varying number of samples $n$.}
    \label{fig:close-ev2}
\end{figure}

While we get a reasonable approximation to the direction of the eigenspaces of $\Sigma$
(i.e. the coordinate axes) for large $\lambda$, the eigenvector directions of $\hat \Sigma$ 
become less well-aligned as $\lambda = \sigma_2^2$ gets close to $1= \sigma_1^2$.
So if we were to measure ``whether the diagonal matrix $diag(1,\lambda)$ is aligned with itself''
by computing two empirical covariance matrices (from 2 independent sets of sampled data), 
their eigenvectors, and angles between the two 
empirical results, we would think that these two empirical covariances are not aligned
for $\lambda\approx 1$, although they actually come from the same distribution.

Mathematically, the problem is that any ``angles between eigenvectors'' measure would be
undefined for matrices with an eigenspace of dimension $> 1$. We could fix that e.g.\ by 
using the choice of eigenvectors that gives the smallest possible result, but then 
the resulting function would not be continuous around such matrices, i.e.\ we may
need estimates with unlimited precision to get the alignment measure with a fixed precision.

On the other hand, for any fixed $\sigma_1^2 \neq \sigma_2^2$ we can in theory make $n$ large enough 
to get the sampling error as small as we want, e.g. in the above example we can go from $n=100$ to
$n=100\,000$ (Figure~\ref{fig:close-ev2}).

We can say that for $\lambda_1 \approx \lambda_2$ we have less information per sample
from the normal distribution $\calN(0,\Sigma)$ about the rotation 
angle $\alpha$ than for $\lambda_1  \gg \lambda_2$. 
This is formalized by the Fisher information, we now summarize some of its
properties for the case we are interested in.

We denote by $M$ the manifold of symmetric positive definite $d\times d$ matrices,
it is an open cone in the $D := \frac{d(d+1)}{2}$ dimensional vector space
of symmetric matrices.
We identify the points $\Sigma\in M$ also with their corresponding
normal probability distributions $\calN(0,\Sigma)$. Given $n$ samples
$x_1,...,x_n\in \BR^d$, the empirical covariance matrix 
$\frac 1n \sum_{i=1}^n x\cdot x^T \in M$ is the maximum likelihood estimator.
Given some local coordinates $\theta_1,...,\theta_D$ around a matrix $\Sigma$,
the Fisher information of $\theta_k$ at $\Sigma$ is defined as
\begin{equation}
  I(\theta_k) := \BE_{x\sim \calN(0,\Sigma) }\left[ \left(  
    \frac{\partial \log p(x| \theta_k)}{\partial \theta_k}
  \right)^2 \right]
  \label{eq:FisherInformation}
\end{equation}
We are interested in the variance of the maximum likelihood estimate 
$\hat \theta_k$ when we estimate the parameter $\theta_k$ from $n$ samples
$x_1,..,x_n$, it can be approximated as
\[
   Var(\hat \theta_k) \approx \frac{1}{n \cdot I(\theta_k)},
\]
and this approximation becomes exact as $n\arrow \infty$ (see e.g. \cite{wasserman2004all}, chapter 9.7. for a precise statement).
In our toy example we are interested in the rotation angle $\alpha$, and
local coordinates would be $\alpha, \lambda_1, \lambda_2$
(we assume for the moment that $\lambda_1\neq \lambda_2$ at $\Sigma$). In this
case one can compute the Fisher information about $\alpha$ (see e.g. \cite{costa2015fisher}) as
\[
   I(\alpha) = \frac{(\lambda_1 - \lambda_2)^2}{\lambda_1\cdot \lambda_2}.
\]

In the Figure~\ref{fig:close-ev2}, the images on the diagonal look similar, and indeed we have chosen
$\lambda$ such that the Fisher Information about $\alpha$ of $n$ samples 
is approximately the same: In each case 
$n\cdot (\lambda-1)^2/(\lambda\cdot 1)$ is between 904.1 and 909.1,
which gives a standard deviation of about 1/30 
(corresponding to about 2 degrees) for the rotation angle $\alpha$, and 
indeed generating 100 directions from a normal distribution of $\alpha$ with
a standard deviation of 1/30 matches the above plots on the diagonal.

This suggests that 
\[
   \delta = \alpha \cdot \sqrt{I(\alpha)}
\]
is an appropriate distance measure between matrices that differ by a rotation 
of angle\ $\alpha$: When two matrices $A, B$ have ``distance'' $\delta$, this 
means that for large $n$ we can put about 
$k = \lfloor \delta \cdot \sqrt n\rfloor$ points 
$\Sigma_1, \Sigma_2,..., \Sigma_k$ between $A=\Sigma_0$ and $B=\Sigma_{k+1}$
such that each pair $\Sigma_i$ and $\Sigma_{i+1}$ are within measurement error 
of each other when our measurement consists of estimating the covariance from a sample of $n$ points.

So far, we have defined the Fisher information in a particular coordinate system
on the manifold $M$. However, the definition \eqref{eq:FisherInformation}
depends only on the tangent vector $\partial/\partial \theta_k$. 
So the Fisher information is actually assigned to tangent vectors of the manifold,
is independent of the coordinate system, and it turns out 
to satisfy the properties of a Riemannian metric, so it gives for any
tangent vector $v \in T_\Sigma(M)$ a ``length'' $\sqrt{I(v)}$.
This is a special case of the general theory of Information Geometry, see e.g.
\cite{amari2016information,ay2017information}.

\subsection{Defining misalignment measures}

While the Fisher information gives a local distance measure (i.e.\ a length
of tangent vectors), we rather need a global
distance measure $D(\Sigma_1,\Sigma_2)$ between points. For that we integrate the 
Fisher Information along the straight line 
\begin{equation}
   \gamma: t \mapsto (1-t)\cdot \Sigma_1 + t \cdot \Sigma_2 \qquad \hbox{for} \ t\in [0,1].
    \label{eq:straightLine}
\end{equation}
If instead of \eqref{eq:straightLine} we had used the real shortest path (geodesic) of the Riemannian metric, 
which goes in constant speed from $t=0$ to $t=1$, this would give the square length of the shortest
path. Since in general
\eqref{eq:straightLine} is not the shortest path, and the parameterization does not have 
constant derivative with respect to the standard connection, this can give slightly larger results.
Of course, for small distances it still gives the same result in first order.

In terms of information geometry, the straight line is the $e$--geodesic, and the integration
of the Fisher information against the $e$--geodesic gives the same result as integrating 
along the $m$--geodesic:
The symmetrized Kullback--Leibler
divergence between $p$ and $q$ (see e.g. Theorem 3.2. of \cite{amari2016information}, also compare
section 4.4.2. in \cite{ay2017information}).
The symmetrized Kullback--Leibler divergence exists in two normalizations: With or without the 
factor $1/2$. We are using the version with $1/2$
\[
   D(\Sigma_1, \Sigma_2) :=  \frac{ 
       D_{KL} \Big(\calN(0,\Sigma_1) \ ||\  \calN(0,\Sigma_2)\Big) \ 
                          + \ D_{KL} \Big(\calN(0,\Sigma_2) \ ||\ \calN(0,\Sigma_1)\Big)
                          }{2}
\]
which is equal to half of the integral of the Fisher information:
\[ 
    D(\Sigma_1, \Sigma_2) = \frac{1}{2}\int_0^1 I\left( \frac {\partial \gamma}{\partial t}(t) \right) dt
\]

So we will use this distance measure as a basis for our (mis)alignment measure.
It can be expressed analytically as
\[
    D(\Sigma_1, \Sigma_2) = \frac{\textbf{tr}(\Sigma_1^{-1}\Sigma_2 + \Sigma_2^{-1}\Sigma_1)}{2} - d
\]

While it would be possible to use different distance measures that also have
the Fisher information as the infinitesimal version (e.g. the usual, asymmetric
Kullback--Leibler divergence), this definition has the additional benefit that
it is invariant under scaling: $D(\lambda\Sigma_x, \lambda\Sigma_w) 
= D(\Sigma_x, \Sigma_w)$, 
and this allows the simple definition of the alignment measure given in 
Section~\ref{sec:PCA} (for more general distance measures one would need
to restrict the $\Sigma$ to a subset e.g.\ by requiring $\textbf{tr}(\Sigma)=1$ and
write $\inf_{\Sigma, \lambda>0} D(\Sigma, \lambda\cdot \Sigma_w)$ to get
a result $>0$ when the matrices are not aligned).

So we define the ``misalignment'' score for two positive definite symmetric matrices $A, B$ as

\begin{equation*}
   M(A, B) := \inf_{\Sigma\succ 0\ \hbox{ aligned with}\ A} 
     D(\Sigma, B)
   = \inf_{\Sigma\succ 0\ \hbox{ aligned with}\ A} 
   \left\{\frac{\textbf{tr}(\Sigma^{-1}B + B^{-1}\Sigma)}{2} - d\right\},
\end{equation*}
which was our definition in Section~\ref{subsec:Gaussian}.

\begin{proposition}\label{prop:AlignmentProperties}
This misalignment measure has the following properties:\\
For all positive definite symmetric matrices $A, B$
\begin{enumerate}

    \item $M(A, B) \geq 0$
    \item $M(A, B) = 0$ if and only if $B$ is aligned with $A$.
    \item $M(A, B)$ is continuous in $B$.
    \item Equivariance under orthogonal group:
          $M(U A U^T, U B U^T) = M(A, B)$ for $U \in O(d)$
    \item Invariance under scalar multiples of $B$:
          $M(A, \lambda B) = M(A, B)$ for 
          $\lambda > 0 $
     \item $M(A, B)$ only depends on the eigenspaces of $A$.
     \item $M(A,B) +d = \sum_{i=1}^r \sqrt{ \textbf{tr}(B|_{V_i}) \cdot \textbf{tr}(B^{-1}|_{V_i})}$,
   where $V_1 \oplus ... \oplus V_r$ is the orthogonal decomposition of $\BR^d$ into eigenspaces of $A$,
   and $B|_{V_i}$ is the linear map $V_i\arrow V_i, \textbf{v}\mapsto pr_i(B(\textbf{v}))$ with $pr_i$ the 
   orthogonal projection $\BR^d\arrow V_i$. 
\end{enumerate}
The function $M(A, B)$ is not continuous in $A$, and there
cannot be a function $M(A, B)$ that is continuous in both arguments
and still satisfies condition 2.
\end{proposition}

\begin{proof}
1, 4, 5, 6, and the ``if'' part of 2 follow directly from the definitions.\\
3 follows from 7. We will see in the proof of 7 that the infimum is obtained for a
particular matrix $\Sigma$, and from that also the ``only if'' part of 2 follows.\\
Formula 7: We use the orthogonal decomposition $V_1 \oplus ... \oplus V_r$ of $\BR^d$ into eigenspaces of $A$.
A positive definite symmetric matrix $\Sigma$ aligned with $A$ is given by its eigenvalues $\lambda_i$ on the
subspaces $V_i$. For a linear map $f:\BR^d\arrow\BR^d$ we have
\[
   \textbf{tr}(f) = \sum_{i=1}^r  \textbf{tr}(f|_{V_i})
\]
So the definition of $M(A,B)$ can be rewritten
\begin{eqnarray*}
  M(A,B) +d &=&  \frac{1}{2} \cdot \inf_{\lambda_1,...,\lambda_r>0} \ \ \sum_{i=1}^r 
                \textbf{tr}\left(\lambda_i^{-1} B + \lambda_i B^{-1}\right)|_{V_i} \\
           &=& \frac{1}{2} \cdot \sum_{i=1}^r  \ \ \inf_{\lambda>0}\ \ \lambda^{-1} \textbf{tr}(B|_{V_i}) + \lambda \textbf{tr}(B^{-1}|_{V_i})
\end{eqnarray*}
Since $B$ is positive definite, $B^{-1}$ is positive definite as well. When $B$ is positive definite, then 
also $B|_V$ is positive definite, so we have $\textbf{tr}(B|_{V_i})>0$ and $\textbf{tr}(B^{-1}|_{V_i})>0$.
Thus for $b:= \textbf{tr}(B|_{V_i})$ and $c:= \textbf{tr}(B^{-1}|_{V_i})$ the function 
$ \lambda^{-1} b + \lambda c$
has a minimum for a finite positive $\lambda$,  and it is
\[
    \argmin_{\lambda>0}  \lambda^{-1} b + \lambda c = \sqrt{b/c}
\]
as is seen e.g. by comparing the derivative with zero, and hence 
\[
   \inf_{\lambda>0}\ \ \lambda^{-1} b + \lambda c
   = \min_{\lambda>0}\ \  \lambda^{-1} b + \lambda c = 2\cdot \sqrt{b\cdot c},
\]
which gives the above formula 7. 

$M(A,B)$ cannot be continuous in $A$ if condition 2 should hold: 
Take $B$ the diagonal matrix with entries 1 and 2, and for $A$ consider the diagonal matrices with entries 1 and $\lambda$.
Then for $\lambda$ = 1 the matrix $B$ is not aligned with $A$, but it is for all $\lambda \neq 1$.
Therefore we must have $M(diag(1,1), B) > 0$, but $M(diag(1,\lambda), B) = 0$ for all $\lambda \neq 0$, so $A$ would not be continuous with respect to $A$.
\end{proof}

\clearpage
\section{The shape of the transfer function \texorpdfstring{$f(\sigma)$}{f(sigma)}}

\subsection{Gaussian centered input}
\label{subsec:E_Gaussian}
In the ``ideal'' case of Gaussian centered input, we experimentally observe
curves that ``look continuous'', see Figure~\ref{fig:MoreGaussianFsigma}.
For this figure, we used different settings to obtain curves that go only down (left), rise 
and then fall (middle) or only go up (right).

We always used Gaussian input $\calN(0,\Sigma_x)$ with mean 0 and a diagonal matrix for the covariance $\Sigma_x$ (which is no restriction of
generality, since applying an orthogonal matrix to the input does not change the dynamic
of the neural networks, and any symmetric positive definite matrix can be written as
$O^T L O$ with diagonal matrix $L$ and orthogonal matrix $O$).
The labels are always uniformly randomly sampled from $\{0,1,...,9\}$. 
The networks are fully connected networks with 2 hidden layers and ReLU activation, we use
cross entropy loss. 
We use standard He initialization and train with Gradient Descent until convergence.
In each of the three plots of Figure \ref{fig:MoreGaussianFsigma} we show the results from 5
runs with different random initializations.

Settings for the three plots in Figure \ref{fig:MoreGaussianFsigma}:

\centerline{
\begin{tabular}{@{}rccc@{}}
\hline 
  & Left & Middle & Right \\
\hline
Input dim  & 10   &  30   &  30   \\
Layer 1    & 2048 & 2048 & 256 \\
Layer 2    & 256  &  256 & 256 \\
Output dim  & 10   &  10   &  10   \\
Input size &10\,000 & 10\,000 & 2\,000 \\
Input $\Sigma_x=$ 
& $diag(1, 1.1, 1.2, ... 1.9)$ 
           & $diag(0.1, 0.2,...,3.0)$ 
           & $diag(0.1, 0.2,...,3.0)$ \\
\hline 
\end{tabular}
}

We can give a heuristic argument for why the curves should ``look continuous'':
If two eigenvalues of the data covariance $\Sigma_x$ are close, exchanging them
gives an input distribution that is close to the original distribution. Because 
this exchange is an orthogonal transformation, we will get trained networks 
in which the weights also only differ by the same orthogonal transformation.
This means that the effect of exchanging the sigmas will also exchange the 
taus, and if exchanging the sigmas had a small effect on the input distribution,
we may expect also a small effect on the weights distribution, which means we 
expect the corresponding taus also to be close.

For the other experimental observation, that the curves first rise and then fall again (where
one of these two parts can also be missing), we already sketched the two conjectured 
mechanisms in \ref{subsect::mapping_eigenvalues}:

\begin{enumerate}
  \item Larger eigenvalues $\sigma_i$ lead to larger effective learning rate in gradient descent, which leads in turn to larger corresponding $\tau_i$, hence the increasing part of $f$.
  \item We find experimentally that the first eigenvector(s) dominate the input 
  (see e.g. the first table in Appendix D.1).
  Using an orthonormal basis $e_i$ of eigenvectors of $\Sigma_x$ and $\Sigma_w$, we 
  can decompose the variance of the output to a neuron $\BE_x[\langle w, x \rangle^2]$ as
  $\sum_i \langle w, e_i\rangle^2 \BE_x[\langle e_i, x \rangle^2] = \sum_i \langle w, e_i\rangle^2\cdot \sigma_i^2$. Averaging over $w$ gives $\sum_i \tau_i^2 \cdot \sigma_i^2$.
  So if $f(\sigma)$ would be increasing, direction $e_1$ would dominate the output even more.
  We speculate that backprop finds a near optimal solution, and it seems plausible
  that one component dominating is not optimal when there is also important information in the 
  other components.
\end{enumerate}

\begin{figure}[!b] 
   \centering
    \includegraphics[height=0.23\textwidth]{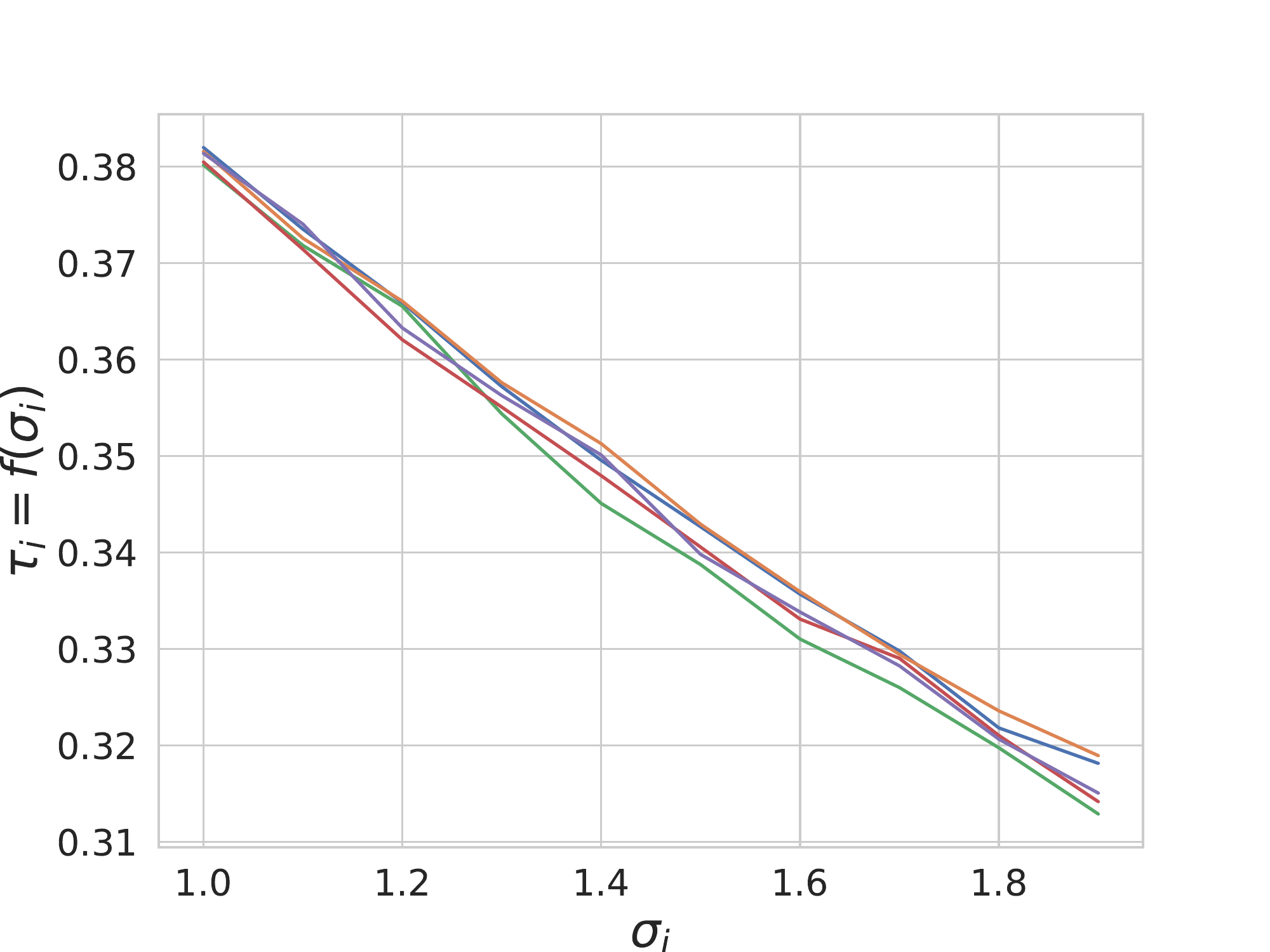}
    \includegraphics[height=0.23\textwidth]{images/st8.pdf}
    \includegraphics[height=0.23\textwidth]{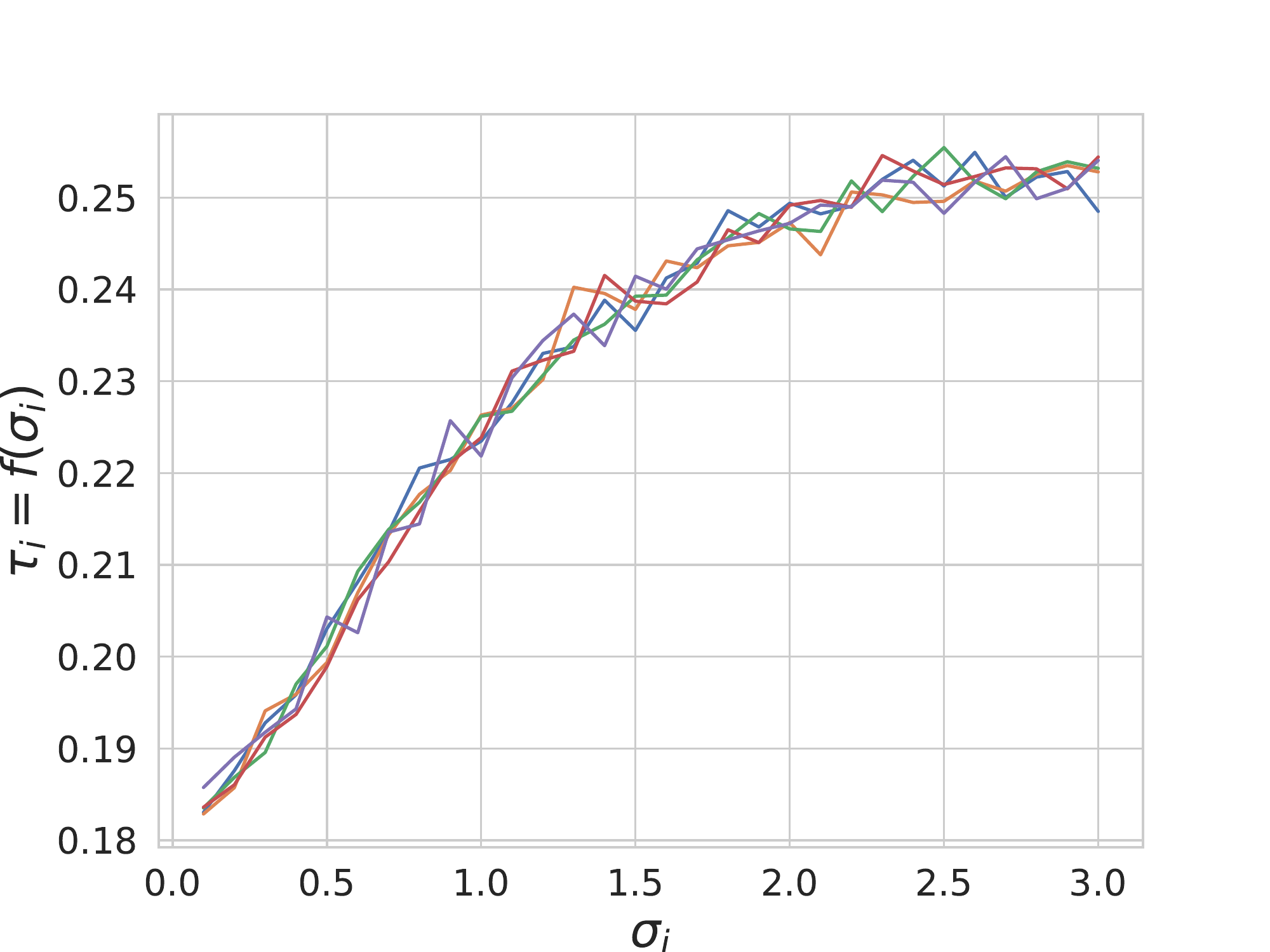}
   \caption{
   $f(\sigma)$--curves for fully connected networks with two hidden layers and Gaussian centered inputs. See text for details.
   }
   \label{fig:MoreGaussianFsigma}
\end{figure}

\subsection{Convolutional networks on natural images}

While the $f(\sigma)$ curves in the ideal centered Gaussian case look smooth, 
the real-world $f(\sigma)$ curves seem to contain perturbations, although they 
still roughly have a rising/falling shape as well. To investigate what 
the largest contributor to this perturbations is, we will go from the ideal 
case to the real case in a series of steps.

As a first step, we crop the images of CIFAR10 to $27\times 27$ pixels and
approximate their distribution by the normal distribution with mean 0 and 
the covariance of all $3 \times 3$ patches. When we apply a convolutional
network with $3\times 3$ convolutions and stride 3, we are in the ``ideal'' situation
of \ref{subsec:E_Gaussian} and expect a $f(\sigma)$ curve that goes up and then down.
This is indeed the case, as shown by the green curve in Figure~\ref{fig:independent-patches}.
\begin{figure}[tbp]
    \centering
\includegraphics[width=0.49\textwidth]{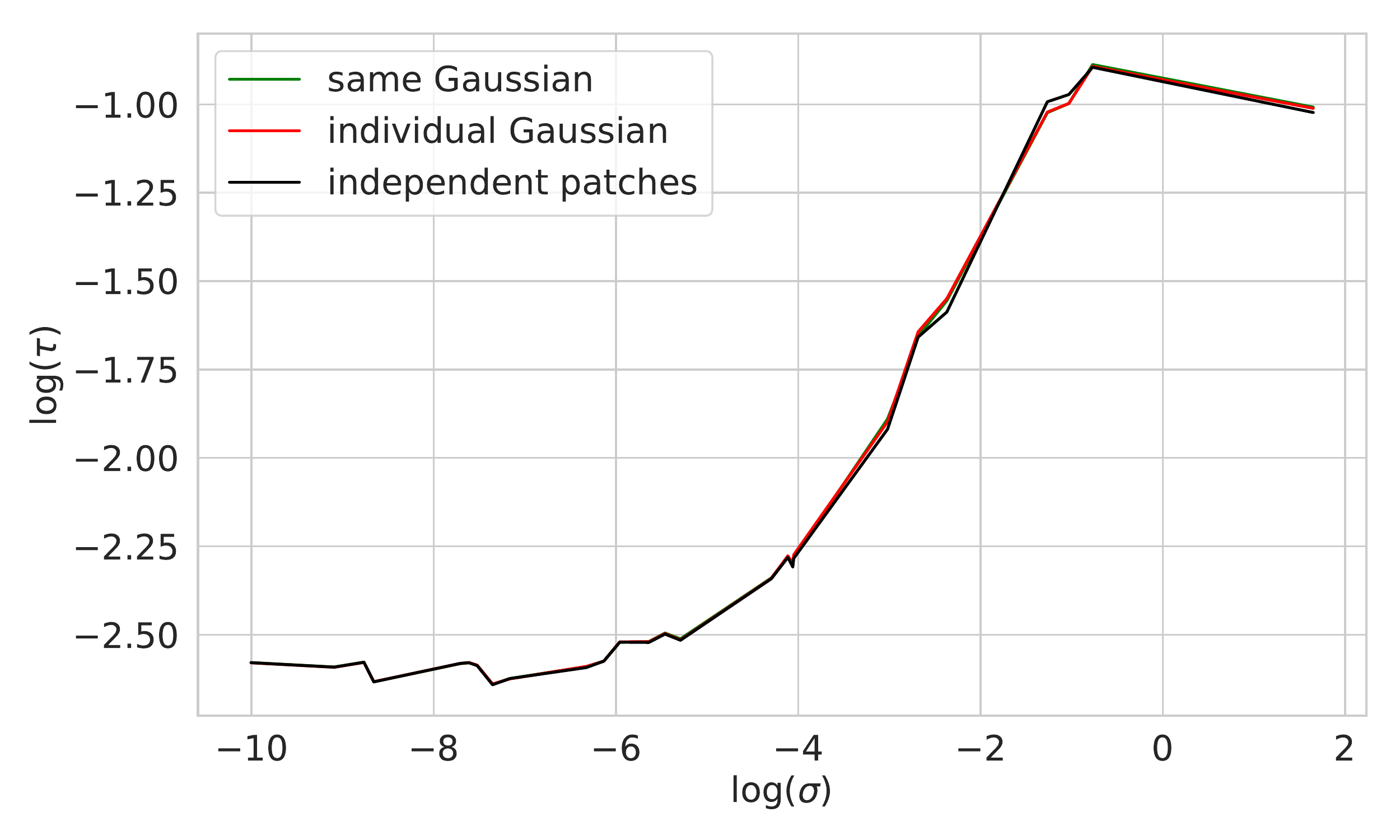}
\includegraphics[width=0.49\textwidth]{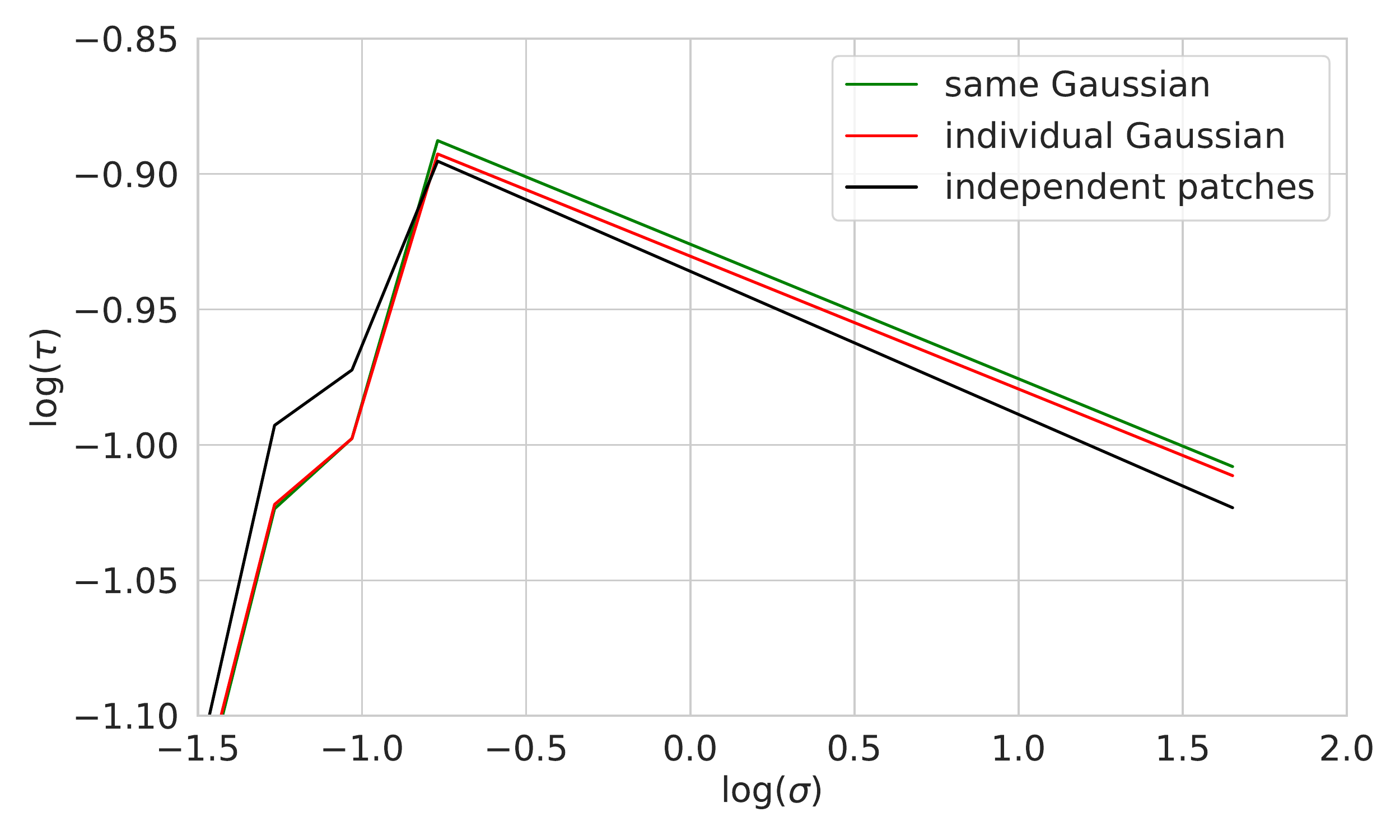}
    \caption{$f(\sigma)$ curve for training on random labels with 
    different approximations of CIFAR10 images, cropped to $27\times 27$ pixels. 
    Two $3\times 3$-convolutional layers with stride 3, one fully connected layer. 
    {\sc Right}: Zoomed in to the upper right corner to show the small differences.}
    \label{fig:independent-patches}
\end{figure}
Using the same distribution / covariance for all patch positions is of course a simplification: For example, for the patches at the upper boundary it is more likely to see light blue (from the sky) than at the lower boundary.
To get closer to reality, we can replace the one global covariance by the covariances corresponding to the possible patch positions, yielding the red curve in Figure~\ref{fig:independent-patches}.

The next step is to abandon Gaussian approximation, and take the real patches. 
In the $27\times27$ pixel images we have $9\times9$ patches of $3\times 3$ pixels each.
To destroy the correlations between neighboring patches, we permute
for each of the $9\times9$ positions the patches in that position of all images.
So for each position we still have the same set of $3\times 3$ patches that can 
appear in an image, but the patches appearing in one new image no longer fit together since they (almost always) came from different original images. This leads to images that are stitched together from random patches and also means that previously (potentially similar) patches from the same original image can now occur with two different random labels.
This setting yields 
the black curve in Figure~\ref{fig:independent-patches}. 
So far it seems we are still very close to the ideal situation.

The next step is to take the original images; now the correlations between neighboring
patches do create a more significant change in the $f(\sigma)$ curve, see the blue curve in Figure~\ref{fig:stride-effect}. In particular, we see the first significant deviation from
``up and then down''.
This effect becomes stronger when we go from stride 3 to overlapping inputs of the 
convolution (stride 2 and stride 1, green and red curves in Figure~\ref{fig:stride-effect}). 

\begin{SCfigure}
\includegraphics[width=0.49\textwidth]{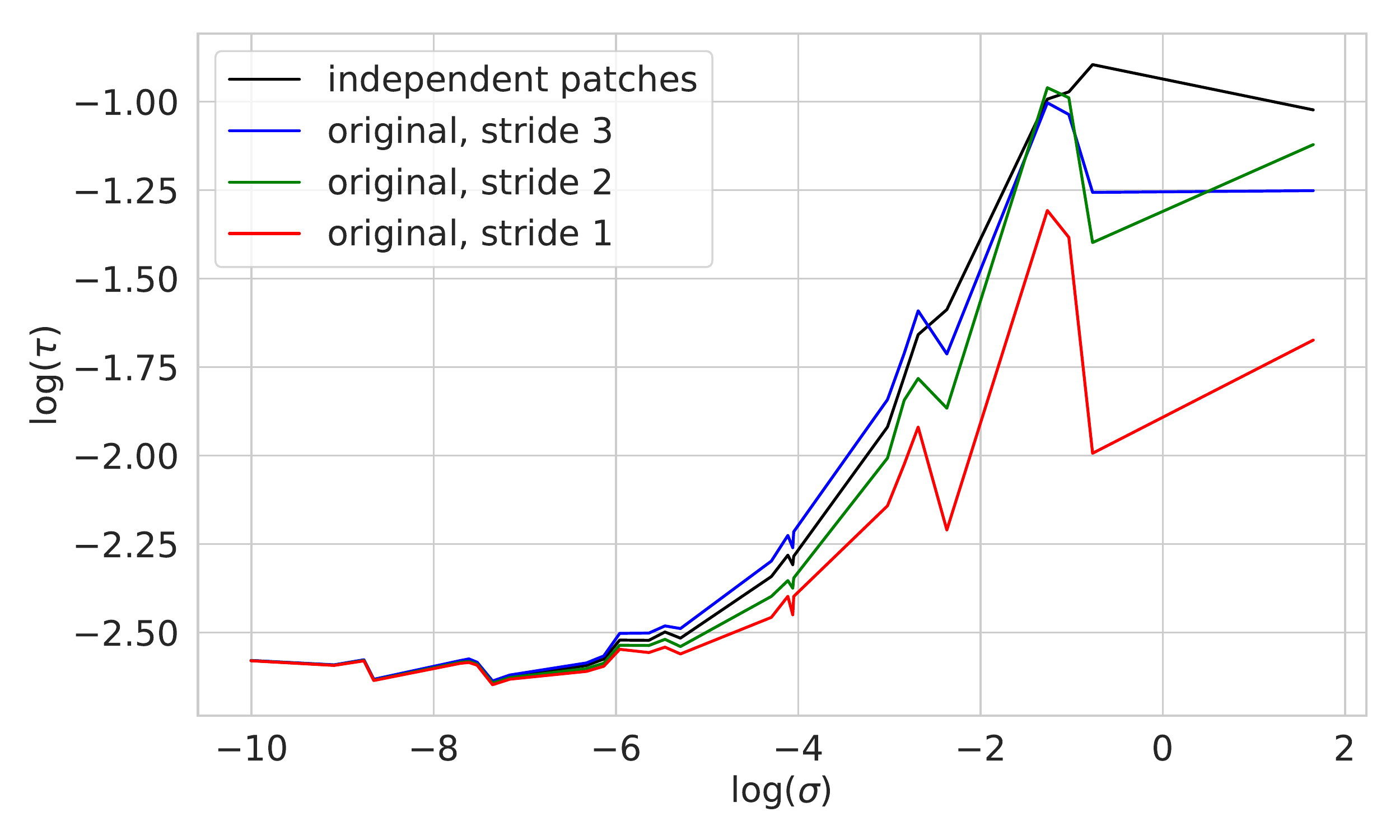}
    \caption{$f(\sigma)$ curve for training on random labels with 
    different approximations of CIFAR10 images. Two $3\times 3$-convolutional layers, one
    fully connected layer. Correlations between neighboring patches create deviation
    from the curve seen for the ``ideal'' case of independent Gaussian approximations. 
    Smaller strides lead to stronger correlations and stronger 
    deviation.
    }
    \label{fig:stride-effect}
\end{SCfigure}

\subsection{Deeper Layers}

We use the following simple CNN on the CIFAR10 data set here (as in Section~\ref{subsec:DeepNetworks}, Table 1):
\begin{Verbatim}[fontsize=\small]
       conv 3x3, 16 filters 
       conv 3x3, 16 filters 
       conv 3x3, 16 filters 
       maxpool 3x3, stride 2
       dense layer, 512 units
       dense layer, num_outputs units (classifier head)
\end{Verbatim}

Training is done using 
30k images for pre-training on random labels (for comparison run),
20k images for training on real labels, 
10k images for determining the test error, and using a
learning rate of~0.002.

{\bf Sampling from covariance vs. picking eigenvectors}
In Figure \ref{fig:PCA_init} we compared a pre-trained convolutional layer with
a layer consisting of filters randomly sampled from the same covariance matrix
and found that this preserved the performance benefit exactly. In the experiment
for Table \ref{table:experiment_3conv} we instead used the most significant eigenvectors
$e_1,...,e_{16}$ with factors $\tau_1,...,\tau_{16}$ directly as filters. 
This is a (more stable) approximation to sampling from the covariance matrix that 
has $\tau_i^2$ as eigenvalue for $e_1,...,e_{16}$ and $\tau=0$ for the other eigenvectors.
Experimentally, we can see that this does not seem to make a significant difference:
Compare Table~\ref{table:deep_f2} with Table~\ref{table:sampled_from_cov} for the case that
$\tau_1=...=\tau_{16}=1$. 

\begin{SCfigure}
  \includegraphics[width=0.49\textwidth]{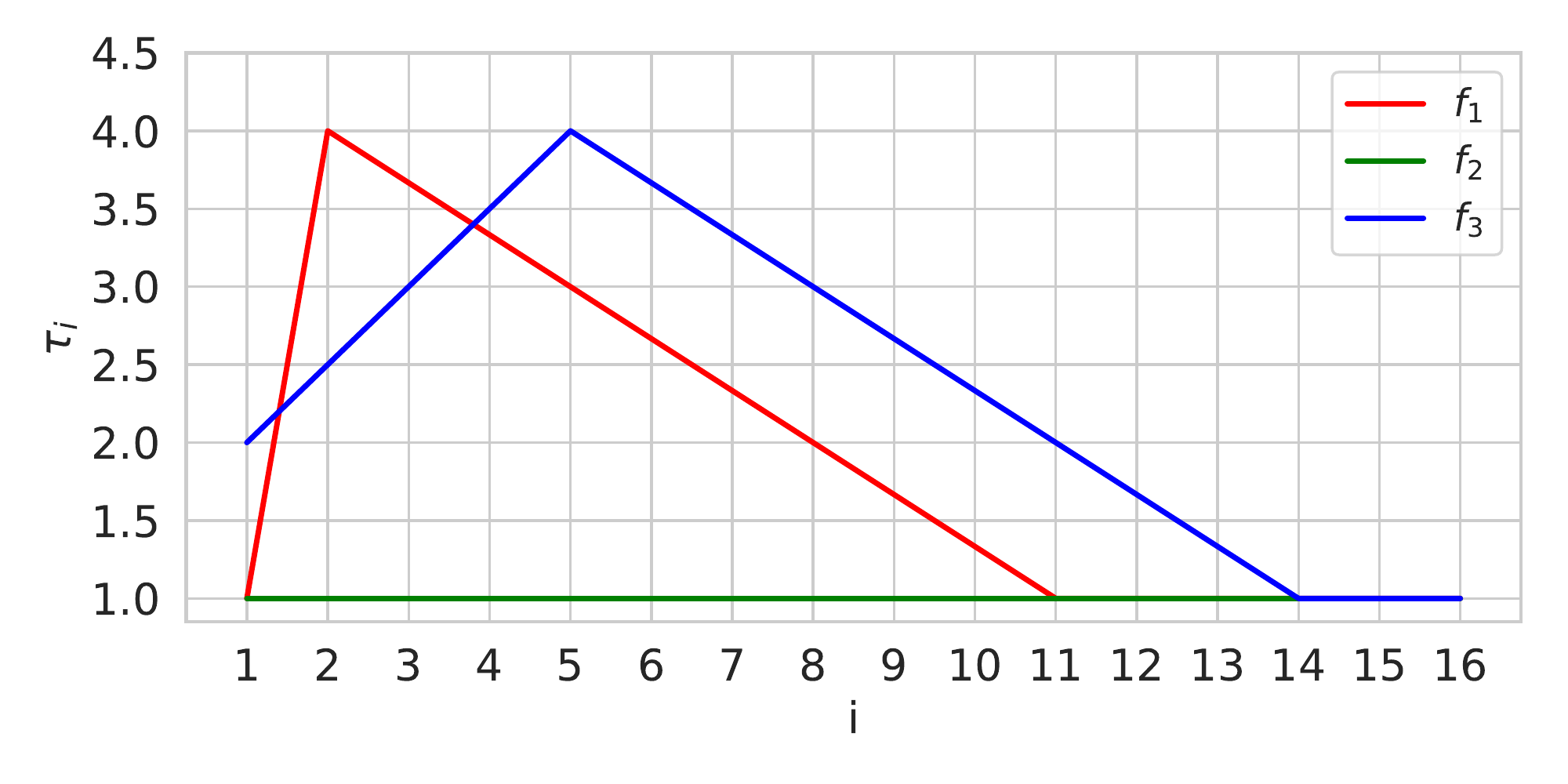}
  \caption{Sequence of $\tau_1,...,\tau_{16}$ used to weight the 16 most
  significant eigenvectors in our experiments. 
  The sequence $f_1$ was used for Table 1 in section 
  \ref{subsec:DeepNetworks}, $f_2$ and $f_3$ were used for Table 2 and Table 3.}
  \label{fig:deep_taus}
\end{SCfigure}

{\bf \boldmath Choice of eigenvectors $e_1,...,e_{16}$ and $\tau_1,...,\tau_{16}$}
In our previous experiments we observed curves for $f(\sigma)$ that were ``rising and
then falling'', favoring the most significant eigenvalues, but somewhat downweighting the
most significant one(s) (e.g.\ the center and right plot in Figure 4 -- this was a different
network on the same data). So we picked the 16 largest $\sigma$ to be the ones with
$f(\sigma)>0$ (i.e. used the most significant eigenvectors), and used a made-up 
curve ($f_1$ in Figure \ref{fig:deep_taus}) for the $\tau_1,...,\tau_{16}$ which downweights
the largest eigenvalue. We can see that the choice of this curve does not make a big 
difference: We can also just set $\tau_1=...=\tau_{16}=1$  ($f_2$, resulting in Table
\ref{table:deep_f2}), or choose another curve which downweights the largest 
eigenvalue less ($f_3$, resulting in Table \ref{table:deep_f3}). The stronger 
downweighting of the most significant eigenvalue seems to give a slight advantage
when only initializing the first layer, but when initializing two or three layers, 
all curves seem to lead to essentially the same performance. 
However, compared to adjusting some values $f(\sigma)$ between 1 and 4, the much more
important choice is which $f(\sigma)$ we set to 0, i.e.\ which eigenvectors we use.
Tables \ref{table:smallest_eigenvectors} and \ref{table:medium_eigenvectors} show the
effect of choosing other eigenvectors, which leads to a significant drop in performance
(for the choices in these tables even below the standard initialization).

\begin{SCtable}[2][p] 
\caption{Using $f_2 = 1$: The 16 eigenvectors with 
largest eigenvalues are equally weighted. This corresponds to a $f(\sigma)$ which is
1 for the 16 largest $\sigma$s and 0 for the other (11 or 128) $\sigma$s.
We achieve significant gains compared to the standard initialization. 
}
\label{table:deep_f2}
\begin{tabular}{@{}rlcccc@{}}
\hline 
  & & \multicolumn{4}{c}{Convolutional layers sampled} \\
Iterations & Data & $\{\}$ & $\{1\}$ & $\{1,2\}$ & $\{1,2,3\}$ \\ \hline 
100 & Train  & 0.31  & 0.29  & 0.37  & 0.44 \\
    & Test   & 0.31  & 0.29  & 0.37  & 0.43 \\
1000 & Train & 0.58  & 0.59  & 0.65  & 0.68 \\
     & Test  & 0.53  & 0.55  & 0.56  & 0.59 \\
\hline
\end{tabular}
\end{SCtable}

\begin{SCtable}
\caption{Using $f_3$: Changing the $f(\sigma)$ between 1 and 4
does not affect the performance significantly, as long as 
we keep the same set of $\sigma$ with $f(\sigma)>0$.}
\label{table:deep_f3}
\begin{tabular}{@{}rlcccc@{}}
\hline 
  & & \multicolumn{4}{c}{Convolutional layers sampled} \\
Iterations & Data & $\{\}$ & $\{1\}$ & $\{1,2\}$ & $\{1,2,3\}$ \\ \hline 
100 & Train  & 0.31  & 0.30  & 0.37  & 0.44 \\
    & Test   & 0.31  & 0.29  & 0.37  & 0.42 \\
1000 & Train & 0.58  & 0.61  & 0.64  & 0.72 \\
     & Test  & 0.53  & 0.55  & 0.58  & 0.56 \\
\hline
\end{tabular}
\end{SCtable}

\begin{SCtable}
\caption{Sampling from a covariance matrix using 
$f(\sigma)=1$ for the highest 16 eigenvalues and $f(\sigma)=0$ else.
This gives essentially the same performance as using the 16 most
significant eigenvectors directly as filters.}
\label{table:sampled_from_cov}
\begin{tabular}{@{}rlcccc@{}}
\hline 
  & & \multicolumn{4}{c}{Convolutional layers sampled} \\
Iterations & Data & $\{\}$ & $\{1\}$ & $\{1,2\}$ & $\{1,2,3\}$ \\ \hline 
100 & Train  & 0.31  & 0.29  & 0.37  & 0.41 \\
    & Test   & 0.31  & 0.29  & 0.36  & 0.40 \\
1000 & Train & 0.58  & 0.57  & 0.62  & 0.69 \\
     & Test  & 0.53  & 0.54  & 0.54  & 0.56 \\
\hline
\end{tabular}
\end{SCtable}

\begin{SCtable}
\caption{Using the 16 eigenvectors
with the smallest eigenvalues gives significantly worse
performance than random initialization. The 16 chosen
eigenvectors were given the same weight.
}
\label{table:smallest_eigenvectors}
\begin{tabular}{@{}rlcccc@{}}
\hline 
  & & \multicolumn{4}{c}{Convolutional layers sampled} \\
Iterations & Data & $\{\}$ & $\{1\}$ & $\{1,2\}$ & $\{1,2,3\}$ \\ \hline 
100 & Train  & 0.31  & 0.15  & 0.14  & 0.15 \\
    & Test   & 0.31  & 0.15  & 0.14  & 0.15 \\
1000 & Train & 0.58  & 0.50  & 0.48  & 0.47 \\
     & Test  & 0.53  & 0.46  & 0.47  & 0.44 \\
\hline
\end{tabular}
\end{SCtable}

\begin{SCtable}
\caption{Using eigenvectors $e_4,...e_{19}$ (out of 27) in the first layer,
$e_{10},...,e_{25}$ (out of 144) in the second and third layer.
 The 16 chosen eigenvectors were given the same weight.
}
\label{table:medium_eigenvectors}
\begin{tabular}{@{}rlcccc@{}}
\hline 
  & & \multicolumn{4}{c}{Convolutional layers sampled} \\
Iterations & Data & $\{\}$ & $\{1\}$ & $\{1,2\}$ & $\{1,2,3\}$ \\ \hline 
100 & Train  & 0.31  & 0.17  & 0.20  & 0.27 \\
    & Test   & 0.31  & 0.18  & 0.20  & 0.26 \\
1000 & Train & 0.58  & 0.56  & 0.60  & 0.62 \\
     & Test  & 0.53  & 0.53  & 0.56  & 0.59 \\
\hline
\end{tabular}
\end{SCtable}

\begin{SCtable}
\caption{Base line: These are the accuracies we obtain by 
pre-training on random labels. In particular when we use two or three
convolutional layers, the results are very similar to what we get with
our method that only uses the data, no training. (Compare e.g.\ with Table 2)}
\label{table:base_line}
\begin{tabular}{@{}rlcccc@{}}
\hline 
  & & \multicolumn{4}{c}{Convolutional layers sampled} \\
Iterations & Data & $\{\}$ & $\{1\}$ & $\{1,2\}$ & $\{1,2,3\}$ \\ \hline 
100 & Train  & 0.31  & 0.37  & 0.39  & 0.41 \\
    & Test   & 0.31  & 0.34  & 0.37  & 0.39 \\
1000 & Train & 0.58  & 0.70  & 0.75  & 0.78 \\
     & Test  & 0.53  & 0.51  & 0.53  & 0.57 \\
\hline
\end{tabular}
\end{SCtable}

\clearpage  
\section{Specialization at the later layers}

Earlier, we claimed that neural activations at the outer layer can drop abruptly and permanently after switching to the downstream task. Figure \ref{fig::neural_activation_transfer} illustrates this. In this figure, the $x$-axis corresponds to the number of training iterations, where the transfer to the downstream tasks happens at the middle. Note that in both network architectures shown in the figure, neural activation in the upper layers drops abruptly and permanently after switching to the downstream tasks. We interpret this to effectively reduce the available capacity for the downstream task, which masks the positive transfer happening from the alignment effect at the lower layers. This effect can be mitigated by increasing the width, as discussed in the context of Figure~\ref{fig:SimpleCNNIncreasingWIdth} in the main text.

\begin{figure}[b!] 
\centering
\includegraphics[width=0.6\textwidth]{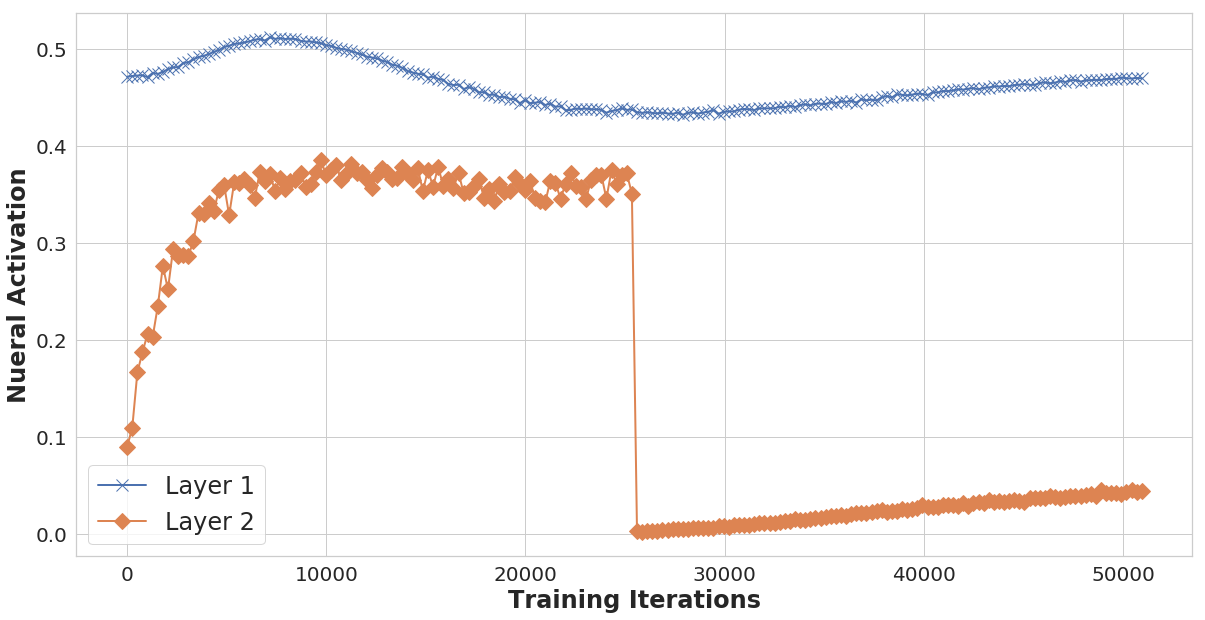}\\[5pt]
\includegraphics[width=0.6\textwidth]{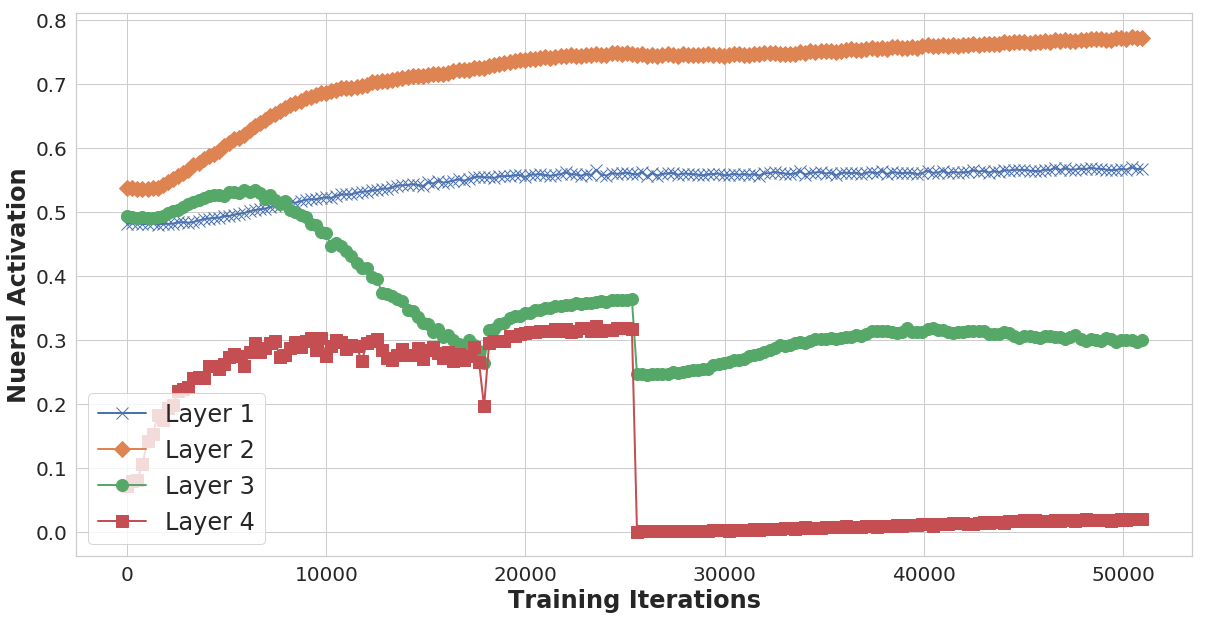}
\caption{{\sc top:} Neural activation is plotted against the number of training iterations in a two-layer CNN (256 $3\times 3$ filters followed by a dense layer of width 64). The $y$-axis is the frequency of the output of a hidden activation function being non-zero measured over a hold-out dataset. The abrupt drop in neural activation coincides with the switch to the downstream task. {\sc bottom:} A similar plot for a deeper neural network with three convolutional layers  (256 $3\times 3$ filters) followed by a single dense layer of width 64. }\label{fig::neural_activation_transfer}
\end{figure}

Regarding specialization, Figures \ref{fig:VGGNeuronActivationSupplementary} and \ref{fig:VGGNeuronActivationSupplementaryPositive} are the extended versions of Figure \ref{fig:SimpleCNNIncreasingWIdth}.
They include activation plots for \emph{all intermediate layers} of VGG16 models
(i) at initialization, 
(ii) in the end of the pre-training, 
(iii) in the end of fine-tuning on real labels,
(iv) in the end of fine-tuning on 10 random labels.
Figures \ref{fig:VGGNeuronActivationSupplementary} and \ref{fig:VGGNeuronActivationSupplementaryPositive} illustrate the negative and positive transfer examples respectively. See Section \ref{subsect::nuron_act_plots_vgg} for details about the parameters used when generating those figures. As shown in those figures, neurons at the upper layers tend to \emph{specialize}, i.e. become activated by fewer images. This is evident if we compare the fraction of examples activating neurons when pretrained for 1 training iteration ({\sc top-left}) vs. pre-trained for 6240 training iterations ({\sc top-right}).

Because neural activations drop permanently after switching to the downstream task, the capacity of the neural network in the downstream task is effectively diminished. This is reminiscent of the \lq\lq critical stages" that have been observed for deep neural networks, in which neural networks seem to possess less capacity to fit a new distribution of data (e.g. after image blur is removed) later during training than when trained from scratch. Here, the upstream task is analogous to the critical early stage and the switch to the downstream task is analogous to the change of distribution (e.g. removing blur in images). 

\end{document}